\documentclass[11pt]{article}

\usepackage[a4paper,margin=1in]{geometry}
\usepackage{amsmath,amssymb,amsthm}
\usepackage{enumitem}
\usepackage[table]{xcolor}   
\usepackage{hyperref}
\usepackage{booktabs}
\usepackage{multirow}
\usepackage{array}
\usepackage{graphicx}
\usepackage{wrapfig}
\usepackage{needspace}
\usepackage{caption}        
\usepackage[T1]{fontenc}
\usepackage{tikz}
\usetikzlibrary{positioning, arrows.meta, calc, shapes}
\hypersetup{
  colorlinks=true,
  linkcolor=blue!50!black,
  citecolor=blue!50!black,
  urlcolor=blue!50!black,
}

\newtheorem{theorem}{Theorem}[section]
\newtheorem{proposition}[theorem]{Proposition}
\newtheorem{lemma}[theorem]{Lemma}

\theoremstyle{definition}
\newtheorem{definition}[theorem]{Definition}
\newtheorem{assumption}[theorem]{Assumption}
\theoremstyle{remark}
\newtheorem{remark}[theorem]{Remark}

\newcommand{\LN}{\mathrm{LN}}
\newcommand{\Lip}{\mathrm{Lip}}

\newcommand{\citep}[1]{\cite{#1}}
\newcommand{\citet}[1]{\cite{#1}}

\newcommand{\qhat}{\hat q}
\newcommand{\shat}{\hat\sigma}

\newcommand{\defPrim}{Definition~\ref{def:primitive}}

\graphicspath{{figures/}{../figures/}{./}}
\makeatletter
\providecommand\input@path{}
\edef\input@path{{figures/}{../figures/}}
\makeatother
\title{Certifying Residual Architectures from Their Primitives: Growth Exponents and Criticality}
\author{%
\begin{tabular}{c@{\hspace{4em}}c}
Hyemin Gu & Michael Tyrrell \\
\normalsize University of Massachusetts Amherst & \normalsize SRI International \\
\texttt{hgu@umass.edu} & \texttt{michael.tyrrell@sri.com} \\[1.2em]
Tuhin Sahai & Markos A.~Katsoulakis \\
\normalsize SRI International & \normalsize University of Massachusetts Amherst \\
\texttt{tuhin.sahai@sri.com} & \texttt{markos@umass.edu}
\end{tabular}
}
\date{}

\begin{document}
\maketitle

\begin{abstract}
Whether a deep residual architecture avoids numerical blow-up in its forward and backward dynamics is typically determined through costly trial-and-error training. Existing stabilization strategies, including normalization and residual scaling, lack a common framework for predicting their effects across architectures. We develop an a priori certification framework based on two exponents derived from architectural primitives: the forward growth exponent $q$ and the Jacobian sensitivity exponent $\sigma$. These critical exponents characterize qualitative stability regimes and determine how certified allowable depth scales with the floating-point limit $M$. In particular, $q=1$ marks a sharp threshold: for $\sigma>0$, certified depth scales as $\log M/\log\log M$, $\sqrt{\log M}$, and $\log_q\log M$ in the regimes $q<1$, $q=1$, and $q>1$, respectively, exhibiting qualitative changes reminiscent of phase transitions in statistical physics. The sensitivity exponent $\sigma$ further constrains backward stability, even when $q\le1$. An arithmetic of exponents enables certification directly from computation graphs, including heterogeneous residual stacks, and identifies targeted modifications to existing backbones without requiring uniform normalization. Experiments on Mamba, PatchTST, FNO, GPT-2, and a Toto~2.0 proxy support the predicted stability distinctions: controlling superlinear growth prevents divergence in Mamba and Toto~2.0, while reducing attention sensitivity restores stable GPT-2 training. Among stable designs, $q=1$ architectures can outperform normalized $q=0$ counterparts in forecasting and operator-learning extrapolation. Rather than prescribing a single architecture, our framework identifies families of architectures with certifiable stability, enabling task-dependent design choices before training. 
Together, our theoretical and empirical results narrow the scope of trial-and-error in architecture design, enabling more systematic, task-dependent exploration.
\end{abstract}

\paragraph{Acknowledgements.}
This material is based upon work of the authors supported by the Defense Advanced Research Projects Agency (DARPA) under Agreement No. HR00112590112. Approved for public release; distribution is unlimited.

\section{Introduction}
\label{sec:intro}

Modern deep architectures, spanning Transformers \citep{vaswani2017attention} and emerging post-Transformer backbones \citep{gu2024mamba,peng2025rwkv7,sun2023retnet,poli2023hyena,lieber2024jamba}, combine an increasingly diverse range of computational primitives, including self-attention, multiplicative gating \citep{shazeer2020glu}, normalization \citep{ba2016layer}, and beyond. These primitives determine not only what a model can represent, but also how its hidden states and gradients behave as depth increases. As new architectures combine them in increasingly diverse ways, a central design question is whether a proposed block can be stacked to the desired depth without inducing unstable forward or backward dynamics.

In practice, this question is often answered by training the architecture and modifying it when instability appears. A wide range of stabilization strategies has emerged, including normalization and its placement \citep{xiong2020layer,kim2025periln}, residual scaling \citep{bachlechner2021rezero,zhang2019fixup,touvron2021going}, spectral approaches \citep{haber2017stable,miyato2018spectral}, and attention-specific stabilization \citep{zhai2023stabilizing,wortsman2024smallscale}; see Sec.~\ref{app:related} for a detailed discussion. These approaches mitigate training instability through different mechanisms, and their stability implications are typically analyzed within individual architectures. What remains missing is a common framework for comparing their stability implications, identifying the limiting primitives in a new architecture, and determining whether less restrictive modifications can suffice.

We address this problem by characterizing the forward growth and backward sensitivity of a residual block through two exponents, $(q,\sigma)$. For a residual architecture \citep{he2016resnet} of depth $D$, the forward states evolve according to
\begin{equation}
X_{k+1}=X_k+\Delta t\,v(X_k;\theta_k),\qquad k=0,\dots,D-1,
\label{eq:residual}
\end{equation}
where $v(\cdot;\theta_k)$ is the $k$-th residual velocity and $\Delta t>0$ is its update scale. The corresponding backward costates satisfy
\begin{equation}
p_k=\bigl(I+\Delta t\,\nabla_xv(X_k;\theta_k)\bigr)^\top p_{k+1},
\qquad p_D=\nabla_XG(X_D,y).
\label{eq:adjoint}
\end{equation}
We characterize the forward growth and backward sensitivity through
\begin{equation}
\|v(x;\theta)\|\le c\|x\|^q+b,
\qquad
L(R):=\sup_{\substack{\theta\in\Theta\\ \|x\|\le R}}
\|\nabla_xv(x;\theta)\|_2
\le \ell_1R^\sigma+\ell_0.
\label{eq}
\end{equation}
The latter allows state-dependent Jacobian growth, relaxing the requirement of a globally uniform Lipschitz bound, as is common in modern architectures, see Tab.~\ref{tab:primitives}. The \emph{forward growth exponent} $q$ controls the growth of the reachable-state envelope with depth, while the \emph{Jacobian sensitivity exponent} $\sigma$ controls Jacobian amplification over that region. Together, they bound forward states and backward costates, yielding a certified depth $D_{\mathrm{cert}}=\min\{D_F,D_\nabla\}$ for a specified floating-point limit $M$ (Tab.~\ref{tab:q-sigma} and  Sec.~\ref{sec:threshold}).

\begin{figure}[ht]
\centering
\begin{minipage}[t]{0.63\linewidth}
\vspace{0pt}
\centering
\captionof{table}{Asymptotic scaling regimes of certified allowable depth $D_{\rm cert}=\min\{D_F,D_\nabla\}$ with the floating-point limit $M\gg1$, shown to leading order (and to leading power of $\log M$ for $D_\nabla$ and $D_{\rm cert}$). The critical behavior is reflected in three distinct certified-depth scaling regimes: nearly logarithmic for $q<1$, square-root logarithmic for $q=1$ ($\sigma>0$), and doubly logarithmic for $q>1$. The region $q>1$ with $\sigma=0$ is empty under the growth-attainment hypothesis of Proposition~\ref{prop:attained-sigma}. Exact expressions are in Tab.~\ref{tab:q-sigma-exact} (App.\ref{app:overflow}).}
\label{tab:q-sigma}
\footnotesize
\setlength{\tabcolsep}{3.5pt}
\begin{tabular}{@{}ll cc@{\hspace{5pt}}|@{\hspace{5pt}}c@{}}
\toprule
$q$ & $\sigma$ & $D_F$ & $D_\nabla$ & $\boldsymbol{D_{\rm cert}}$ \\
\midrule
$<1$ & $0$   & $M^{1-q}$ & $\log M$ & $\boldsymbol{\log M}$ \\
$<1$ & $>0$  & $M^{1-q}$ & $\log M$ & $\boldsymbol{\log M}$ \\
$=1$ & $0$   & $\log M$ & $\log M$ & $\boldsymbol{\log M}$ \\
\midrule
$=1$ & $>0$  & $\log M$ & $\sqrt{\log M}$ & $\boldsymbol{\sqrt{\log M}}$ \\  
\midrule
$>1$ & $>0$  & $\log_q\log M$ & $\log_q\log M$ & $\boldsymbol{\log_q\log M}$ \\
\bottomrule
\end{tabular}
\end{minipage}\hfill
\begin{minipage}[t]{0.35\linewidth}
\vspace{0pt}
\centering
\includegraphics[width=\linewidth]{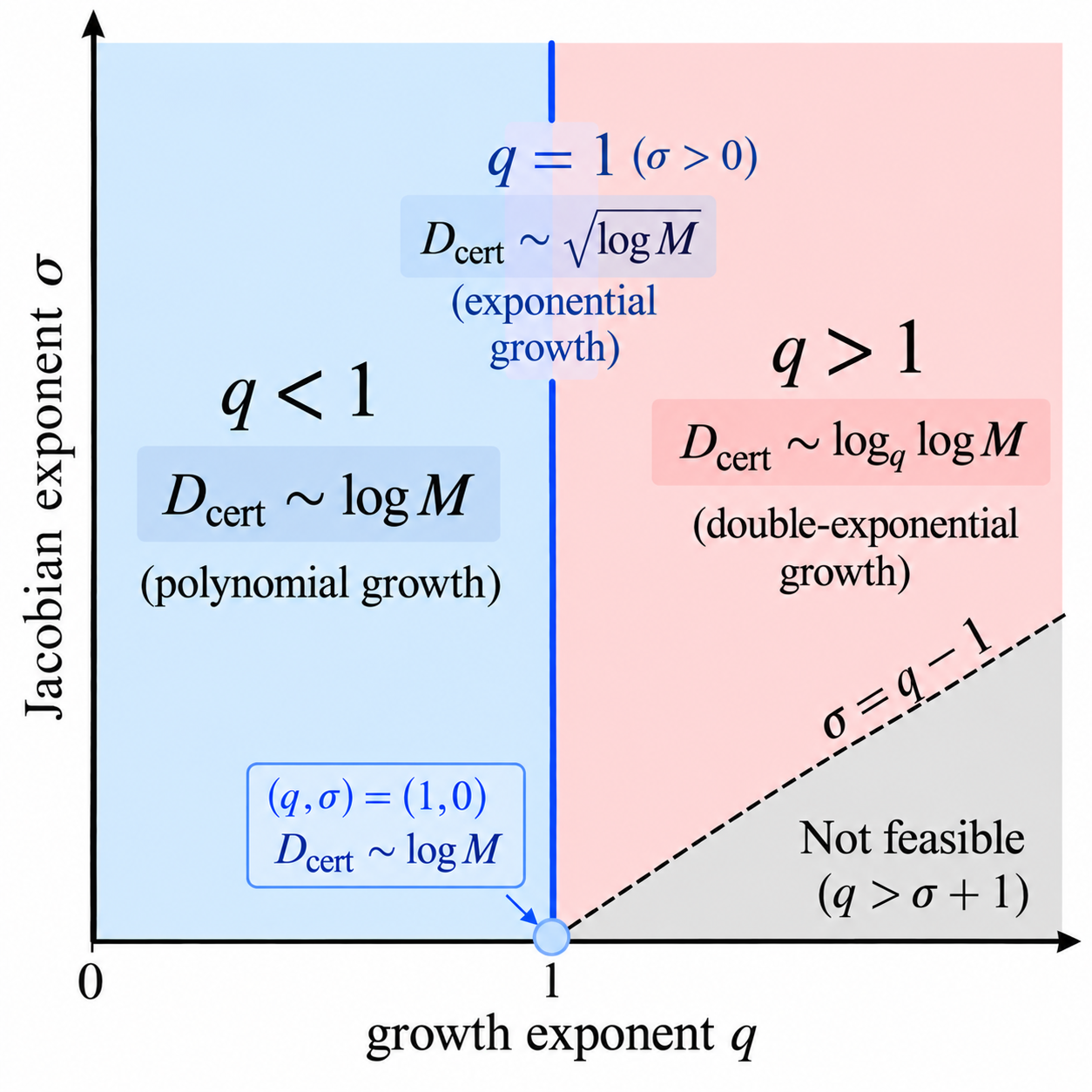}
\caption{Certified depth $D_{\rm cert}$ in the $(q,\sigma)$ plane. Fig.~\ref{fig:phase-empirical} overlays the measured exponents of existing and newly constructed architectures, with all stable ones in the blue region and all unstable ones in the red region.}
\label{fig:phase}
\end{minipage}
\end{figure}

\paragraph{Critical growth and sensitivity exponents $(q, \sigma)$ for stability.}
Our analysis identifies $q = 1$ as the critical growth
exponent, which separates the certifiable families q$ \leq 1$ from $q > 1$. As summarized in Fig.~\ref{fig:phase} and Tab.~\ref{tab:q-sigma}, the forward growth envelope is polynomial for $q<1$, exponential for $q=1$, and double-exponential for $q>1$, with corresponding certified-depth scalings of $\log M$, $\sqrt{\log M}$ (for $\sigma>0$), and $\log_q\log M$, respectively, so that the order of
the certified depth in $\log M$ changes at $q=1$; at finite precision the change is a crossover, in the manner of finite-size scaling (Remark~\ref{rem:crossover}). While $q$ determines the forward growth regime, $\sigma$ controls backward sensitivity along the reachable states. Together, these exponents determine the qualitative stability regime (Sec.~\ref{sec:threshold}), while their associated coefficients yield quantitative certificates for a given architecture, depth, and precision (App.~\ref{app:overflow}). 

\paragraph{From primitives to architecture-level certification.} We connect these bounds to architecture design through a primitive-level arithmetic for $(q,\sigma)$ (Theorems~\ref{thm:arith} and \ref{thm:sigma-arith}). Composition, addition, and multiplication propagate exponents through computation graphs; direct estimates handle composite operations such as self-attention and selective state-space blocks. This provides a common language for stabilization strategies: Pre-LN can yield $(q,\sigma)=(0,0)$, and query-key input normalization (QK-input norm) \citep{henry2020qknorm,dehghani2023scaling,wortsman2024smallscale} reduces attention sensitivity while retaining linear forward growth. Given the required coefficient bounds, these  estimates enable \emph{inspection} of individual blocks, \emph{targeted modification}, and \emph{certification} of heterogeneous residual stacks before training (App.~\ref{app:certification}).

Our theoretical and empirical contributions are twofold:
\begin{enumerate}
\item \textbf{Critical exponents define a certified design space with multiple stability regimes.} 
Our framework characterizes a broad class of architectural primitives (Tab.~\ref{tab:primitives}) through $(q,\sigma)$, identifying distinct stable design choices, notably the $q=0$ and $q=1$ regimes under appropriate sensitivity bounds. The certification extends to heterogeneous residual architectures, allowing different block types and stability regimes across depth without prescribing a uniform layer design.

\item \textbf{Task-dependent performance gains within the stable design space.}
Across time-series forecasting, operator learning, and language modeling, our experiments show that controlling superlinear growth ($q>1$) in Mamba \citep{gu2024mamba}, as used in the CMamba forecaster \citep{zeng2024cmamba} (Sec.~\ref{sec:tsf}), and in Toto~2.0 \citep{khwaja2026toto} (Sec.~\ref{sec:foundation}), and excessive backward sensitivity ($\sigma=2$) in GPT-2 \citep{radford2019language} (Sec.~\ref{sec:lm}), enables stable training. Among stable designs, $q=1$ architectures outperform their normalized $q=0$ counterparts in FNO \citep{li2021fno} extrapolation (Sec.~\ref{sec:fno}) and Toto~2.0 forecasting (Sec.~\ref{sec:foundation}), demonstrating the potential of task-dependent regime selection.

\end{enumerate}

\section{Related work}
\label{app:related}

We relate our growth exponents $(q,\sigma)$ to classical dynamical-systems stability criteria and existing approaches to stabilizing deep architectures, including Jacobian constraints, normalization, and attention-specific interventions.

\paragraph{Continuous-time growth criteria.}
Residual networks can be viewed as explicit Euler discretizations of continuous-time dynamics $\dot x=v(x;\theta)$ \citep{e2017proposal,haber2017stable,chen2018neuralode}. Under the growth bound $|v(x)|\le\omega(|x|)$, the Bihari--LaSalle inequality \citep{bihari1956,lasalle1949} gives
\begin{equation}
r(t)\le G^{-1}(t),\qquad
G(s)=\int_{r_0}^{s}\frac{d\xi}{\omega(\xi)},
\end{equation}
where $r(t)=|x(t)|$. The Wintner condition \citep{wintner1945},
$$
\int^\infty\frac{d\xi}{\omega(\xi)}=\infty,
$$
ensures global boundedness on every finite time interval. For $\omega(r)=cr^q+b$, this condition holds exactly when $q\le1$, separating global-in-time boundedness from possible finite-time blow-up in continuous dynamics. In our discrete-depth analysis, the same threshold instead distinguishes polynomial or exponential growth from double-exponential growth, determining how rapidly the forward envelope reaches the finite floating-point limit (Tab.~\ref{tab:q-sigma}).

\paragraph{Stable architectures.}
Existing approaches to stabilizing deep architectures include spectral constraints, residual scaling, and global or local Lipschitz control. These methods address different aspects of stability, from the spectral properties of network dynamics to the magnitude and input dependence of residual updates. Spectral approaches \citep{haber2017stable} control its eigenvalues through structural parameterizations, whereas our backward certificate bounds the operator norm of the residual Jacobian. These criteria are complementary: eigenvalue conditions need not prevent transient amplification for non-normal Jacobians, while norm-based bounds can be conservative.

Residual scaling methods, including ReZero \citep{bachlechner2021rezero}, Fixup \citep{zhang2019fixup}, and LayerScale \citep{touvron2021going}, stabilize deep networks by controlling the magnitude of residual updates. In our framework, such scaling modifies the growth and sensitivity coefficients without changing the exponents $(q,\sigma)$, provided the scaling factors are independent of the state.

Global Lipschitz bounds correspond to $\sigma=0$ in our framework and can be enforced through methods such as spectral normalization \citep{miyato2018spectral}. Standard dot-product self-attention, however, is not globally Lipschitz \citep{kim2021lipschitz}, but admits a local Lipschitz bound of order $O(R^2)$ on bounded input domains \citep{castin2024smooth}, consistent with $\sigma=2$ (App.~\ref{app:attention-sigma}). Existing methods address this sensitivity by modifying attention \citep{kim2021lipschitz,dasoulas2021lipschitznorm,sander2022sinkformers} or bounding its inputs through normalization \citep{castin2024smooth}. Our framework instead evaluates sensitivity at the worst-case reachable-state radius, whose depth dependence is governed by the forward growth exponent $q$. The resulting $(q,\sigma)$ classification identifies potential sources of instability at the architectural level, before training and without requiring precise coefficients or reachable-radius estimates.

\paragraph{Normalization layers and bounded-activation alternatives.}
LayerNorm \citep{ba2016layer}, RMSNorm \citep{zhang2019root}, and ScaleNorm \citep{nguyen2019scalenorm} yield $(q,\sigma)=(0,0)$ through uniformly bounded outputs and Jacobians for $\epsilon>0$ (Tab.~\ref{tab:primitives}). When $\epsilon=0$, the outputs remain bounded away from the origin, but the Jacobians become unbounded near zero, making the regularization essential for $\sigma=0$. Bounded-activation alternatives such as DyT \citep{zhu2025dyt} and BHyT \citep{byun2026bhyt} achieve the same exponents without normalization statistics. The choice and placement of normalization, including Post-LN \citep{vaswani2017attention}, Pre-LN \citep{xiong2020layer}, and Peri-LN \citep{kim2025periln}, have largely relied on empirical trial and error, with relatively few theoretical analyses focusing on individual schemes \citep{kan2025stability}. In contrast, our framework provides an a priori stability classification based on architectural primitives rather than specific normalization choices. Our experiments in time-series forecasting and operator learning further show that removing normalization can sometimes improve predictive performance, motivating stable architectural designs that do not necessarily require normalization.

\paragraph{Query-key normalization for self-attention.}
Query-key normalization has been used to stabilize large-scale self-attention models by controlling attention logit growth and softmax saturation. Existing approaches include $\ell_2$ normalization or LayerNorm applied to queries and keys \citep{henry2020qknorm,dehghani2023scaling,wortsman2024smallscale}, as well as spectral-norm-based reparameterization of attention weights \citep{zhai2023stabilizing}. Our framework explains the effect of query-key normalization through a reduction in the sensitivity exponent from $\sigma=2$ to $\sigma=1$ while preserving $q=1$ (App.~\ref{app:attention-sigma}). Consistent with this prediction, our language-modeling experiments show that QK-input norm prevents the training divergence observed in the unnormalized architecture.

\paragraph{Control, signal propagation, and training without normalization.}
Reading training as optimal control goes back to \citet{e2017proposal} and
\citet{li2018maximum}; the costate recursion \eqref{eq:adjoint} and the maximum
principle of App.~\ref{app:escape} are discrete forms of
\citet{pontryagin1962}, and the threshold $q=1$ is the discrete counterpart of
the linear-growth condition that excludes finite-time blow-up
\citep{hartman2002ode}. Lipschitz constants of whole networks are estimated by
semidefinite programming \citep{fazlyab2019lipschitz} or enforced during
training \citep{gouk2021regularisation}; we use local bounds only on the
reachable ball. Mean-field analyses track how signals and gradients propagate at
initialization \citep{poole2016exponential,schoenholz2017deep}, including
residual networks \citep{yang2017mean}; in transformers, rank collapse
\citep{dong2021attention,noci2022signal} is a different failure from the
overflow certified here. Signal-propagation theory also has a critical point,
the edge of chaos at initialization, which separates ordered from chaotic
propagation of correlations \citep{poole2016exponential,schoenholz2017deep}. The
criticality here is of a different kind: it concerns the growth exponent of the
block over the reachable states, holds for every parameter in $\Theta$, and
governs how the certified depth scales with the format, not how correlations
decay with depth. Normalization layers are also removed by initialization and
scaling \citep{de2020batch}, by adaptive gradient clipping \citep{brock2021high},
by rescaling the residual of Post-LN blocks \citep{wang2022deepnet}, or by
normalization onto the sphere \citep{loshchilov2025ngpt}; reversible instance
normalization of the input series \citep{kim2022reversible}, which PatchTST keeps
in every variant (App.~\ref{app:tsf}), acts outside the residual stack.
\citet{zhai2023stabilizing} relate instability to attention-entropy collapse.
Assumption~\ref{ass:lip} is the state-space analogue of the $(L_0,L_1)$-smoothness
of \citet{zhang2020why}.

\section{Forward and backward certification}
\label{sec:threshold}

We study the residual recursion \eqref{eq:residual} and its adjoint
\eqref{eq:adjoint} to certify whether the states and costates remain
within the floating-point magnitude limit $M$ up to a given depth, without requiring global Lipschitz continuity.
The forward analysis bounds the reachable state radius, while the backward
analysis bounds the costate through the local Jacobian sensitivity over
that region. Together, these bounds give the certified depths $D_F$,
$D_\nabla$, and $D_{\rm cert}=\min\{D_F,D_\nabla\}$. 

\subsection{Forward pass: reachable-state growth}
\label{sec:forward}

The first question is how far the hidden state can grow under repeated
residual updates. To isolate this mechanism, we consider residual blocks with
a uniform polynomial growth envelope. This class is broad enough to include
linear-growth blocks as well as blocks whose state-dependent operations
produce superlinear growth, while reducing their large-state behavior to a
single exponent that can be tracked across architectural primitives.

\begin{assumption}[Uniform polynomial growth]
\label{ass:growth}
There are $q\ge0$, $c>0$, and $b\ge0$ such that
\begin{equation}
\|v(x;\theta)\|\le c\|x\|^q+b
\qquad\text{for all }x\in\mathbb{R}^N,\ \theta\in\Theta.
\label{eq:growth}
\end{equation}
\end{assumption}

We call $q$ the \emph{input-magnitude exponent}. The cases $q<1$,
$q=1$, and $q>1$ correspond to sublinear, linear, and superlinear
growth, respectively. For a depth-$D$ network, let
$\mathcal R_D(q)$ denote a deterministic upper envelope on the reachable
radius,
\[
\|X_k\|\le \mathcal R_k(q),\qquad k\le D,
\]
obtained from Assumption~\ref{ass:growth} and the residual recursion.
The following comparison bound gives an explicit envelope in each growth
regime. The proof is in App.~\ref{app:proof:thm2.2}.
\begin{theorem}[Reachable-radius bound]
\label{thm:threshold}
Under Assumption~\ref{ass:growth}, let $\|X_0\|\le r_0$ with $r_0>0$, and $\bar r_0=\max(r_0,1)$. Every trajectory satisfies $\|X_k\| \le \mathcal R_k(q)$ for all $k\ge0$, where
\begin{enumerate}[label=\textnormal{(\roman*)},leftmargin=*,topsep=1pt,itemsep=1pt]
\item $q<1$: $\;\mathcal R_D=\bigl(\bar r_0^{\,1-q}+(1-q)(c+b)D\Delta t\bigr)^{1/(1-q)}$;
\item $q=1$: $\;\mathcal R_D=(1+c\Delta t)^D(r_0+b/c)-b/c$;
\item $q>1$, $b=0$: $\;\mathcal R_D=\rho\,\exp\bigl(C_q\,q^D\bigr)$ with $\rho=(c\Delta t)^{-1/(q-1)}$ and $C_q=\log\bigl(2^{1/(q-1)}(1+r_0/\rho)\bigr)$; for $b>0$, the same form holds with $c+b$ and $\bar r_0$ in place
of $c$ and $r_0$.
\end{enumerate}
The envelope is finite and nondecreasing in $D$.
\end{theorem}

The forward certified depth is the largest depth for which the reachable
radius remains representable,
$
D_F:=\max\{D\in\mathbb N:\mathcal R_D(q)<M\}.
$
Since $\mathcal R_D$ is nondecreasing, Theorem~\ref{thm:threshold}
can be inverted to give
\begin{equation}
D_F=
\left\lceil
\begin{cases}
\dfrac{M^{1-q}-\bar r_0^{\,1-q}}
{(1-q)(c+b)\Delta t}, & q<1,\\[8pt]
\dfrac{\log\bigl((M+b/c)/(r_0+b/c)\bigr)}
{\log(1+c\Delta t)}, & q=1,\\[8pt]
\log_q\dfrac{\log(M/\rho)}{C_q}, & q>1,
\end{cases}
\right\rceil-1.
\label{eq:DF}
\end{equation}
Thus the forward depth scales polynomially in $M$ for $q<1$, grows logarithmically as $\Theta(\log M)$ at $q=1$, and collapses to the double-logarithmic scale $\Theta(\log_q\log M)$ for $q>1$.

\subsection{Backward pass: local sensitivity and certification}
\label{sec:backward}

The forward exponent $q$ controls how far the state can reach, but it does
not by itself determine how large the block Jacobian can become on that
region. We therefore measure the worst-case local sensitivity over the
reachable states. For $R>0$, define
\begin{equation}
L(R):=
\sup_{\theta\in\Theta,\ \|x\|\le R}
\|\nabla_Xv(x;\theta)\|_2 .
\end{equation}
Assuming $\nabla_Xv$ is continuous in $(x,\theta)$ and $\Theta$ is compact,
$L(R)$ is finite for every finite $R$. Since every state satisfies
$\|X_k\|\le\mathcal R_D(q)$ for $k\le D$ by Theorem~\ref{thm:threshold},
every Jacobian along a depth-$D$ trajectory is bounded by
$L(\mathcal R_D(q))$. The adjoint recursion therefore gives
\begin{equation}
\|p_i\|
\le
\bigl(1+\Delta t\,L(\mathcal R_D(q))\bigr)^{D-i}\Lip(G),
\qquad 0\le i\le D. 
\label{eq:backward-bound}
\end{equation}

\begin{theorem}[Backward depth certificate]
\label{thm:backward}
Under Assumption~\ref{ass:growth}, with each $v(\cdot;\theta)$ continuously
differentiable, if
\begin{equation}
L(\mathcal R_D(q))
<
\frac{(M/\Lip(G))^{1/D}-1}{\Delta t},
\label{eq:Lcond}
\end{equation}
then $\|p_i\|<M$ for every $0\le i\le D$.
\end{theorem}

The proof is in App.~\ref{app:backward-proof}. The right-hand side of \eqref{eq:Lcond} is the allowable Jacobian budget at
depth $D$; it decreases as $D$ increases. We define $D_\nabla$ as the largest
depth satisfying \eqref{eq:Lcond}. To characterize how this certificate scales with depth, we determine how the local sensitivity $L(R)$ grows with the reachable radius.

Assumption~\ref{ass:growth} controls the growth of the state, but it does not
specify how that state growth is amplified by the Jacobian. A global
Lipschitz bound would set $L(R)=O(1)$, but, as discussed in Sec.~\ref{sec:intro}, that
restriction excludes important state-dependent blocks. We therefore use a more general power-law bound that allows Jacobian sensitivity to grow with the reachable radius:
\begin{assumption}[Uniform Jacobian growth]
\label{ass:lip}
Each $v(\cdot;\theta)$ is $C^1$, and there are $\sigma\ge0$ and
$\ell_1,\ell_0\ge0$ such that
\begin{equation}
L(R)\le \ell_1R^\sigma+\ell_0
\qquad\text{for all }R\ge1.
\label{eq:jac-growth}
\end{equation}
\end{assumption}

We call $\sigma$ the \emph{Jacobian sensitivity exponent}. Thus $\sigma=0$
recovers the global-Lipschitz case, whereas $\sigma>0$ allows the local
sensitivity to increase with the state radius. Affine maps, Lipschitz
activations, and bounded-weight MLPs have $\sigma=0$; pre-normalization also
gives $\sigma=0$, while gated MLPs and unnormalized self-attention have
$\sigma=1$ and $\sigma=2$, respectively (App.~\ref{app:sigma-arith}, Lemma~\ref{lem:attention-sigma}).
For self-attention, the $\sigma=2$ growth arises from the bilinear
query-key logits $S=X^\top W_Q^\top W_KX/\sqrt{d_k}$ and the
value-dependent Jacobian $W_VX\,\partial_X\mathrm{softmax}(S)$ (Lemma~\ref{lem:attention-sigma}).

Assumption~\ref{ass:lip} implies Assumption~\ref{ass:growth} with $q=\sigma+1$ (Proposition~\ref{prop:sigma-bounds-q}), so every block has $q\le\sigma+1$: $\sigma=0$ forces $q\le1$, and $q>1$ forces $\sigma\ge q-1>0$. Otherwise the two exponents are independent (examples after Proposition~\ref{prop:sigma-bounds-q}).

Under Assumption~\ref{ass:lip}, combining
\eqref{eq:backward-bound} with Theorem~\ref{thm:threshold} gives, with
$\bar{\mathcal R}_D=\max\{\mathcal R_D,1\}$,
\begin{equation}
\log\frac{\|p_0\|}{\Lip(G)}\le D\log\bigl(1+\Delta t(\ell_0+\ell_1)\bigr)+\sigma D\log\bar{\mathcal R}_D,
\qquad
\sigma D\log\bar{\mathcal R}_D=
\begin{cases}
O(\sigma D\log D), & q<1,\\
O(\sigma D^2), & q=1,\\
O(\sigma D\,q^D), & q>1,
\end{cases}
\label{eq:envelope}
\end{equation}
Requiring the right-hand side to remain below
$\log(M/\Lip(G))$ gives, on the informative range $0\le q\le\sigma+1$,
\begin{equation}
D_\nabla\gtrsim
\begin{cases}
\dfrac{\log M}{\log\bigl(1+\Delta t(\ell_0+\ell_1)\bigr)}, & \sigma=0\ (\text{then }q\le1),\\[8pt]
\dfrac{1-q}{\sigma}\dfrac{\log M}{\log\log M}, & \sigma>0,\ q<1,\\[8pt]
\sqrt{\dfrac{2\log M}{\sigma\log(1+c\Delta t)}}, & \sigma>0,\ q=1,\\[8pt]
\log_q\log M\,(1-o(1)), & \sigma>0,\ q>1.
\end{cases}
\label{eq:Dnabla}
\end{equation}
Thus, $
D_{\rm cert}:=\min\{D_F,D_\nabla\}$, as summarized in Tab.~\ref{tab:q-sigma}. 
The estimate \eqref{eq:Dnabla} is asymptotic in $M \gg 1$; App.~\ref{app:overflow}
gives the underlying implicit condition and computes $D_F$, $D_\nabla$
and $D_{\rm cert}$ across $q$, $\sigma$ and $M$
(Tab.~\ref{tab:overflow}).
The exponents $q$ and $\sigma$ need not vary together.
Crucially, $q=1$ is a sharp threshold: a block with $q>1$ has $\sigma\ge q-1>0$, and its certified depth is $D_{\rm cert}=\Theta(\log_q\log M)$,
because the doubly exponential growth of $\mathcal R_D$ is inherited by the
backward bound, however small $\sigma$ is.

\section{From stability exponents to architecture design}
\label{sec:exponents}

Sec.~\ref{sec:threshold} characterizes the growth classes of the forward and backward states in terms of $q$ and $\sigma$. For architecture design, however, it is important to determine a priori how $q$ grows across the primitives of a residual block and which primitives increase its state-dependent sensitivity, rather than estimating these exponents empirically during or after training (App.~\ref{app:qhat}). We therefore develop arithmetic rules for the input-magnitude exponents of individual operations and use the exponents of standard primitives to obtain an upper bound on the exponents of the entire residual block.

\paragraph{Input-magnitude exponent arithmetic.}

To obtain an a priori bound for a composite residual block, we propagate
the input-magnitude exponent through its constituent operations. Let
$q(B)$ denote the infimum of the exponents in
Assumption~\ref{ass:growth} for a primitive $B$
(App.~\ref{app:arith}). 

\begin{theorem}[Arithmetic of input-magnitude exponents]
\label{thm:arith}
Let $B_1,B_2$ be primitives satisfying Assumption~\ref{ass:growth}, with
$q_i=q(B_i)$ and compatible dimensions. Let $P$ be a bounded bilinear
map satisfying $|P(u,w)|\le|P|,|u|,|w|$, including the Hadamard
product $\odot$ and matrix product, for which $|P|\le1$ under the
Frobenius norm. Then
\begin{enumerate}[label=\textnormal{(\roman*)},leftmargin=*,topsep=1pt,itemsep=0pt]
\item \emph{composition:} $q(B_1\circ B_2)\le q_1q_2$;
\item \emph{sum:} $q(B_1+B_2)\le\max(q_1,q_2)$;
\item \emph{bilinear product:} $q\bigl(P(B_1,B_2)\bigr)\le q_1+q_2$, so
$q(B_1\odot B_2)\le q_1+q_2$;
\item \emph{residual wrapping:} $q(I+\Delta t B_1)\le\max(1,q_1)$;
\item \emph{normalization:} $q(N\circ B)=q(B\circ N)=0$ for
$N\in\{\LN,\mathrm{RMSNorm},\mathrm{ScaleNorm}\}$.
\end{enumerate}
\end{theorem}

The proof is in App.~\ref{app:proof:thm:3.1}.

Restricting $q\le1$ controls the forward growth class but does not
prevent the Jacobian from increasing with the state radius. We therefore
track the Jacobian sensitivity exponent at the same primitive level. For a
$C^1$ primitive $B$, let
\[
J_B(R):=\sup_{\|x\|\le R}\|\nabla B(x)\|_2,
\]
and let $\sigma(B)$ be the infimum of the $\sigma\ge0$ for which
\[
J_B(R)\le \ell_1R^\sigma+\ell_0,\qquad R\ge1,
\]
as in \eqref{eq:jac-growth}.

\begin{theorem}[Arithmetic of Jacobian sensitivity exponents]
\label{thm:sigma-arith}
Let $B_1,B_2$ be $C^1$ primitives with
$q_i=q(B_i)$ and $\sigma_i=\sigma(B_i)$, and let $P$ be as in
Theorem~\ref{thm:arith}. Then
\begin{enumerate}[label=\textnormal{(\roman*)},leftmargin=*,topsep=1pt,itemsep=0pt]
\item \emph{composition:} $\sigma(B_1\circ B_2)\le\sigma_1q_2+\sigma_2$;
\item \emph{sum:} $\sigma(B_1+B_2)\le\max(\sigma_1,\sigma_2)$;
\item \emph{bilinear product:}
$\sigma(P(B_1,B_2))\le\max(\sigma_1+q_2,q_1+\sigma_2)$, so
$\sigma(B_1\odot B_2)\le\max(\sigma_1+q_2,q_1+\sigma_2)$;
\item \emph{residual wrapping:} $\sigma(I+\Delta t B_1)\le\sigma_1$;
\item \emph{pre-normalization:} $\sigma(B\circ N)=0$ for
$N\in\{\LN,\mathrm{RMSNorm},\mathrm{ScaleNorm}\}$.
\end{enumerate}
\end{theorem}

The proof is in App.~\ref{app:sigma-arith}. The rules of both arithmetics induce
\emph{associative} operations on the exponents: a residual velocity can
therefore be decomposed into primitives, each carrying the exponents in
Tab.~\ref{tab:primitives}, and the block exponents obtained by applying
Theorems~\ref{thm:arith} and~\ref{thm:sigma-arith} are independent of
the parenthesization of its computation graph (App.~\ref{app:wellposed}).

\paragraph{Exponents of standard primitives.}

The arithmetic reduces the exponent bounds of a residual block to those
of its constituent primitives. Tab.~\ref{tab:primitives} lists the
$(q,\sigma)$ exponents of common building blocks, from which block-level
bounds follow by the two arithmetics. The corresponding coefficients are
given in App.~\ref{app:catalogue}, with architecture-specific derivations
in App.~\ref{app:exponents}.

\begin{table}[h]
\centering
\caption{Input-magnitude exponents $q$ and Jacobian sensitivity exponents
$\sigma$ of common architectural primitives. Normalizations use
$\epsilon>0$; ReLU and ReLU6 are replaced by $C^1$ regularizations with the same exponents, such as softplus; -- indicates
that the exponent is not derived. Attention is multi-token, with
$q=1$ understood in the worst-token sense
(App.~\ref{app:exponents}). $^\dagger$: top-$k$ routing is
discontinuous, so Assumption~\ref{ass:lip} admits no finite $L(R)$ on a ball
meeting a routing boundary and $\sigma$ does not exist; $q$ is unaffected,
since Assumption~\ref{ass:growth} asks only for a growth bound. Each
composite entry names the normalizations that are \emph{removed} and those
that are \emph{retained}; unless an entry says otherwise, the exponents are
those of the residual branch with the residual-stream normalization removed.}
\label{tab:primitives}
\footnotesize\setlength{\tabcolsep}{4pt}
\begin{tabular}{@{}>{\raggedright\arraybackslash}p{0.72\linewidth}cc@{}}
\toprule
Primitive & $q$ & $\sigma$ \\
\midrule
Saturating activation (tanh, sigmoid, ReLU6 \citep{howard2017mobilenets}) & $0$ & $0$ \\
DyT $\gamma\tanh(\alpha x)+\beta$ \citep{zhu2025dyt} & $0$ & $0$ \\
BHyT \citep{byun2026bhyt} & $0$ & -- \\
Softmax, row-stochastic \citep{bridle1990softmax} & $0$ & $0$ \\
LayerNorm \citep{ba2016layer} ($\epsilon>0$), RMSNorm \citep{zhang2019root}, ScaleNorm \citep{nguyen2019scalenorm} & $0$ & $0$ \\
ViT-22B \citep{dehghani2023scaling} & $0$ & $0$ \\
\midrule
Affine map $x\mapsto Wx+b$ & $1$ & $0$ \\
Lipschitz activation (ReLU, GELU \citep{hendrycks2016gelu}, SiLU \citep{elfwing2018silu}) & $1$ & $0$ \\
MLP $W_L\phi_L(\cdots W_1x)$ with bounded weights and Lipschitz $\phi_\ell$ & $1$ & $0$ \\
Query-key logits $X^{\top}W_Q^{\top}W_KX/\sqrt{d_k}$ & $2$ & $1$ \\
Self-attention $\mathrm{softmax}(QK^\top/\sqrt{d_k})V$ \citep{vaswani2017attention} & $1$ & $2$ \\
Multi-head attention $\sum_hW_O^{(h)}\mathrm{Attn}^{(h)}(X)$ & $1$ & $2$ \\
Self-attention with normalized query-key input & $1$ & $1$ \\
\midrule
Entrywise power $x\mapsto x^p$, $p\ge2$ & $p$ & $p-1$ \\
Selective SSM, Mamba, normalization removed \citep{gu2024mamba} & $5$ & $\le5$ \\
\midrule
\multicolumn{3}{@{}l}{\emph{post-Transformer branches, residual-stream normalization removed (App.~\ref{app:post-transformer})}} \\
RWKV-4 time mixing \citep{peng2023rwkv} & $1$ & $1$ \\
RWKV-4 channel mixing, squared ReLU retained \citep{peng2023rwkv} & $2$ & $2$ \\
RWKV-7 time mixing, in-block normalizations retained \citep{peng2025rwkv7} & $3$ & $3$ \\
RetNet retention, in-block GroupNorm retained \citep{sun2023retnet} & $1$ & $3$ \\
RetNet retention, GroupNorm and score rescaling also removed & $4$ & $3$ \\
Hyena operator of order $m$ \citep{poli2023hyena} & $m{+}1$ & $m$ \\
BDH layer map as released, in-block and Post-LN retained \citep{kosowski2025dragon} & $0$ & $3$ \\
BDH residual branch, all LayerNorms removed & $4$ & $3$ \\
Jamba, top-$k$ MoE layers \citep{lieber2024jamba} & $2$ & --$^\dagger$ \\
\bottomrule
\end{tabular}
\end{table}

\paragraph{From symbolic exponents to architecture-level certification.}

The exponent arithmetic provides a systematic procedure for certifying residual architectures (App.~\ref{app:certification}). First, identify the blockwise exponents $(q,\sigma)$ from the architecture's computation graph using primitive-level arithmetic. Since $q>1$ severely restricts certified depth (Tab.~\ref{tab:q-sigma}), the primary design principle is to enforce $q\le1$ by modifying  primitives responsible for superlinear growth, as demonstrated in Sections~\ref{sec:tsf} and~\ref{sec:foundation}. Backward sensitivity can limit stability even at $q=1$: our GPT-2 experiments illustrate how reducing $\sigma$ from $2$ to $1$ can restore stability even when $q=1$ (Sec.~\ref{sec:lm}).
Beyond these qualitative regimes, the coefficients in the growth and Jacobian bounds determine quantitative stability margins. The use of QK normalization alongside LayerNorm in large-scale architectures such as ViT-22B \citep{dehghani2023scaling} suggests that controlling these coefficients may become important for stability at scale.

\section{Results}
\label{sec:results}
The experiments investigate how the architectural exponents $q$ and $\sigma$ govern stability across time-series forecasting, operator learning, and language modeling, and whether different certified stability regimes offer task-dependent performance benefits
(Tab.~\ref{tab:setup}, App.~\ref{app:setup}). We first test the $q=1$ threshold and depth scaling in Mamba and compare stable PatchTST designs. We then assess task performance, including OOD extrapolation, across $q=0$ and $q=1$ FNO architectures. GPT-2 isolates the role of backward sensitivity $\sigma$, while a Toto~2.0 proxy tests whether removing superlinear growth restores stability and improves forecasting performance.

\subsection{Time-series forecasting}
\label{sec:tsf}

In long-horizon time-series forecasting, a model observes a multivariate series over a fixed input window, here $96$ time steps, and predicts its values over the next $H$ steps.
We evaluate on two standard benchmarks, ETTm1 \citep{zhou2021informer} (7 channels, 15-minute sampling) and Weather \citep{wu2021autoformer} (21 channels, 10-minute sampling), with $H\in\{96,192,336,720\}$.
To study the effect of the residual exponent, we use two backbones, Mamba \citep{gu2024mamba} in the CMamba forecaster \citep{zeng2024cmamba} and Transformer-based PatchTST \citep{nie2023patchtst}, with $D\in\{3,8\}$ residual updates. Each backbone is evaluated with RMSNorm, Pre-LN, or Peri-LN \citep{zhang2019root,xiong2020layer,kim2025periln}, or without normalization (no-LN), and Mamba is additionally evaluated with the linear-growth modification described below.
We use 6 seeds per cell on ETTm1 and 3 on Weather (App.~\ref{app:tsf}).

\paragraph{Only the $q>1$ block blows up, and more often at larger depth.}
Without normalization, Transformer attention and MLP blocks have $q=1$,
whereas the selective state-space block of Mamba has $q=5$ because
input-dependent quantities are multiplied together. Normalizing only the
input to the selective-parameter branch and gating its output with a sigmoid
reduces the Mamba block to $q=1$ without introducing a normalization layer;
we call this the linear-growth block. With normalization, all blocks have
$q=0$ (App.~\ref{app:tsf}). Tab.~\ref{tab:blowup} reports the fraction
of runs that blow up, defined as either a nonfinite loss or a final test MSE
above $10^{12}$ on standardized targets, along with the maximum empirical exponent $\hat q_{\max}$  (App.~\ref{app:tsf:results}). For $q>1$, Tab.~\ref{tab:q-sigma}
certifies only $O(\log_q\log M)$ layers ($D_{\rm cert}=1$ at $q=5$ in FP32, below both depths used), and the observed blow-up rate rises with depth.
Consistent with this threshold, only no-LN Mamba ($q=5$) blows up, in
$67\%$ of runs, with the rate increasing from $39\%$ at $D=3$ to $94\%$ at
$D=8$.

\begin{table}[h]
\centering
\caption{Blow-up rate (\%) of forecasting runs per cell, per dataset (``all''), and over both datasets (``total''). Blow-ups occur only for $q>1$, and their rate increases with depth. %
The empirical exponent $\hat q_{\max}$, estimated using \eqref{eq:empirical:q}, approaches the symbolic exponent $q$, notably reaching $5.02$ for Mamba no-LN ($q=5$).  Run counts and empirical exponent statistics are given in App.~\ref{app:tsf:results}.
}
\label{tab:blowup}
\scriptsize
\setlength{\tabcolsep}{3.1pt}
\begin{tabular}{@{}llccc cccc>{\bfseries}c cccc>{\bfseries}c c@{}}
\toprule
& & & & & \multicolumn{5}{c}{ETTm1 (6 seeds / cell)} & \multicolumn{5}{c}{Weather (3 seeds / cell)} & \\
\cmidrule(lr){6-10}\cmidrule(lr){11-15}
Backbone & Variant & $q$ & $\qhat_{\max}$ & $D$ & 96 & 192 & 336 & 720 & all & 96 & 192 & 336 & 720 & all & total \\
\midrule
\multirow{2}{*}{Mamba} & \multirow{2}{*}{no-LN} & \multirow{2}{*}{5} & \multirow{2}{*}{$5.02$} & 3 & 33 & 50 & 67 & 50 & 50 & 0 & 33 & 0 & 33 & 17 & 39 \\
 & & & & 8 & 100 & 100 & 100 & 100 & 100 & 33 & 100 & 100 & 100 & 83 & 94 \\
\midrule
Mamba & linear growth & 1 & $1.09$ & 3, 8 & 0 & 0 & 0 & 0 & 0 & 0 & 0 & 0 & 0 & 0 & 0 \\
PatchTST & no-LN & 1 & $1.43$ & 3, 8 & 0 & 0 & 0 & 0 & 0 & 0 & 0 & 0 & 0 & 0 & 0 \\
Mamba \& PatchTST & RMSNorm, Pre-LN, Peri-LN & 0 & $0.06$ & 3, 8 & 0 & 0 & 0 & 0 & 0 & 0 & 0 & 0 & 0 & 0 & 0 \\
\bottomrule
\end{tabular}
\end{table}

\begin{figure}[t]
\centering
\includegraphics[width=\linewidth]{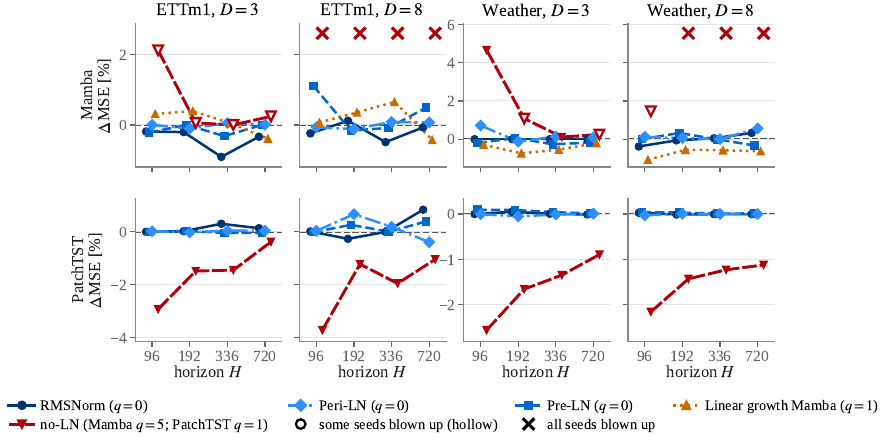}
\caption{Forecasting accuracy, shown as the signed relative deviation of each variant's mean test MSE from the per-cell median across variants, $\Delta\mathrm{MSE}=100(\mathrm{MSE}-\operatorname{med}(\mathrm{MSE}))/\operatorname{med}(\mathrm{MSE})$. The dashed line indicates the per-cell median; negative values indicate lower MSE. For PatchTST, $q=1$ is below the median in all cells, while for Mamba it is consistently lower on Weather and comparable to the normalized variants on ETTm1.}
\label{fig:tsf}
\end{figure}

\paragraph{Among stable blocks, $q=1$ tends to be more accurate.}
Fig.~\ref{fig:tsf} compares the test MSE of the stable variants. For each backbone, dataset, depth, and prediction horizon, we take the median MSE across variants as the reference and report relative deviations from it. For PatchTST, the no-LN variant ($q=1$) is below this reference in all cells, while for Mamba it is consistently lower on Weather and comparable to the normalized variants on ETTm1.
A possible explanation is that a $q=1$ update can scale with the state, whereas a $q=0$ update is uniformly bounded in it. This interpretation is not isolated by the experiment, however, since removing normalization also changes the optimization and removes its affine parameters.

\subsection{Operator learning}
\label{sec:fno}

\paragraph{Problem.}
We learn the solution operator for transient heat conduction in a 3-D battery module \citep{mathworksbattery}. Given the ambient temperature, convection coefficient on the cooled faces, and heat generation per cell, $\theta=(T_{\rm amb},h,Q)$, the model predicts the temperature field $T(\cdot,600\,{\rm s})$. Specifically, the governing equation is
\begin{align*}
\rho c_p\,\partial_t T
&=\nabla\!\cdot(K\nabla T)+Q\,\mathbf 1_{\Omega_{\rm cell}},\\
-n\cdot K\nabla T
&=h(T-T_{\rm amb})
\quad\text{on }\Gamma_{\rm conv},
\qquad
n\cdot K\nabla T=0
\quad\text{on }\partial\Omega\setminus\Gamma_{\rm conv}
\end{align*}
with initial condition  $T(x,0)=T_{\rm amb}$.
For fixed $h$, the solution operator is affine in $T_{\rm amb}$ and $Q$, whereas its dependence on $h$ is generally nonlinear, as $h$ enters the boundary operator.
FNO \citep{li2021fno} takes these parameters as fields on a $32^3$ grid and predicts the final temperature field.
We use two parameter configurations for training. The \textbf{original box} is $[280,300]\,{\rm K}\times[10,20]\,{\rm W/(m^2K)}\times[10,20]\,{\rm W}$, with $216$ training solves. The \textbf{larger box} is $[270,310]\times[10,40]\times[10,40]$, sampled at the same grid spacing and giving $2{,}816$ solves. In the four residual Fourier layers, $X_{k+1}=X_k+\mathrm{GELU}(\mathrm{LN}(\mathcal K_kX_k+W_kX_k))$, LN gives $q=0$ and removing it gives $q=1$ (Fig.~\ref{fig:fno-arch}).

\paragraph{The stable $q=1$ design improves task performance in extrapolation.} 
Fig.~\ref{fig:fno} 
shows comparable in-distribution accuracy for LN ($q=0$) and no-LN ($q=1$), but generally lower out-of-distribution (OOD) errors for no-LN, with the gap widening as extrapolation distance increases. Expanding the training box at the same sampling density reduces extrapolation errors along $T_{\rm amb}$ while widening the LN--no-LN gap along $Q$. This advantage also persists for a nonlinear solution operator with a temperature-dependent source term (App.~\ref{app:fno} and \ref{app:fno:results}), extending beyond the original operator's affine dependence on $T_{\rm amb}$ and $Q$. These results suggest that OOD performance depends not only on training coverage but also on architectural flexibility: unlike the bounded updates of $q=0$, the $q=1$ regime permits residual updates that scale with the state, potentially improving extrapolation.

\begin{figure}[h]
\centering
\includegraphics[width=\linewidth]{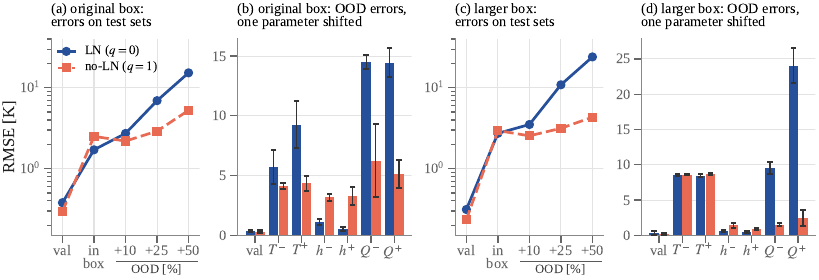}
\caption{FNO test RMSE for LN ($q=0$) and no-LN ($q=1$). The original
parameter box is used in (a,b), and the larger box at the original sampling
density in (c,d). (a,c) Test sets ordered by distance from the training data:
validation, off-grid points within the training box, and OOD points at
$10$, $25$, and $50\%$ beyond the box. (b,d) One parameter shifted $50\%$
beyond its range, with the others fixed at mid-range; whiskers show $\pm1$
sd. The two variants are comparable within the training box, while no-LN
achieves lower RMSE outside the box, with the gap increasing with
extrapolation distance and training-box size.}
\label{fig:fno}
\end{figure}


\subsection{Language modeling}
\label{sec:lm}

We train nanoGPT GPT-2 124M \citep{radford2019language,karpathy2022nanogpt} on OpenWebText \citep{gokaslan2019openwebtext}, comprising approximately 8 million training documents and 9 billion tokens, to examine stability at language-modeling scale. The model has 12 Transformer layers, each consisting of an attention and an MLP residual update ($D=24$). Without normalization, both blocks have $q=1$, but their sensitivities differ: $\sigma=2$ for attention and $\sigma=0$ for the MLP. We compare three variants: \emph{Pre-LN}, which gives $(q,\sigma)=(0,0)$; \emph{no-LN}, which retains attention sensitivity $\sigma=2$; and \emph{no-LN + QK-input norm}, which reduces attention sensitivity to $\sigma=1$ while preserving $q=1$. The last is a variant of query-key normalization \citep{henry2020qknorm,dehghani2023scaling,wortsman2024smallscale} that applies RMS normalization to the residual-stream input before the query and key projections, while leaving the value projection and residual stream unnormalized; Sec.~\ref{app:related} relates these variants and explains why the exponent arithmetic does not distinguish them. The attention sensitivity exponents are derived in App.~\ref{app:attention-sigma}. We train each variant for up to $20{,}000$ iterations with 5 seeds (App.~\ref{app:owt}).

\paragraph{$q=1$ is not enough: sensitivity exponent $\sigma$ should be controlled.}
Fig.~\ref{fig:owt} shows that the no-LN model, despite satisfying $q=1$, develops sharply amplified attention logits near iteration $6{,}750$, while its forward state remains comparable in magnitude to the stable variants (Fig.~\ref{fig:owt}a,b). During subsequent training, attention logits grow much faster than the forward state, preceding divergence (Fig.~\ref{fig:owt}c; depth-resolved increments in Fig.~\ref{fig:owt}). This indicates that the dominant instability is associated with state-dependent attention sensitivity ($\sigma=2$), rather than excessive forward state growth. The instability emerges early in training, before the model completes one epoch over the 9-billion-token corpus. All no-LN runs diverge, whereas Pre-LN ($\sigma=0$) and QK-input norm ($\sigma=1$) train successfully across all five seeds (Fig.~\ref{fig:owt}). This illustrates the critical $q=1$ regime in Table~\ref{tab:q-sigma}, where stability remains sensitive to $\sigma$ and may depend on depth and training conditions, including data scale. Our exponent arithmetic narrows the scope of trial-and-error stabilization by identifying which architectural sensitivities warrant further scrutiny.
In terms of predictive performance, the LN-based variant also achieves better validation performance than the unnormalized alternatives (Fig.~\ref{fig:owt-qknorm}
in App.~\ref{app:owt}).

\begin{figure}[h]
\centering
\includegraphics[width=\linewidth]{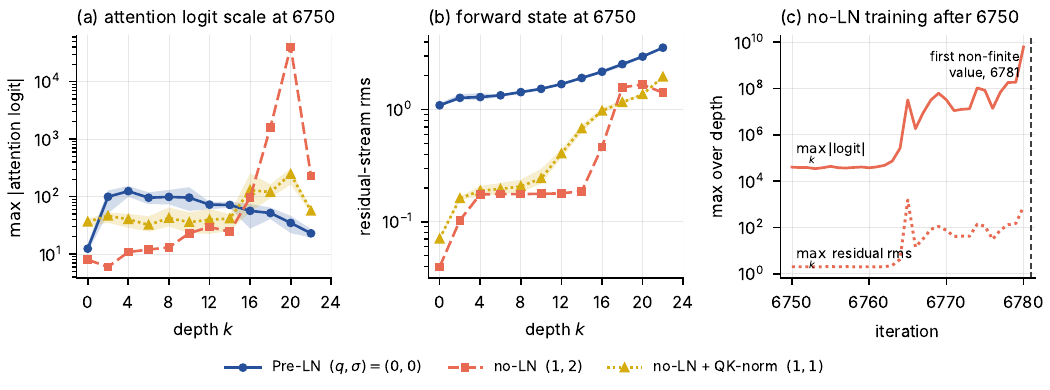}
\caption{
GPT-2 124M on OpenWebText.
Alternating attention and MLP updates give $(q,\sigma)=(0,0)$ with Pre-LN and $(q,\sigma)=(1,2)$ without normalization. 
(a,b) Near the onset of instability (iteration 6,750), no-LN exhibits a sharp increase in attention logits while its forward state remains comparable to the stable variants. (c) During the subsequent training steps, attention logits grow rapidly ahead of the forward state, preceding divergence (Fig.~\ref{fig:owt-replay}). QK-input norm reduces attention sensitivity from $\sigma=2$ to $\sigma=1$ and prevents divergence across all five seeds (Fig.~\ref{fig:owt-qknorm}).
}
\label{fig:owt}
\end{figure}

\subsection{Foundation models}
\label{sec:foundation}

Toto~2.0 \citep{khwaja2026toto} is a time-series foundation model with released checkpoints up to $2.5$B parameters. We train a $10.5$M-parameter proxy from scratch using the released recipe, retaining $12$ Transformer layers with alternating self-attention and gated-MLP residual updates ($D=24$). Without normalization, attention has $(q,\sigma)=(1,2)$, while the gated MLP $(W_gX)\odot\mathrm{SiLU}(W_uX)$ has $(q,\sigma)=(2,1)$. The released architecture applies RMSNorm to each branch input; we remove it to investigate the source of training instability. Five variants independently modify attention and gated-MLP normalization to separate the effects of $q$ and $\sigma$ (Fig.~\ref{fig:attention_mlp_toto}). Each is trained for up to $5{,}000$ steps with five seeds on a $3.73$ GB, $28$-dataset subset of the GIFT-Eval pretraining corpus \citep{aksu2024gifteval}, which its authors construct to carry no leakage into the GIFT-Eval evaluation split; the official Toto~2.0 mixture is proprietary and is not used (App.~\ref{app:toto}).

\paragraph{Superlinear growth dominates stability in the Toto~2.0 proxy.}
Consistent with the sharp depth restriction for $q>1$
(Tab.~\ref{tab:q-sigma}), all $q=2$ variants diverge across five seeds, whereas all $q\leq1$ variants complete training (Fig.~\ref{fig:attention_mlp_toto}).
Reducing $\sigma$ from $2$ to $1$ at $q=2$ delays divergence but does not prevent it, highlighting the dominant role of superlinear forward growth. Empirical growth estimates further support the predicted exponent classes 
(Tab.~\ref{tab:toto-qhat}). Among the stable designs, the two $q=1$ variants outperform the normalized $q=0$ baseline in both in-distribution and OOD forecasting, improving CRPS by $1.0$--$1.1\%$ and $1.5$--$1.8\%$, respectively (Tab.~\ref{tab:toto-arms}).

\begin{figure}[h]
\centering
\includegraphics[width=\linewidth]{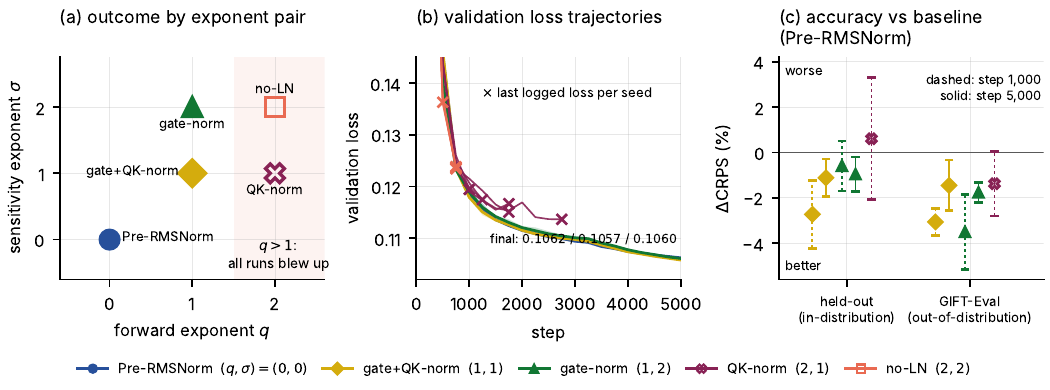}
\caption{Training stability and forecasting performance of a Toto~2.0 proxy across five architectural variants. \textbf{(a)} Training outcomes by theoretical exponent pair $(q,\sigma)$: all $q=2$ variants diverge even with $\sigma=1$, while all $q\leq1$ variants complete training, consistent with the backward sensitivity classes in Tab.~\ref{tab:q-sigma}.  Filled and open markers indicate completed and diverged runs. \textbf{(b)} Validation loss trajectories: seed means with min--max bands for completed variants, and individual trajectories with crosses at the last logged loss for diverged runs. \textbf{(c)} CRPS relative to Pre-RMSNorm on held-out (in-distribution) and GIFT-Eval (OOD) data. The $q=1$ variants generally outperform the $q=0$ baseline (Tab.~\ref{tab:toto-arms}). Error bars show seed standard deviations; dashed/solid lines indicate steps $1{,}000$/$5{,}000$.
}
\label{fig:attention_mlp_toto}
\end{figure}

\section{Conclusion and Discussion}
\label{sec:discussion}

We developed a framework for certifying the forward and backward stability of deep residual architectures from their computational primitives. By separating state growth $q$ from Jacobian sensitivity $\sigma$, we identified critical exponents that induce phase transitions in allowable depth: crossing $q=1$ leads to severe depth restrictions under superlinear growth, while $\sigma$ further governs backward stability. Exponent arithmetic translates these bounds into architecture-level certificates, including for heterogeneous residual stacks, without requiring a uniform block design across depth. This enables the design of new architectures and targeted diagnosis and modification of existing state-of-the-art backbones, with our experiments validating the predicted transitions in stability across exponent regimes.

The resulting framework does not prescribe a single optimal architecture. Instead, it narrows the scope of trial-and-error in architecture design by identifying families of architectures with certified stability, allowing design choices to be guided by task performance. Across time-series forecasting and operator learning, our experiments support the predicted stability distinctions and show that less restrictive $q=1$ designs can outperform normalized $q=0$ counterparts on selected tasks.

Our analysis provides a priori certificates against forward and backward numerical blow-up, but does not establish optimizer convergence or guarantee predictive performance. Rather than eliminating trial and error, our framework makes architectural exploration more systematic by substantially narrowing the space of designs requiring empirical evaluation.  Our framework could also be combined with analyses of contractive properties in LayerNorm-based Transformers to further refine stability certificates. 
Accounting for the stochasticity of learned weights could provide more refined estimates of architecture-dependent coefficients beyond uniform deterministic bounds. Incorporating these mechanisms offers directions for extending the framework toward a more comprehensive characterization of architectural stability.


\bibliography{Arxiv_later_versions_2026/references}
\bibliographystyle{plain}

\clearpage
\appendix

\section{Forward and Backward Stability: Proofs and Quantitative Bounds}
\label{app:proofs}

Throughout, $a(x):=c\|x\|^q+b$ denotes the growth bound of Assumption~\ref{ass:growth}, $r_k=\|X_k\|$, and a block \emph{saturates} at $x$ if $\|v(x;\theta)\|=a(x)$.

\subsection{Proof of Theorem~\ref{thm:threshold}}
\label{app:proof:thm2.2}

\emph{Step 1 (scalar comparison).} By the triangle inequality and Assumption~\ref{ass:growth}, $r_{k+1}\le r_k+\Delta t\,(cr_k^q+b)=:\phi(r_k)$. Since $\phi$ is nondecreasing on $[0,\infty)$, $r_k$ is bounded by the sequence that starts at $r_0$ and satisfies the recursion with equality.

\emph{Step 2 ($q=1$).} The recursion is affine, $r_{k+1}\le(1+c\Delta t)r_k+b\Delta t$; with $y_k=r_k+b/c$, $y_D\le(1+c\Delta t)^Dy_0$, which is (ii).

\emph{Step 3 ($q<1$).} Concavity of $t\mapsto t^{1-q}$ gives $(\alpha+h)^{1-q}\le\alpha^{1-q}+(1-q)\alpha^{-q}h$ for $\alpha>0$, $h\ge0$. With $\alpha=r_k$ and $h=c\Delta t\,r_k^q$ (case $b=0$), $r_{k+1}^{1-q}\le r_k^{1-q}+c(1-q)\Delta t$, which telescopes. For $b\ge0$, let $\bar r_k=\max(r_k,1)$. Then $b\le b\,\bar r_k^{\,q}$, so $r_{k+1}\le\bar r_k+(c+b)\Delta t\,\bar r_k^{\,q}$, and the right side is at least $1$, so $\bar r_{k+1}$ obeys the same inequality. This is the case $b=0$ with $c+b$ in place of $c$, which gives (i).

\emph{Step 4 ($q>1$).} Let $b=0$, $\rho=(c\Delta t)^{-1/(q-1)}$, and $u_k=1+r_k/\rho\ge1$. Since $c\Delta t\,\rho^{q-1}=1$, Step 1 gives $u_{k+1}\le u_k+(u_k-1)^q\le u_k+u_k^q\le2u_k^q$. With $s_k=\log u_k$, $s_{k+1}+\tfrac{\log2}{q-1}\le q\bigl(s_k+\tfrac{\log2}{q-1}\bigr)$, so $\log u_D\le q^D\bigl(\log u_0+\tfrac{\log2}{q-1}\bigr)$, and $r_D\le\rho u_D$ gives (iii). For $b>0$, the reduction of Step 3 applies.

\emph{Attainment.} The aligned block $v^\sharp(x)=a(x)\,x/\|x\|$ saturates the bound and is a positive multiple of $x$, so Step 1 holds with equality: $r_{k+1}=r_k+\Delta t\,(cr_k^q+b)$.
At $q=1$ this is (ii) with equality.
At $q\in[0,1)$ with $b=0$, the sequence is nondecreasing, and $r_{k+1}/r_k=1+c\Delta t\,r_k^{q-1}\le\kappa:=1+c\Delta t\,r_0^{q-1}$. Concavity with the tangent at $r_{k+1}$ gives $r_{k+1}^{1-q}-r_k^{1-q}\ge(1-q)r_{k+1}^{-q}(r_{k+1}-r_k)=(1-q)c\Delta t\,(r_k/r_{k+1})^q\ge(1-q)c\Delta t\,\kappa^{-q}$, so $r_D^{1-q}\ge r_0^{1-q}+(1-q)c\Delta t\,\kappa^{-q}D$, and $r_D$ grows like $D^{1/(1-q)}$, the rate of (i). Adding $b>0$ only increases the sequence.
At $q>1$, $r_{k+1}\ge c\Delta t\,r_k^q$, and Lemma~\ref{lem:dexp} of App.~\ref{app:escape} (with $\gamma=c\Delta t$) gives $\liminf_D\log\log r_D/D\ge\log q$ once $r_K\ge e\rho$, which happens at a finite $K$ because $r_k$ increases by at least $\Delta t(cr_0^q+b)$ per step; (iii) gives the matching $\limsup$. \qed

\subsection{Backward bound, proof of Theorem~\ref{thm:backward}, and its radius dependence}
\label{app:backward-proof}\label{app:backward-sigma}

The proof uses only $\|X_k\|\le\mathcal R_k$, which holds for every $q\ge0$ by Theorem~\ref{thm:threshold}; for $q>1$ the envelope is finite at every depth, so no restriction on $q$ is needed.
By \eqref{eq:adjoint}, $p_i=J_{i:D}^{\top}p_D$ with $J_{i:D}=\prod_{k=D-1}^{i}(I+\Delta t\,\nabla_Xv(X_k;\theta_k))$, and $\|p_D\|\le\Lip(G)$, the Lipschitz constant of $G(\cdot,y)$ on the ball of radius $\mathcal R_D$. Submultiplicativity, $\|I+\Delta t\,A\|_2\le1+\Delta t\|A\|_2$, and $\|\nabla_Xv(X_k;\theta_k)\|_2\le L(\mathcal R_k)\le L(\mathcal R_D)$ give
\begin{align}
\|p_i\|
&\;\overset{\text{(T)}}{\le}\;\prod_{k=i}^{D-1}\bigl(1+\Delta t\,\|\nabla_Xv(X_k;\theta_k)\|_2\bigr)\Lip(G)\nonumber\\
&\;\overset{\text{(A)}}{\le}\;\prod_{k=i}^{D-1}\bigl(1+\Delta t\,L(\mathcal R_k)\bigr)\Lip(G)
\;\le\;\bigl(1+\Delta t\,L(\mathcal R_D)\bigr)^{D-i}\Lip(G),
\label{eq:backward-chain}
\end{align}
whose last member is \eqref{eq:backward-bound}. The trajectory bound (T) uses the Jacobians actually visited, an a posteriori quantity; step (A) replaces each by its worst case over $\theta\in\Theta$ and the ball of radius $\mathcal R_k$, an outer bound on the reachable region. Hence $\|\nabla_Xv(X_k;\theta_k)\|_2\le L(\mathcal R_k)\le L(\mathcal R_D)\le\sup_{R}L(R)$, from trajectory-specific to global, and $L(\mathcal R_D)$ can be loose by many orders of magnitude (Sec.~\ref{sec:lm}). For $v(x)=Lx$ with $L\le c$, the block satisfies Assumption~\ref{ass:growth} at $q=1$, $\nabla_Xv\equiv LI$, and every inequality is an equality for $G(x,y)=\langle e,x\rangle$ with a unit vector $e$, so the certificate is sharp.

\emph{Proof of Theorem~\ref{thm:backward}.} $(1+\Delta t\,L(\mathcal R_D))^{D}\Lip(G)<M$ if and only if \eqref{eq:Lcond} holds, and $(1+\Delta t\,L(\mathcal R_D))^{D-i}\le(1+\Delta t\,L(\mathcal R_D))^{D}$, so \eqref{eq:backward-bound} gives $\|p_i\|<M$ for every $i$. \qed

\emph{Budget.} With $\Lip(G)=1$ in FP32 ($\log M=88.7$) at $\Delta t=1$, the right side of \eqref{eq:Lcond} is $39.3$ at $D=24$ and $1.52$ at $D=96$. Since $\mathcal R_D$ and $L(\mathcal R_D)$ are nondecreasing in $D$ and the budget decreases, \eqref{eq:Lcond} holds for every $D\le D_\nabla$.

\emph{Layerwise refinement.} The middle member of \eqref{eq:backward-chain} replaces $L(\mathcal R_D)$ by the smaller $L(\mathcal R_k)$ in each layer. It certifies at least the same depth, improves the constant at $q=1$ by a factor $\sqrt2$ (Proposition~\ref{prop:backward-sigma}), and is what Sec.~\ref{sec:lm} evaluates.

\emph{Radius dependence.} The chain \eqref{eq:backward-chain} is linear in $D$
only when $L$ is bounded. The next proposition converts the envelope of
Theorem~\ref{thm:threshold} into the depth at which the backward bound reaches
$M$, in each growth regime.

\begin{proposition}[Radius dependence of the backward bound]
\label{prop:backward-sigma}
Let Assumptions~\ref{ass:growth} and~\ref{ass:lip} hold. With $\mathcal R_k$ as in Theorem~\ref{thm:backward}:
\begin{enumerate}[label=\textnormal{(\roman*)},leftmargin=*,topsep=1pt,itemsep=1pt]
\item ($q<1$) $\log(\|p_0\|/\Lip(G))\le D\log(1+\Delta t(\ell_0+\ell_1))+\frac{\sigma}{1-q}D\log\bigl(\bar r_0^{\,1-q}+(1-q)(c+b)D\Delta t\bigr)$, which stays below $\log M$ for $D\le\frac{1-q}{\sigma}\frac{\log M}{\log\log M}(1+o(1))$ as $M\to\infty$;
\item ($q=1$) $\log(\|p_0\|/\Lip(G))\le D\bigl[\log(1+\Delta t(\ell_0+\ell_1))+\sigma\log^{+}(r_0+b/c)\bigr]+\frac{\sigma}{2}\log(1+c\Delta t)\,D(D-1)$, which stays below $\log M$ for $D\le\sqrt{2\log M/(\sigma\log(1+c\Delta t))}\,(1+o(1))$;
\item ($q>1$, $b=0$) with $\rho=(c\Delta t)^{-1/(q-1)}$ and $C_q=\log\bigl(2^{1/(q-1)}(1+r_0/\rho)\bigr)>0$, $\log(\|p_0\|/\Lip(G))\le D\bigl[\log(1+\Delta t(\ell_0+\ell_1))+\sigma\log^{+}\rho\bigr]+\sigma C_q\,\frac{q^D-1}{q-1}$, which stays below $\log M$ for $D\le\log_q\log M-K$, with $K$ depending only on $q,\sigma,c\Delta t,r_0,\ell_0,\ell_1$; for $b>0$ the same holds with $c+b$ and $\bar r_0$ in place of $c$ and $r_0$.
\end{enumerate}
\end{proposition}

\begin{proof}
Put $\bar{\mathcal R}_k=\max(\mathcal R_k,1)$. For $R\ge1$, $1+\Delta t(\ell_1R^{\sigma}+\ell_0)\le(1+\Delta t(\ell_0+\ell_1))R^{\sigma}$, so $\log(1+\Delta t\,L(\mathcal R_k))\le\log(1+\Delta t(\ell_0+\ell_1))+\sigma\log\bar{\mathcal R}_k$; sum over $k<D$ in the middle member of \eqref{eq:backward-chain}. At $q<1$, $\bar{\mathcal R}_k\le\mathcal R_D$ since $\bar r_0\ge1$. At $q=1$, $\log\bar{\mathcal R}_k\le k\log(1+c\Delta t)+\log^{+}(r_0+b/c)$ and $\sum_{k<D}k=D(D-1)/2$. At $q>1$, $\log\bar{\mathcal R}_k\le\log^{+}\rho+q^kC_q$ by Theorem~\ref{thm:threshold}(iii), and $\sum_{k<D}q^k=(q^D-1)/(q-1)$; with $D=\log_q\log M-K$ the last term is at most $\sigma C_q\log M/((q-1)q^K)$, below $\tfrac12\log M$ for $K$ large, and the linear term is $O(\log\log M)$. The right sides of (i) and (ii) are $\frac{\sigma}{1-q}D\log D(1+o(1))$ and $\frac{\sigma}{2}\log(1+c\Delta t)D^2(1+o(1))$; solving for the depth at which each equals $\log M$ gives the ranges.
\end{proof}

Against the linear $D\log(1+\Delta t\,L)$ of a bounded Jacobian, the exponent of the backward bound is $D\log D$ at $q<1$, quadratic in $D$ at $q=1$, and exponential in $D$ at $q>1$; the last requires $\sigma>0$, not $q>1$ alone. Both are upper envelopes; equality requires $L(R)\asymp R^{\sigma}$ and Jacobians aligned along the trajectory. At $\sigma=2$, (ii) in FP32 at $c=\Delta t=1$ gives about $\sqrt{2\cdot88.7/(2\log2)}\approx11$ layers, against $128$ for a bounded Jacobian with $L=1$.

\begin{proposition}[An attained exponent forces radius dependence]
\label{prop:attained-sigma}
Let Assumptions~\ref{ass:growth} and~\ref{ass:lip} hold, and suppose that for some $\theta\in\Theta$, unit vector $e$, and $c'>0$, $b'\ge0$, $\|v(se;\theta)\|\ge c's^q-b'$ for all $s>0$. Then $L(R)\ge c'R^{q-1}-(b+b')/R$ for all $R>0$. In particular $q>1$ attained along a ray implies $L(R)\to\infty$ with $\sigma\ge q-1$ in \eqref{eq:jac-growth}. Without attainment, Assumption~\ref{ass:growth} gives no lower bound on $L$.
\end{proposition}

\begin{proof}
By the mean value inequality along the segment from $0$ to $Re$, $\|v(Re;\theta)\|\le\|v(0;\theta)\|+L(R)R\le b+L(R)R$. With the lower bound at $s=R$, $c'R^q-b'\le b+L(R)R$. For the last claim, $v\equiv0$ satisfies Assumption~\ref{ass:growth} for every $q$ and has $L\equiv0$.
\end{proof}

\begin{proposition}[The Jacobian growth exponent bounds the growth exponent]
\label{prop:sigma-bounds-q}
Let Assumption~\ref{ass:lip} hold with $\Theta$ compact, and let
$b_0=\sup_{\theta\in\Theta}\|v(0;\theta)\|$. Then for all $x$ and
$\theta\in\Theta$,
\[
\|v(x;\theta)\|\le b_0+L(1)+\ell_0\|x\|+\frac{\ell_1}{\sigma+1}\|x\|^{\sigma+1},
\]
so Assumption~\ref{ass:growth} holds with $q=\sigma+1$,
$c=\ell_0+\ell_1/(\sigma+1)$ and $b=b_0+L(1)+\ell_0$. In particular $\sigma=0$
implies $q\le1$, and under both assumptions the forward envelope may be computed
with $\min(q,\sigma+1)$ in place of $q$, with the constants above when
$\sigma+1<q$.
\end{proposition}

\begin{proof}
With $r=\|x\|$, the mean value inequality along the segment from $0$ to $x$
gives $\|v(x;\theta)-v(0;\theta)\|\le\int_0^rL(s)\,ds$. For $s\le1$,
$L(s)\le L(1)$, and for $s\ge1$, $L(s)\le\ell_1s^\sigma+\ell_0$; integrating
gives the display, and $r\le r^{\sigma+1}+1$ gives the constants.
\end{proof}

No attainment is needed: every block with Jacobian sensitivity exponent
$\sigma$ has growth exponent at most $\sigma+1$. Conversely, if
$\|v(se;\theta)\|\gtrsim s^q$ along a ray, Proposition~\ref{prop:attained-sigma}
gives $\sigma\ge q-1$. The bound $q\le\sigma+1$ is the only constraint between
the exponents. It is attained: $v(x)=\|x\|^\sigma x$ satisfies
Assumption~\ref{ass:lip} with exponent $\sigma$ and Assumption~\ref{ass:growth}
with no $q<\sigma+1$. Conversely, Assumption~\ref{ass:growth} constrains no
Jacobian: $v(x)=\sin(\|x\|^k)\,e$, with $e$ a fixed unit vector, has $q=0$ for
every $k>1$, while its Jacobian grows like $\|x\|^{k-1}$. The informative range
is therefore $0\le q\le\sigma+1$; on $q<\sigma+1$ neither assumption implies the
other, and on $q=\sigma+1$ the growth bound follows from the Jacobian bound up
to constants.

\subsection{Overflow depths}
\label{app:overflow}

The asymptotic estimates in \eqref{eq:Dnabla} describe how certified depth scales with the floating-point limit $M$. For specific formats, the forward depth $D_F$ follows explicitly from Theorem~\ref{thm:threshold}, while substituting Assumption~\ref{ass:lip} into the reachable-radius envelope $\mathcal R_D(q)$ \eqref{eq:backward-bound} with Theorem~\ref{thm:threshold} gives the implicit backward condition~\eqref{eq:envelope}. Since $\mathcal R_D(q)$ is nondecreasing and the right-hand side decreases with $D$ for $M>\Lip(G)$, the largest feasible depth can be computed by bisection.

Tab.~\ref{tab:q-sigma-exact} gives the exact forms behind the leading-order
entries of Tab.~\ref{tab:q-sigma}, with every constant carried through.

\begin{table}[h]
\centering
\caption{
Exact expressions behind the leading-order entries of Tab.~\ref{tab:q-sigma},
as functions of the forward growth exponent $q$ and the Jacobian sensitivity
exponent $\sigma$, from Theorem~\ref{thm:threshold} and \eqref{eq:Dnabla}. The
exponent $q=1$ is the critical case: the forward envelope is exponential in
depth at $q=1$ and double-exponential above it, so $D_{\rm cert}$ drops from a
power of $\log M$ to $\log_q\log M$ across that row. $M$: largest finite
floating-point value; $D_F$: forward depth; $D_\nabla$: backward depth;
$D_{\rm cert}=\min\{D_F,D_\nabla\}$. The $\sigma=0$ row is the global-Lipschitz
baseline. See Sec.~\ref{sec:threshold}.
}
\label{tab:q-sigma-exact}
\small
\setlength{\tabcolsep}{4pt}
\begin{tabular}{@{}ll l l l@{}}
\toprule
$q$ & $\sigma$ & $D_F$ & $D_\nabla\ge$ & $D_{\rm cert}$ \\
\midrule
$<1$ & $>0$ & (i) $\frac{M^{1-q}}{(1-q)(c+b)\Delta t}$ & $\frac{1-q}{\sigma}\,\frac{\log M}{\log\log M}$ & $\ge\frac{1-q}{\sigma}\,\frac{\log M}{\log\log M}$ \\[3pt]
$=1$ & $>0$ & (ii) $\frac{\log M}{\log(1+c\Delta t)}$ & $\sqrt{\frac{2\log M}{\sigma\log(1+c\Delta t)}}$ & $\ge\sqrt{\frac{2\log M}{\sigma\log(1+c\Delta t)}}$ \\[3pt]
$>1$ & $>0$ & $\log_q\log M+O(1)$ & $\log_q\log M\,(1-o(1))$ & $\sim\log_q\log M$ \\[3pt]
\arrayrulecolor{gray}\midrule\arrayrulecolor{black}
$\le1$ & $0$ & (i) if $q<1$, (ii) if $q=1$ & $\frac{\log M}{\log(1+\Delta t(\ell_0+\ell_1))}$ & $\ge\frac{\log M}{\log(1+\Delta t(\ell_0+\ell_1))}$ \\
\bottomrule
\end{tabular}
\end{table}

We evaluate this implicit condition with $\Delta t=c=r_0=\Lip(G)=\ell_1=1$ and $b=\ell_0=0$, computing the certified depths by bisection for each $(q,\sigma)$ and floating-point format (Tab.~\ref{tab:overflow}). The same procedure can be applied to any specified parameter values to obtain the corresponding finite-precision certificates.

\begin{table}[h]
\centering
\caption{Certified depths from the implicit condition~\eqref{eq:envelope}, at $\Delta t=c=r_0=\Lip(G)=\ell_1=1$ and
$b=\ell_0=0$, so that $L(R)=R^\sigma$. $D_F$ is the forward depth
\eqref{eq:DF} and does not depend on $\sigma$; $D_\nabla$ is the backward
depth; $D_{\rm cert}=\min\{D_F,D_\nabla\}$. Each is the largest certified
depth, one less than the depth at which the corresponding bound first reaches
$M$. Each $q$ is listed with the three smallest \emph{admissible} exponents,
$\sigma\in\{\max(0,q-1),\max(0,q-1)+1,\max(0,q-1)+2\}$: by
Proposition~\ref{prop:attained-sigma} an attained $q$ forces $\sigma\ge q-1$, so
pairs below that line, such as $(2,0)$ and $(5,\sigma)$ for $\sigma\le3$, are
realized by no block and are omitted. Where $\sigma=0$ is admissible the
Jacobian bound is constant and $D_\nabla$ is independent of $q$. A zero entry
means the condition already fails at $D=1$, so the certificate is vacuous in
that format. Entries use the exact IEEE maxima.}
\label{tab:overflow}
\small
\setlength{\tabcolsep}{3.5pt}
\begin{tabular}{@{}ll rrr rrr rrr@{}}
\toprule
& & \multicolumn{3}{c}{FP16} & \multicolumn{3}{c}{FP32 / BF16} & \multicolumn{3}{c}{FP64} \\
\cmidrule(lr){3-5}\cmidrule(lr){6-8}\cmidrule(l){9-11}
& & \multicolumn{3}{c}{$M=65{,}504$} & \multicolumn{3}{c}{$M=3.40\!\times\!10^{38}$} & \multicolumn{3}{c}{$M=1.80\!\times\!10^{308}$} \\
$q$ & $\sigma$ & $D_F$ & $D_\nabla$ & $D_{\rm cert}$ & $D_F$ & $D_\nabla$ & $D_{\rm cert}$ & $D_F$ & $D_\nabla$ & $D_{\rm cert}$ \\
\midrule
\multirow{3}{*}{$0$} & $0$ & \multirow{3}{*}{$65{,}502$} & $15$ & $15$ & \multirow{3}{*}{$3.4\!\times\!10^{38}$} & $127$ & $127$ & \multirow{3}{*}{$1.8\!\times\!10^{308}$} & $1{,}023$ & $1{,}023$ \\
 & $1$ &  & $5$ & $5$ &  & $26$ & $26$ &  & $142$ & $142$ \\
 & $2$ &  & $3$ & $3$ &  & $15$ & $15$ &  & $80$ & $80$ \\
\addlinespace[2pt]
\multirow{3}{*}{$\tfrac12$} & $0$ & \multirow{3}{*}{$509$} & $15$ & $15$ & \multirow{3}{*}{$3.7\!\times\!10^{19}$} & $127$ & $127$ & \multirow{3}{*}{$2.7\!\times\!10^{154}$} & $1{,}023$ & $1{,}023$ \\
 & $1$ &  & $4$ & $4$ &  & $18$ & $18$ &  & $92$ & $92$ \\
 & $2$ &  & $3$ & $3$ &  & $11$ & $11$ &  & $53$ & $53$ \\
\addlinespace[2pt]
\multirow{3}{*}{$1$} & $0$ & \multirow{3}{*}{$15$} & $15$ & $15$ & \multirow{3}{*}{$127$} & $127$ & $127$ & \multirow{3}{*}{$1{,}023$} & $1{,}023$ & $1{,}023$ \\
 & $1$ &  & $3$ & $3$ &  & $11$ & $11$ &  & $31$ & $31$ \\
 & $2$ &  & $2$ & $2$ &  & $7$ & $7$ &  & $22$ & $22$ \\
\addlinespace[2pt]
\multirow{3}{*}{$2$} & $1$ & \multirow{3}{*}{$2$} & $1$ & $1$ & \multirow{3}{*}{$5$} & $3$ & $3$ & \multirow{3}{*}{$8$} & $6$ & $6$ \\
 & $2$ &  & $1$ & $1$ &  & $3$ & $3$ &  & $5$ & $5$ \\
 & $3$ &  & $1$ & $1$ &  & $2$ & $2$ &  & $5$ & $5$ \\
\addlinespace[2pt]
\multirow{3}{*}{$5$} & $4$ & \multirow{3}{*}{$1$} & $0$ & $0$ & \multirow{3}{*}{$2$} & $1$ & $1$ & \multirow{3}{*}{$4$} & $2$ & $2$ \\
 & $5$ &  & $0$ & $0$ &  & $1$ & $1$ &  & $2$ & $2$ \\
 & $6$ &  & $0$ & $0$ &  & $1$ & $1$ &  & $2$ & $2$ \\
\bottomrule
\end{tabular}
\end{table}

\begin{remark}[Finite precision: a crossover at $q=1$]
\label{rem:crossover}
The orders of Tab.~\ref{tab:q-sigma} are leading orders for $M\gg1$; the
power of $\log M$ in $D_{\rm cert}$ jumps at $q=1$ only in that limit. At finite
$M$ the jump is a crossover. Along the envelope of
Theorem~\ref{thm:threshold} (with $b=0$, $r_0=1$), $\log\mathcal R_k$ grows
linearly in $k$, as at $q=1$, while $|q-1|\,c\Delta t\,k\lesssim1$, so
$\xi_q=1/(c\Delta t\,|q-1|)$ is a depth scale. To leading order the exponent of
the layerwise backward bound (Proposition~\ref{prop:backward-sigma}) has the
form $\sigma\log(1+c\Delta t)\,D^2\,\Psi\bigl((1-q)c\Delta t\,D\bigr)$, with
$\Psi(y)=\bigl[(1+y)\log(1+y)-y\bigr]/y^2$ and $\Psi(0)=\tfrac12$. Hence
$D_\nabla=D_c\,\Phi\bigl((1-q)c\Delta t\,D_c\bigr)$ with
$D_c=\bigl(2\log M/(\sigma\log(1+c\Delta t))\bigr)^{1/2}$, the $q=1$ depth of
the layerwise bound, and $\Phi$ defined by $\Phi^2\Psi(z\Phi)=\tfrac12$; blocks
with $|q-1|\lesssim(\sigma/(c\Delta t\log M))^{1/2}$, for which $\xi_q$ exceeds
$D_c$, are certified as at $q=1$. This is the analogue of finite-size scaling,
with depth as the system size. The corrections are of relative order
$(\log M)^{-1/2}$ and are not small in FP32 or FP64.
\end{remark}

\section{Supercritical blocks are driven to the worst case}
\label{app:escape}\label{app:escape:worstcase}

The envelopes of Sec.~\ref{sec:threshold} are upper bounds, so $q>1$ removes
a certificate without by itself forcing anything to happen. This appendix
supplies the missing direction for the class considered here: under the
discrete maximum principle, a supercritical block is not merely permitted to
escape but is driven to its envelope by the optimal control, so at $q>1$ the
double-exponential rate is a statement about the trajectory and not only about
the bound.

\begin{theorem}[Supercritical blocks are driven to the worst case]
\label{thm:escape}
Let $q>1$ and $R_\star=(2/(c\Delta t))^{1/(q-1)}$.
\begin{enumerate}[label=\textnormal{(\roman*)},leftmargin=*,topsep=1pt,itemsep=1pt]
\item If the blocks saturate at every step and $r_k\ge eR_\star$, then $r_j\ge R_\star\exp(q^{\,j-k})$ for all $j\ge k$, whatever the direction of $v$; hence $D^\ast(M)\le k+\lceil\log_q\log(M/R_\star)\rceil$.
\item Let $v_0^\ast,\dots,v_{D-1}^\ast$ minimize $\mathbb E\,G(X_D,y)$ over all $C^1$ fields satisfying the bound of Assumption~\ref{ass:growth}, under the regularity conditions of Proposition~\ref{prop:pmp}, and let $\bar p_{k+1}(x)=\mathbb E[p_{k+1}\mid X_k=x]$ along the optimal trajectory. Then, for almost every state $x$ with $\bar p_{k+1}(x)\ne0$, $v_k^\ast(x)=-a(x)\,\bar p_{k+1}(x)/\|\bar p_{k+1}(x)\|$; in particular $\|v_k^\ast(x)\|=a(x)$.
\end{enumerate}
\end{theorem}

Part (ii) is a discrete Pontryagin maximum principle: the optimal block saturates, so part (i) applies on every layer where the conditional costate is nonzero; a parametric family inherits this when it can represent the saturating field on the data. At $c=\Delta t=1$ and $q\ge2$, $eR_\star\le2e$, so typical inputs meet the hypothesis of (i).

\paragraph{Proof of part (i).}
Two lemmas do the work: the first turns a superlinear recursion into doubly exponential growth, the second shows that once $r_k\ge R_\star$ a saturated step cannot reduce the radius, whatever the direction of $v$.

\begin{lemma}[Doubly exponential growth]
\label{lem:dexp}
Let $q>1$, $\gamma>0$, $\rho=\gamma^{-1/(q-1)}$, and let $r_{j+1}\ge\gamma r_j^q$ for $j\ge K$ with $r_K\ge e\rho$. Then $r_j\ge\rho\exp(q^{\,j-K})$ for all $j\ge K$.
\end{lemma}

\begin{proof}
With $z_j=\log(r_j/\rho)$ and $\gamma\rho^{q-1}=1$, the hypothesis reads $z_{j+1}\ge qz_j$ with $z_K\ge1$; induction gives $z_j\ge q^{\,j-K}$.
\end{proof}

\begin{lemma}[One saturated step]
\label{lem:step}
Let $q>1$ and $\|v(X_k)\|=a(X_k)$. Then $|\Delta t\,a(X_k)-r_k|\le r_{k+1}\le\Delta t\,a(X_k)+r_k$, with the lower end attained when $v(X_k)$ points to the origin and the upper end when it points away. Hence $r_{k+1}\ge r_k$ for every direction if and only if $\Delta t\,a(X_k)\ge2r_k$, which for $b=0$ means $r_k\ge R_\star$; and whenever $r_k\ge R_\star$, $r_{k+1}\ge\tfrac12c\Delta t\,r_k^q+\Delta t\,b$.
\end{lemma}

\begin{proof}
The two bounds are the reverse and direct triangle inequalities for $X_k+\Delta t\,v(X_k)$. The lower bound is at least $r_k$ iff $\Delta t\,a(X_k)\ge2r_k$, and $\Delta t\,cr_k^{q-1}\ge2$ iff $r_k\ge R_\star$. If $r_k\ge R_\star$, then $r_k\le\tfrac12c\Delta t\,r_k^q$, so $r_{k+1}\ge c\Delta t\,r_k^q+\Delta t\,b-r_k\ge\tfrac12c\Delta t\,r_k^q+\Delta t\,b$.
\end{proof}

\begin{proof}[Proof of Theorem~\ref{thm:escape}(i)]
By Lemma~\ref{lem:step}, whenever $r_j\ge R_\star$ we have $r_{j+1}\ge\tfrac12c\Delta t\,r_j^q\ge R_\star$, because $r\mapsto\tfrac12c\Delta t\,r^q$ is increasing and equals $R_\star$ at $R_\star$. So the trajectory stays above $R_\star$, and Lemma~\ref{lem:dexp} with $\gamma=\tfrac12c\Delta t$, for which $\gamma^{-1/(q-1)}=R_\star$, gives $r_j\ge R_\star\exp(q^{\,j-k})$. This exceeds $M$ once $q^{\,j-k}\ge\log(M/R_\star)$.
\end{proof}

The threshold $R_\star$ applies even to blocks that point toward the origin. For the inward field $v(x)=-c\|x\|^{q-1}x$, one step multiplies $X_k$ by $1-\tau_k$ with $\tau_k=c\Delta t\,r_k^{q-1}$. The state decays to the origin from below $R_\star$ ($\tau_k<2$), but from above $R_\star$ it overshoots the origin and escapes doubly exponentially.

\paragraph{Discrete maximum principle and proof of part (ii).}
Let $\mathcal V^q$ be the set of $C^1$ fields $u:\mathbb{R}^N\to\mathbb{R}^N$ with $\|u\|\le a$, and consider
\begin{equation}
\min_{v_0,\dots,v_{D-1}\in\mathcal V^q}\ \mathbb E\,G(X_D,y),\qquad X_{k+1}=X_k+\Delta t\,v_k(X_k),\quad (X_0,y)\sim\mu_0 .
\label{eq:oc-discrete}
\end{equation}
The set $\mathcal V^q$ is convex and contains every $C^1$ field bounded by $a$. We assume: (P1) $b>0$; (P2) a minimizer $v_0^\ast,\dots,v_{D-1}^\ast$ exists, with trajectory $X_k^\ast$ whose law is $\mu_k$; (P3) $\mu_0$ has compact support; (P4) $G(\cdot,y)$ is $C^1$ with $\nabla_XG$ bounded on compact sets.

\begin{proposition}[Discrete maximum principle]
\label{prop:pmp}
Under (P1)--(P4), let $p_k$ follow \eqref{eq:adjoint} along the optimal trajectory and $\bar p_{k+1}(x)=\mathbb E[p_{k+1}\mid X_k^\ast=x]$. Then $\int\langle\bar p_{k+1}(x),u(x)-v_k^\ast(x)\rangle\,d\mu_k(x)\ge0$ for every $u\in\mathcal V^q$, and for $\mu_k$-a.e.\ $x$ with $\bar p_{k+1}(x)\ne0$,
\begin{equation*}
v_k^\ast(x)=-a(x)\,\frac{\bar p_{k+1}(x)}{\|\bar p_{k+1}(x)\|},\qquad\text{so}\qquad\|v_k^\ast(x)\|=a(x).
\end{equation*}
\end{proposition}

\begin{proof}
\emph{Step 1 (co-state).} With $\Psi_D=I$ and $\Psi_j=\Psi_{j+1}(I+\Delta t\,\nabla_Xv_j^\ast(X_j^\ast))$, the recursion \eqref{eq:adjoint} gives $p_j=\Psi_j^{\top}\nabla_XG(X_D^\ast,y)$.

\emph{Step 2 (first variation).} Fix $u\in\mathcal V^q$ and $\varepsilon\in[0,1]$, and set $v_k^\varepsilon=v_k^\ast+\varepsilon(u-v_k^\ast)\in\mathcal V^q$, leaving the other layers unchanged. Then $X_{k+1}^\varepsilon-X_{k+1}^\ast=\varepsilon\Delta t\,(u-v_k^\ast)(X_k^\ast)$, and propagation through the remaining $C^1$ layers gives $X_D^\varepsilon-X_D^\ast=\varepsilon\Delta t\,\Psi_{k+1}(u-v_k^\ast)(X_k^\ast)+o(\varepsilon)$. By (P3) and $\|u\|\le a$, all perturbed trajectories lie in one compact set, on which $\nabla_XG$ and $\nabla_Xv_j^\ast$ are bounded; dominated convergence then gives
\begin{equation*}
\frac{d}{d\varepsilon}\Big|_{0^+}\mathbb E\,G(X_D^\varepsilon,y)=\Delta t\,\mathbb E\langle p_{k+1},(u-v_k^\ast)(X_k^\ast)\rangle=\Delta t\int\langle\bar p_{k+1}(x),u(x)-v_k^\ast(x)\rangle\,d\mu_k(x),
\end{equation*}
where the second equality conditions on $X_k^\ast$. Optimality makes the right derivative nonnegative.

\emph{Step 3 (localization).} For $\varphi\in C^1$ with $0\le\varphi\le1$, the competitor $v_k^\ast+\varphi(u-v_k^\ast)$ is admissible, so the inequality holds with the integrand multiplied by $\varphi$. Letting $\varphi\uparrow\mathbf 1_B$ for open balls $B$ and applying the Besicovitch differentiation theorem to $\mu_k$ gives $\langle\bar p_{k+1}(x),u(x)-v_k^\ast(x)\rangle\ge0$ for $\mu_k$-a.e.\ $x$.

\emph{Step 4 (saturation).} Let $(e_m)$ be dense in the unit sphere, and let $a_n\in C^1$ with $a_n\le a$ and $a_n\uparrow a$ pointwise (take $a_n=a$ when $a$ is $C^1$, i.e., $q=0$ or $q>1$; for $0<q\le1$ take $a_n(x)=c(\|x\|^2+n^{-1})^{q/2}-cn^{-q/2}+b$). Step 3 with the countably many $u=a_ne_m\in\mathcal V^q$ gives, for $\mu_k$-a.e.\ $x$, $\langle\bar p_{k+1}(x),v_k^\ast(x)\rangle\le a_n(x)\langle\bar p_{k+1}(x),e_m\rangle$ for all $n,m$; letting $n\to\infty$ and minimizing over $m$ yields $\langle\bar p_{k+1}(x),v_k^\ast(x)\rangle\le-a(x)\|\bar p_{k+1}(x)\|$. Since $\|v_k^\ast(x)\|\le a(x)$, Cauchy--Schwarz forces $v_k^\ast(x)=-a(x)\bar p_{k+1}(x)/\|\bar p_{k+1}(x)\|$ when $\bar p_{k+1}(x)\ne0$.
\end{proof}

Theorem~\ref{thm:escape}(ii) is Proposition~\ref{prop:pmp}. On the event $\bar p_{j+1}(X_j^\ast)\ne0$ for $k\le j<D$, the optimal blocks saturate on layers $k,\dots,D-1$, and part (i) applies there. A parametric family $\{v(\cdot;\theta)\}$ is generally not convex, so Step 2 does not apply to it directly; the conclusion transfers when the family can represent the saturating field on the support of the data. Both factors of the conclusion can be monitored during training, as the ratio $\|v(X_k)\|/a(X_k)$ and the cosine between $v(X_k)$ and $-p_{k+1}$.

\section{Exponent arithmetic: proofs, primitives, and certification}
\label{app:arith}

\begin{definition}[Architectural primitives]
\label{def:primitive}
The set of \emph{primitives} is $\mathcal P:=\{B:\mathbb{R}^{N_{\rm in}}\to\mathbb{R}^{N_{\rm out}}\mid B\text{ continuous, and }\exists\,q,c,b\ge0:\ \|B(x)\|\le c\|x\|^q+b\ \forall x\}$, where $N_{\rm in},N_{\rm out}$ are fixed by the architectural context and $\|\cdot\|$ is the Euclidean (Frobenius) norm. For $B\in\mathcal P$, the \emph{input-magnitude exponent} $q(B)$ is the infimum of the $q\ge0$ for which such $c,b$ exist.
\end{definition}

\subsection{Proof of Theorem~\ref{thm:arith}}
\label{app:proof:thm:3.1}

Continuity is preserved by each operation, so it suffices to verify the growth bounds.
Because $q(B)$ is an infimum, we prove each rule for arbitrary exponents $q_i'>q_i$ (or $q_i'=q_i$ when attained) for which the bound $\|B_i(x)\|\le c_i\|x\|^{q_i'}+b_i$ holds; letting $q_i'\downarrow q_i$ gives the stated inequalities.
Write $\rho=\|x\|$.

(i) $\|B_1(B_2(x))\|\le c_1(c_2\rho^{q_2'}+b_2)^{q_1'}+b_1\le c_1\max(1,2^{q_1'-1})\bigl(c_2^{q_1'}\rho^{q_1'q_2'}+b_2^{q_1'}\bigr)+b_1$, using $(u+w)^p\le\max(1,2^{p-1})(u^p+w^p)$ for $u,w\ge0$, $p\ge0$. The leading coefficient for large $\rho$ is proportional to $c_1c_2^{q_1'}$.

(ii) $\|B_1(x)+B_2(x)\|\le c_1\rho^{q_1'}+c_2\rho^{q_2'}+b_1+b_2\le(c_1+c_2)\rho^{\max(q_1',q_2')}+c_1+c_2+b_1+b_2$, since $\rho^{p}\le\rho^{p'}+1$ for $0\le p\le p'$.

(iii) By bilinearity, $\|P(B_1(x),B_2(x))\|\le\|P\|\,\|B_1(x)\|\,\|B_2(x)\|\le\|P\|(c_1\rho^{q_1'}+b_1)(c_2\rho^{q_2'}+b_2)$, a sum of powers of $\rho$ with largest exponent $q_1'+q_2'$; the bound in (ii) then gives exponent $q_1'+q_2'$. The two instances used in this paper have $\|P\|\le1$ in the Frobenius norm: entrywise, $\|u\odot w\|_2^2=\sum_iu_i^2w_i^2\le\|u\|_2^2\|w\|_2^2$ for the Hadamard product, and $\|UW\|_F\le\|U\|_F\|W\|_F$ for the matrix product.

(iv) $\|x+\Delta t\,B_1(x)\|\le\rho+\Delta t(c_1\rho^{q_1'}+b_1)$. The identity has $q(I)=1$, and scaling by $\Delta t>0$ changes $c_1$ but not the exponent, so (ii) applies to the two summands and returns $\max(1,q_1')$.

(v) $\LN(z)=\gamma\odot(z-\mu\mathbf 1)/\sqrt{\sigma^2+\epsilon}+\beta$ with $\sigma^2=\|z-\mu\mathbf 1\|^2/d$, so $\|(z-\mu\mathbf 1)/\sqrt{\sigma^2+\epsilon}\|\le\sqrt d$ and $\|\LN(z)\|\le\|\gamma\|_\infty\sqrt d+\|\beta\|$ for every $z\in\mathbb{R}^d$. Applied tokenwise to a $d\times n$ input, the Frobenius norm is at most $\sqrt n$ times this constant. Hence $\LN\circ B$ is bounded, continuous for continuous $B$, and has exponent $0$. RMSNorm and ScaleNorm are bounded by $\|\gamma\|_\infty\sqrt d$ and $|g|$ in the same way. The second equality is (i) with $q(N)=0$: $q(B\circ N)\le q(B)\cdot0=0$, since $B$ is then evaluated on a bounded set. The $\sigma$-arithmetic is not symmetric in this way. Theorem~\ref{thm:sigma-arith}(v) gives $\sigma(B\circ N)=0$, whereas rule (i) of that theorem gives only $\sigma(N\circ B)\le\sigma_N q_B+\sigma_B=\sigma_B$: a normalization bounds the state it is applied to, but not the Jacobian of a block placed after it. This is why the repairs of Sections~\ref{sec:tsf} and~\ref{sec:foundation} normalize a branch argument rather than a branch output. \qed

\subsection{Proof of Theorem~\ref{thm:sigma-arith}}
\label{sec:sigma-arith}\label{app:sigma-arith}

\begin{proof}[Proof of Theorem~\ref{thm:sigma-arith}]
As in the proof of Theorem~\ref{thm:arith}, we prove each rule for exponents $q_i'>q_i$ and $\sigma_i'>\sigma_i$ for which $\|B_i(x)\|\le c_i\|x\|^{q_i'}+b_i$ and $J_{B_i}(R)\le\ell_{1,i}R^{\sigma_i'}+\ell_{0,i}$, and let them decrease to $q_i,\sigma_i$. Fix $R\ge1$ and $\|x\|\le R$; every polynomial in $R$ below is bounded by a constant times its leading power.

(i) $\nabla(B_1\circ B_2)(x)=\nabla B_1(B_2(x))\nabla B_2(x)$ and $\|B_2(x)\|\le\rho:=c_2R^{q_2'}+b_2$, so $J_{B_1\circ B_2}(R)\le J_{B_1}(\rho)J_{B_2}(R)\le(\ell_{1,1}\rho^{\sigma_1'}+\ell_{0,1})(\ell_{1,2}R^{\sigma_2'}+\ell_{0,2})$, of order $R^{\sigma_1'q_2'+\sigma_2'}$.

(ii) $J_{B_1+B_2}\le J_{B_1}+J_{B_2}$.

(iii) $\nabla P(B_1,B_2)(x)[\delta]=P(\nabla B_1(x)\delta,B_2(x))+P(B_1(x),\nabla B_2(x)\delta)$, so $J\le\|P\|\bigl(J_{B_1}(R)\,\|B_2(x)\|+\|B_1(x)\|\,J_{B_2}(R)\bigr)$, of order $R^{\max(\sigma_1'+q_2',\,q_1'+\sigma_2')}$. The growth bound $\|P(B_1,B_2)(x)\|\le\|P\|\,\|B_1(x)\|\,\|B_2(x)\|$ gives $q\le q_1+q_2$ as in Theorem~\ref{thm:arith}(iii).

(iv) $\|I+\Delta t\,\nabla B_1\|_2\le1+\Delta t\,J_{B_1}$.

(v) Each $N$ maps into a bounded set, $\|N(x)\|\le\rho_N$ (Theorem~\ref{thm:arith}(v)), and with $\epsilon>0$ its Jacobian is bounded on $\mathbb{R}^d$ by a constant $C_N$ of order $\|\gamma\|_\infty/\sqrt\epsilon$ (for ScaleNorm with $\epsilon$ added to $\|x\|$, $|g|/\epsilon$). Hence $J_{B\circ N}(R)\le J_B(\rho_N)\,C_N$ for every $R$, so $\sigma(B\circ N)=0$.
\end{proof}

The resulting primitive-level values are listed in Tab.~\ref{tab:primitives}; architecture-specific derivations, including self-attention and Mamba, are in App.~\ref{app:exponents}.

\subsection{Well-posedness of the arithmetic}
\label{app:wellposed}

The rules of Theorems~\ref{thm:arith} and~\ref{thm:sigma-arith} are applied by traversing a computation graph, so two questions must be settled before the exponents of a block are well defined: whether the value depends on the order in which a chain, a sum, or a product of several primitives is combined, and whether it depends on how finely the block is broken into primitives. The first has a negative answer, the second a positive one with a definite direction.

Write $x_i=(q_i,\sigma_i)\in[0,\infty)^2$ for the exponent pair of a primitive $B_i$, and let
\begin{align*}
x_1\circ x_2&:=(q_1q_2,\ \sigma_1q_2+\sigma_2),\qquad
x_1\oplus x_2:=\bigl(\max(q_1,q_2),\ \max(\sigma_1,\sigma_2)\bigr),\\
x_1\otimes x_2&:=\bigl(q_1+q_2,\ \max(\sigma_1+q_2,\ q_1+\sigma_2)\bigr)
\end{align*}
be the operations that the two theorems attach to composition, sum, and bilinear product.

\begin{proposition}[Associativity]
\label{prop:assoc}
The operations $\circ$, $\oplus$, and $\otimes$ are associative on $[0,\infty)^2$; $\oplus$ and $\otimes$ are commutative, and $\circ$ and $\oplus$ have units $(1,0)$ and $(0,0)$. The value of an $n$-fold chain, sum, or product is therefore independent of the parenthesization, and equals
\begin{align*}
\mathop{\bigcirc}_{i=1}^{n}x_i&=\Bigl(\prod_{i=1}^{n}q_i,\ \sum_{i=1}^{n}\sigma_i\prod_{j>i}q_j\Bigr),\qquad
\bigoplus_{i=1}^{n}x_i=\bigl(\max_iq_i,\ \max_i\sigma_i\bigr),\\
\bigotimes_{i=1}^{n}x_i&=\Bigl(\sum_iq_i,\ \max_i\bigl(\sigma_i+\textstyle\sum_{j\ne i}q_j\bigr)\Bigr).
\end{align*}
\end{proposition}

\begin{proof}
Put $M(x)=\bigl(\begin{smallmatrix}q&0\\ \sigma&1\end{smallmatrix}\bigr)$. Then $M(x_1\circ x_2)=M(x_1)M(x_2)$, so associativity of $\circ$ is associativity of matrix multiplication and the $n$-fold formula is the product $M(x_1)\cdots M(x_n)$; the unit is $M^{-1}(I)=(1,0)$, the pair of the identity map. For $\oplus$ both components are $\max$ on a totally ordered set, with unit $(0,0)$, the pair of the zero map. For $\otimes$ the first component is addition; for the second, both bracketings of three factors expand to $\max(\sigma_1+q_2+q_3,\,q_1+\sigma_2+q_3,\,q_1+q_2+\sigma_3)$, and the $n$-fold formula follows by induction and is symmetric in the factors. The operation $\circ$ is not commutative in its second component, which is immaterial because the graph fixes the order of composition.
\end{proof}

The unary rules, residual wrapping $x\mapsto(\max(1,q),\sigma)$ and normalization, take a single argument and raise no such question. For an $n$-fold product the underlying map must itself admit rebracketing for the statement to have content; the Hadamard and matrix products used here are associative, and for a general bilinear $P$ the graph fixes the bracketing.

\begin{proposition}[Monotone refinement]
\label{prop:refine}
Let $T$ be an expression tree whose leaves are primitives and whose nodes are the operations of Theorems~\ref{thm:arith} and~\ref{thm:sigma-arith}, let $F_T$ be the map it computes, and let $\hat x(T)=(\hat q(T),\hat\sigma(T))$ be the value obtained by evaluating the rules from the leaves to the root. Then
\begin{enumerate}[label=\textnormal{(\roman*)},leftmargin=*,topsep=1pt,itemsep=0pt]
\item \emph{soundness:} $q(F_T)\le\hat q(T)$ and $\sigma(F_T)\le\hat\sigma(T)$;
\item \emph{monotonicity:} each rule is nondecreasing in each of its arguments;
\item \emph{refinement:} if $T'$ is obtained from $T$ by replacing a subtree $S$ with a single leaf carrying $\bigl(q(F_S),\sigma(F_S)\bigr)$, then $\hat x(T')\le\hat x(T)$ componentwise.
\end{enumerate}
Consequently $q(F_T)\le\hat q(T')\le\hat q(T)$, and likewise for $\sigma$: a coarser decomposition whose atoms carry exactly evaluated exponents never gives a weaker bound than a finer one, and the sharpest exponents the arithmetic delivers are the infima of $\hat q$ and $\hat\sigma$ over admissible decompositions of the block.
\end{proposition}

\begin{proof}
(i) is Theorems~\ref{thm:arith} and~\ref{thm:sigma-arith} applied at each node, by induction on the height of $T$. (ii) is inspection: $q_1q_2$, $\max(q_1,q_2)$, $q_1+q_2$, $\sigma_1q_2+\sigma_2$, $\max(\sigma_1+q_2,q_1+\sigma_2)$, $\max(\sigma_1,\sigma_2)$, $\max(1,q_1)$ and $\sigma_1$ are all nondecreasing in each argument on $[0,\infty)$. For (iii), (i) applied to $S$ gives $\bigl(q(F_S),\sigma(F_S)\bigr)\le\hat x(S)$ componentwise, and (ii) propagates this inequality along the path from the replaced leaf to the root. As in the proof of Theorem~\ref{thm:arith}, exponents that are infima are handled by working with $q'>q$ and $\sigma'>\sigma$ and passing to the limit, which is legitimate because every rule is continuous in its arguments.
\end{proof}

\begin{remark}[The bound depends on the decomposition]
\label{rem:decomp}
The inequality in Proposition~\ref{prop:refine}(iii) can be strict, for three distinct reasons.

\emph{Cancellation under fan-out.} For complementary coordinate projections $\Pi_1,\Pi_2$ on $\mathbb R^d$, the block $F(x)=(\Pi_1x)\odot(\Pi_2x)$ vanishes identically, so $q(F)=0$, while Theorem~\ref{thm:arith}(iii) returns $1+1=2$. The rule bounds $\|B_1(x)\odot B_2(x)\|$ by $\|B_1(x)\|\,\|B_2(x)\|$ and so treats the factors as independent, although both are functions of the same $x$.

\emph{Loss of coefficient information.} The operation $\otimes$ has no unit: for an affine $B$ with $(q,\sigma)=(1,0)$ and the constant $C\equiv\mathbf 1$ with $(q,\sigma)=(0,0)$ we have $B\odot C=B$, yet Theorem~\ref{thm:sigma-arith}(iii) returns $\sigma\le\max(0+0,1+0)=1$. The underlying bound $J\le J_{B}\|C\|+\|B\|\,J_{C}$ is exact, with $J_C\equiv0$; what the exponent does not record is that $\sigma(C)=0$ holds with $\ell_1=\ell_0=0$. Inserting a factor that is numerically close to a constant, such as a saturated gate, therefore inflates the certified $\sigma$ by $q$. The attention values in Tab.~\ref{tab:primitives} come from Lemma~\ref{lem:attention-sigma} rather than from rule (iii) and are not affected.

\emph{Atoms outside $\mathcal P$.} Self-attention reaches $q=1$ because softmax is taken as a primitive with $q=0$, which annihilates the exponent $2$ of the logits under rule (i). Written out as $\exp(z)/\sum_j\exp(z_j)$, the numerator lies outside $\mathcal P$ and the quotient is not among the operations of Theorem~\ref{thm:arith}, so the finer decomposition returns no bound at all.

Certification therefore proceeds by coarsening: one collapses subexpressions whose exponents are known exactly until the rules suffice, which by Proposition~\ref{prop:refine} never weakens the result. The composite entries of Tab.~\ref{tab:primitives}, self-attention, multi-head attention, and the selective state-space block, are atoms of this kind; their tabulated exponents are themselves upper bounds obtained by direct calculation, so collapsing to them is safe but not necessarily sharp.
\end{remark}

\subsection{Primitive catalogue}
\label{app:catalogue}

Tab.~\ref{tab:primitives} in Sec.~\ref{sec:exponents} lists the exponents of common primitives, obtained from Theorems~\ref{thm:arith} and~\ref{thm:sigma-arith} or by direct calculation; Tab.~\ref{tab:primitive-coeff} gives the coefficients that accompany them.

\begin{table}[h]
\centering
\caption{Coefficients of the primitives of Tab.~\ref{tab:primitives}: $c$ is the growth coefficient of Assumption~\ref{ass:growth} ($M$: saturation level; biases omitted) and $\ell_1$ the coefficient of $R^\sigma$ in Assumption~\ref{ass:lip}, with $\kappa:=\|W_Q^{\top}W_K\|_{\rm op}$ and $\bar w^{(h)}_\bullet:=\|W^{(h)}_\bullet\|_{\rm op}$. At $\sigma=0$ we have $\ell_1=0$ and $\ell_0$ is the Lipschitz constant of the primitive ($\|W\|_{\rm op}$ for an affine map, of order $\|\gamma\|_\infty/\sqrt\epsilon$ for LayerNorm).}
\label{tab:primitive-coeff}
\footnotesize\setlength{\tabcolsep}{4pt}
\begin{tabular}{@{}>{\raggedright\arraybackslash}p{0.44\linewidth}ll@{}}
\toprule
Primitive & Coefficient $c$ & $\ell_1$ \\
\midrule
Saturating activation & $M\sqrt d$ & $0$ \\
DyT $\gamma\tanh(\alpha x)+\beta$ & $\|\gamma\|\sqrt d+|\beta|$ & $0$ \\
BHyT & data-dependent bound & -- \\
Softmax, row-stochastic & $1$ & $0$ \\
LayerNorm ($\epsilon>0$) & $\|\gamma\|_\infty\sqrt d+\|\beta\|$ & $0$ \\
RMSNorm $\gamma\odot x/\mathrm{RMS}(x)$ & $\|\gamma\|_\infty\sqrt d$ & $0$ \\
ScaleNorm $g\,x/\|x\|$ & $|g|$ & $0$ \\
\midrule
Affine map $x\mapsto Wx+b$ & $\|W\|_{\rm op}$ & $0$ \\
Lipschitz activation & $\Lip(\phi)$ & $0$ \\
MLP with Lipschitz $\phi_\ell$ & $\prod_\ell\Lip(\phi_\ell)\|W_\ell\|_{\rm op}$ & $0$ \\
Query-key logits & $\kappa/\sqrt{d_k}$ & $2\kappa/\sqrt{d_k}$ \\
Self-attention & $\|W_V\|_{\rm op}$ & $\|W_V\|_{\rm op}\kappa/\sqrt{d_k}$ \\
Multi-head attention & $\sum_h\bar w_O^{(h)}\bar w_V^{(h)}$ & $\sum_h\bar w_O^{(h)}\bar w_V^{(h)}\kappa_h/\sqrt{d_k}$ \\
\midrule
Entrywise power $x\mapsto x^p$, $p\ge2$ & $1$ & $p$ \\
Selective SSM, Mamba, normalization removed & \multicolumn{2}{l}{parameter-dependent} \\
\bottomrule
\end{tabular}
\end{table}

Super-polynomial maps lie outside $\mathcal P$, but composing them with a normalization returns the composite to $\mathcal P$ by Theorems~\ref{thm:arith} (v).
Discontinuous primitives such as vector quantization are excluded by Definition~\ref{def:primitive}.

\subsection{Self-attention under different normalization treatments}
\label{app:attention-sigma}

The sensitivity exponent of self-attention depends on which of the query, key, and value paths are normalized. We derive the corresponding $(q,\sigma)$ for the variants considered in Sec.~\ref{sec:lm} and App.~\ref{app:toto}, extending the unnormalized case in Tab.~\ref{tab:primitives}.

Let $X\in\mathbb{R}^{d\times n}$, with tokens as columns, and define
\[
Q=W_QX,\quad K=W_KX,\quad V=W_VX,\qquad
S=\frac{Q^\top K}{\sqrt{d_k}},
\]
\[
A=\mathrm{softmax}(S),\qquad
\mathrm{Attn}(X)=VA.
\]
Write $R=\|X\|_F$, $\bar w_V=\sup_\Theta\|W_V\|_2$,
$a=\sup_\Theta\|W_Q\|_2$, $b=\sup_\Theta\|W_K\|_2$,
and $\kappa=\sup_\Theta\|W_Q^\top W_K\|_2\le ab$.
For RMS normalization $N(z)=gz/\sqrt{\|z\|^2/d+\epsilon}$ with $\epsilon>0$, we use
\begin{equation}
\|N(Z)\|_F\le |g|\sqrt{dn}=:\nu,
\qquad
\|\nabla N(Z)\|_2\le |g|\sqrt{d/\epsilon}=L_N.
\label{eq:normalizer}
\end{equation}

We consider four treatments:
\begin{center}
\footnotesize
\begin{tabular}{@{}llll@{}}
\toprule
& Query & Key & Value \\
\midrule
I \ Unnormalized       & $W_QX$    & $W_KX$    & $W_VX$ \\
II \ Query-normalized  & $N(W_QX)$ & $W_KX$    & $W_VX$ \\
III \ Query-key normalized & $N(W_QX)$ & $N(W_KX)$ & $W_VX$ \\
IV \ Input-normalized  & $W_QN(X)$ & $W_KN(X)$ & $W_VN(X)$ \\
\bottomrule
\end{tabular}
\end{center}

\begin{lemma}[Self-attention sensitivity under normalization]
\label{lem:self-attention:sensitivity}\label{lem:attention-sigma}
Under \eqref{eq:normalizer}, for $R\ge1$, the four treatments satisfy
$L_{\rm attn}(R)\le\ell_0+\ell_1R^\sigma$, with
\begin{equation}
\setlength{\arraycolsep}{2.5pt}
\begin{array}{llll}
\text{I:} & \sigma=2, &
\ell_1=\dfrac{\bar w_V\kappa}{\sqrt{d_k}}, &
\ell_0=\sqrt{n}\,\bar w_V, \\[6pt]
\text{II:} & \sigma=2, &
\ell_1=\dfrac{\bar w_V b(L_N a+\nu)}{2\sqrt{d_k}}, &
\ell_0=\sqrt{n}\,\bar w_V, \\[6pt]
\text{III:} & \sigma=1, &
\ell_1=\dfrac{\bar w_V\nu L_N(a+b)}{2\sqrt{d_k}}, &
\ell_0=\sqrt{n}\,\bar w_V, \\[6pt]
\text{IV:} & \sigma=0, &
\ell_1=0, &
\ell_0=\bar w_V L_N
\left(\sqrt{n}+\dfrac{\nu^2\kappa}{\sqrt{d_k}}\right).
\end{array}
\label{eq:attention-sigma}
\end{equation}
The corresponding growth and sensitivity exponents are
$(q,\sigma)=(1,2),(1,2),(1,1),(0,0)$, respectively.
\end{lemma}

\begin{proof}
Differentiating $\mathrm{Attn}(X)=VA(X)$ gives
\begin{equation}
\nabla_X\mathrm{Attn}(X)[\delta]
=(\partial V[\delta])A
+V\,\partial A[\partial S[\delta]].
\label{eq:attn-two-terms}
\end{equation}
Since $\|\partial A/\partial S\|_2\le1/2$, the sensitivity growth is determined by $V$ and $\partial S$.

In Treatment I, both $V$ and $\partial S$ grow as $O(R)$, giving $\sigma=2$. Normalizing only $Q$ (II) leaves an unnormalized key in $\partial S$, so $\sigma=2$ remains. Normalizing both $Q$ and $K$ (III) makes $\partial S=O(1)$, while $V=O(R)$, yielding $\sigma=1$. Finally, normalizing the shared input to all three projections (IV) bounds both $V$ and $\partial S$ uniformly, giving $\sigma=0$. The value-path bounds and \eqref{eq:normalizer} give the coefficients in \eqref{eq:attention-sigma}.
\end{proof}

The QK-input norm used in Sec.~\ref{sec:lm} and App.~\ref{app:toto} instead takes $Q=W_QN(X)$, $K=W_KN(X)$, and $V=W_VX$. Like Treatment III, it gives $(q,\sigma)=(1,1)$, since the query and key paths are bounded while the value path remains unnormalized.

\subsection{Post-Transformer branches}
\label{app:post-transformer}

The arithmetic applies to any block written as a composition of primitives, so
it reaches architectures beyond the four we train. We record six of them here,
in the lower block of Tab.~\ref{tab:primitives}. Each entry is read off
Theorems~\ref{thm:arith} and~\ref{thm:sigma-arith} from the block's defining
equations; $x$ denotes the raw residual stream, so ``Pre-LN removed'' means the
branch reads $x$ and not $N(x)$. Which normalization is kept is part of the
statement, not a detail: in four of the six cases below, moving one
normalization moves an exponent.

\paragraph{RWKV-4.}
Token shift is linear, so $r_t,k_t,v_t$ are affine in $x$ and carry $(1,0)$.
The WKV operator is an exact channel-wise softmax average,
$\mathrm{wkv}_t=\sum_i\lambda_{t,i}\odot v_i$ with
$\lambda_{t,i}\propto e^{k_i-(t-1-i)w}$ summing to one per channel, so its
weights are a softmax of an affine map, $(0,0)$ by rule (i), and the bilinear
rule (iii) gives $q\le0+1=1$ and $\sigma\le\max(0+1,0+0)=1$. The sigmoid gate
is $(0,0)$ and the output projection affine, so time mixing is $(1,1)$: at the
forward threshold and at the backward threshold, the same pair as QK-input
normalized attention above. The contrast with self-attention is instructive and
is entirely about what the softmax sees: here its argument is affine in $x$, so
the weights are $(0,0)$, whereas the bilinear logits of attention make them
$(0,1)$ and push $\sigma$ to $2$.
Channel mixing replaces the average by $W'_v\max(k'_t,0)^2$. Squared ReLU is
the entrywise square $(2,1)$ composed with ReLU $(1,0)$, hence $(2,1)$, and the
sigmoid gate then raises the Jacobian sensitivity exponent by rule (iii),
$\sigma\le\max(0+2,0+1)=2$, giving $(2,2)$. In RWKV-4 the forward certificate
therefore fails in the channel mixing and not in the time mixing.

\paragraph{RWKV-7.}
The state is $\mathrm{wkv}_t=\sum_i(v_i^\top\tilde k_i)\prod_{j>i}G_j$ with
$G_j=\mathrm{diag}(w_j)-\hat\kappa_j^\top(a_j\odot\hat\kappa_j)$. Every factor
of $G_j$ is bounded with bounded derivative: $w_j\in(e^{-e^{-1/2}},1)$,
$a_j\in(0,1)$, and $\hat\kappa_j$ is the $\ell_2$ normalization of $\kappa_j$,
so $G_j$ is $(0,0)$. The outer product $v_i^\top\tilde k_i$ is bilinear in two
affine maps, $(2,1)$, so $\mathrm{wkv}_t$ is $(2,2)$ and $r_t\mathrm{wkv}_t^\top$
is $(3,3)$. The in-block normalization is applied to that product, and
post-normalization sets $q$ to $0$ while leaving $\sigma$ untouched: rule (v)
of Theorem~\ref{thm:sigma-arith} requires the normalization to precede the map,
and in the other order $\sigma(N\circ B)\le\sigma_Nq_B+\sigma_B=\sigma_B$. The
bonus term $u_t=(r_t(\rho\odot\tilde k_t)^\top)v_t$ is a product of three affine
maps, $(3,2)$, and is added \emph{after} that normalization. Summing gives
$(3,3)$, which the bounded gate and the output projection preserve. Two
conditions carry the pair, and both are placement conditions rather than
properties of the recurrence: $q=3$ comes entirely from the un-normalized bonus
term, without which the branch would be $(0,3)$; and $\sigma=3$ survives the
in-block normalization only because that normalization sits downstream of the
product rather than on the branch input.

\paragraph{RetNet.}
Retention is $(QK^\top\odot D)V$ with $Q,K,V$ affine in $X$ and $D$ the fixed
causal decay matrix, $\gamma$ being a per-head constant rather than an
input-dependent quantity. There is no softmax, and that is the whole difference
from self-attention: the score matrix stays at $(2,1)$ instead of collapsing to
$(0,1)$, so retention is $(3,2)$ where attention is $(1,2)$. The head-wise
GroupNorm is applied to the retention output, so it sets $q$ to $0$ and leaves
$\sigma=2$; the swish gate then gives $q\le1+0=1$ and
$\sigma\le\max(0+0,1+2)=3$, that is $(1,3)$. Removing the GroupNorm as well
leaves the gate acting on $(3,2)$, giving $q\le1+3=4$ and
$\sigma\le\max(0+3,1+2)=3$, that is $(4,3)$.
We do not credit the retention score rescaling
$\tilde R_{nm}=R_{nm}/\max(|\sum_iR_{ni}|,1)$ with lowering either exponent.
Its stated purpose is numerical precision and it is justified through the scale
invariance of the GroupNorm that follows it, and the released implementation
detaches the divisor, so it does not enter the Jacobian at all. Crediting it as
a normalization would give $(1,2)$; the entry in Tab.~\ref{tab:primitives} is
the conservative one.

\paragraph{Hyena.}
The order-$m$ operator is $z^{1}=v$, $z^{n+1}=x^{n}\odot(h^{n}*z^{n})$,
$y=z^{m+1}$, with $v$ and the gates $x^{n}$ affine in the input and the filters
$h^{n}$ generated from positional encodings alone, hence input-independent
linear maps. Induction with rule (iii) from $z^1=(1,0)$ gives $z^{n+1}$ the
pair $(n+1,n)$, since $q\le1+n$ and $\sigma\le\max(0+n,1+(n-1))=n$; the output
projection does not change it, so the operator is $(m+1,m)$. Hyena is the one
family here whose exponents are unbounded by design rather than by the
placement of a normalization: both grow linearly in the order, and already at
$m=1$ the forward exponent is $2$, above the threshold.

\paragraph{BDH.}
In the released implementation the layer is $x\mapsto N(x+N(F(x)))$ with
$F(x)=\bigl[\mathrm{ReLU}(xD_x)\odot\mathrm{ReLU}(N(\mathrm{Attn})D_y)\bigr]E$
and $N$ a non-affine LayerNorm. The attention is a Gram matrix of the post-ReLU
activations applied to the stream, with no softmax and no $1/\sqrt{d_k}$
scaling, so by the same count as retention it is $(3,2)$. The in-block
normalization sets that to $(0,2)$, which the second ReLU and the projection
preserve; the Hadamard product with the un-normalized $\mathrm{ReLU}(xD_x)$,
which is $(1,0)$, gives $q\le1+0=1$ and $\sigma\le\max(0+0,1+2)=3$; and the two
outer normalizations set $q$ to $0$ and leave $\sigma=3$. The layer map as
released is therefore $(0,3)$. BDH is the one architecture here that reaches
$q=0$ without a Pre-LN, because its normalization sits on the residual sum
instead of on the branch input. That placement carries a caveat the other
entries do not: $X_{k+1}=N(X_k+N(F(X_k)))$ is not the additive recursion
\eqref{eq:residual}, so the pair $(0,3)$ is the exponent pair of the layer map,
and the depth recursion it generates is outside the scope of
Theorem~\ref{thm:threshold}. What the pair does say is that the forward
exponent is bought by the outer normalization while $\sigma$ is left exactly
where the un-normalized product put it.
Removing every LayerNorm leaves the attention at $(3,2)$
and the product at $q\le1+3=4$, $\sigma\le\max(0+3,1+2)=3$, that is $(4,3)$.

\paragraph{Jamba.}
Its layers are Pre-RMSNorm residual blocks around a Mamba variant, self-attention,
or a dense or sparse MLP. With the stream normalizations removed the branch
exponents are $(2,2)$ for the Mamba branch, $(1,2)$ for attention and $(2,1)$
for the dense MLP. The Mamba branch sits at $q=2$ rather than at the $5$ of
Tab.~\ref{tab:primitives} because of the three in-block RMSNorms applied to
$\Delta$, $B$ and $C$: they make the selective coefficients bounded, which is
exactly the mechanism of the linear-growth block of
Fig.~\ref{fig:velocity-mamba-lin}, so only the input and the output gate
remain unbounded and the product of the two is degree $2$.
The sparse layers are different in kind. Top-$k$ routing over an unrenormalized
softmax is discontinuous: at a configuration where the $k$-th and $(k{+}1)$-th
router logits agree, the selected set changes while the weights do not, so the
layer output jumps by $w(E_{k}-E_{k+1})$. Assumption~\ref{ass:lip} then admits
no finite $L(R)$ on any ball meeting a routing boundary, and $\sigma$ does not
exist for those layers. The forward exponent is unaffected, because
Assumption~\ref{ass:growth} asks only for a growth bound, which a discontinuous
map can satisfy. This is the one entry in the catalogue where the backward
certificate is not violated but undefined, and it is a property of the routing
rather than of the experts: a soft or renormalized router with the same experts
would return the layer to $(2,2)$.

\subsection{Certification procedures}
\label{app:certification}

Our results suggest an architecture-design procedure based on the growth and sensitivity exponents $(q,\sigma)$. The objective is not necessarily to minimize both exponents: our experiments in time-series forecasting (Sec.~\ref{sec:tsf}), operator learning (Sec.~\ref{sec:fno}), and Toto~2.0 (Sec.~\ref{sec:foundation}) show that architectures with $q=1$ can outperform their normalized $q=0$ counterparts. The goal is therefore to ensure controlled forward growth ($q\le1$) while limiting backward sensitivity through $\sigma$, without unnecessarily restricting architectural expressivity.

\emph{The role of sensitivity becomes apparent at larger scales.}
In our smaller-scale forecasting and neural-operator experiments, training failures were associated with superlinear forward growth ($q>1$), whereas architectures with $q\le1$ trained successfully regardless of $\sigma$. At larger scales, however, backward sensitivity becomes relevant: while the Toto~2.0 variant with $(q,\sigma)=(1,2)$ trains successfully, the unnormalized GPT-2 model with the same exponents diverges, and reducing $\sigma$ to $1$ through QK-input norm restores stable training. These observations suggest that controlling forward growth alone may be insufficient as model size or training scale increases, although the experiments do not isolate a universal scaling threshold.

\begin{enumerate}[leftmargin=*,topsep=1pt,itemsep=2pt]
\item \emph{Identify residual updates.}
Express the architecture as \eqref{eq:residual}, treating each residual branch as a separate update.
\item \emph{Compute $(q,\sigma)$.}
Decompose each update into primitives from Tab.~\ref{tab:primitives} and apply the composition rules in Theorems~\ref{thm:arith} and~\ref{thm:sigma-arith}. The resulting exponents provide upper bounds on forward growth and Jacobian sensitivity.
\item \emph{Control superlinear growth.}
For blocks with $q>1$, modify the responsible primitives to obtain $q\le1$. This can be achieved through normalization or bounded substitutions, such as the linear-growth Mamba modification in Fig.~\ref{fig:velocity-mamba-lin}. Retaining $q=1$ may be preferable to imposing $q=0$ throughout the architecture.
\item \emph{Assess backward sensitivity.}
For blocks with large $\sigma$, consider modifications that reduce input-dependent Jacobian growth. For example, query-key normalization reduces self-attention from $(q,\sigma)=(1,2)$ to $(1,1)$ (Lemma~\ref{lem:attention-sigma}). Our language-modeling experiments (Sec.~\ref{sec:lm}) illustrate the importance of this distinction at larger training scales.
\item \emph{Evaluate finite-depth certificates.}
Given the coefficients in $|v(x)|\le c|x|^q+b$ and $L(R)\le\ell_0+\ell_1R^\sigma$, apply Theorems~\ref{thm:threshold}
and~\ref{thm:backward}  to estimate certified depths for the intended precision and residual step size. The calculation in Tab.~\ref{tab:overflow} can be repeated with architecture-specific parameters.
\end{enumerate}
For heterogeneous architectures, the procedure applies block by block. Forward growth depends on the sequence of blockwise exponents and coefficients, while backward sensitivity accumulates through the product of layerwise Jacobian bounds. Consequently, a block with large $\sigma$ need not dominate the overall certificate if its coefficient or reachable radius is small.
These certificates are sufficient and can be conservative, since they use worst-case bounds over reachable states rather than the Jacobians encountered during training. The exponent classification therefore serves as an a priori architectural diagnostic, while quantitative depth guarantees require the corresponding coefficients.

\section{Empirical Exponents and Training-Time Measurements}
\label{app:qhat}

The exponents $q$ and $\sigma$ characterize two distinct properties of a
residual velocity: the growth of its output magnitude and the growth of its
Jacobian sensitivity with respect to the input magnitude. The certification
framework of Sec.~\ref{sec:exponents} derives these exponents \emph{a priori} from the
architectural primitives, independently of a particular training trajectory.
Here we distinguish these architecture-level bounds from two complementary
forms of empirical analysis: local exponent estimation by controlled input
rescaling, and inexpensive measurements of forward and backward increments
along an actual training trajectory.
Tab.~\ref{tab:diagnostics} summarizes what each class of measurement does and
does not establish, and the rest of this appendix defines them. 

\begin{table}[h]
\centering
\caption{The four classes of measurement used in this paper. $(q,\sigma)$ are
the certified architecture-level exponents; $\qhat_k,\shat_k$ are the local
estimates of App.~\ref{app:qhat:exponents}; the RMS and relative-increment
quantities are the training-time measurements of
App.~\ref{app:qhat:increments}.}
\label{tab:diagnostics}
\footnotesize
\setlength{\tabcolsep}{4pt}
\begin{tabular}{@{}>{\raggedright\arraybackslash}p{0.17\linewidth}>{\raggedright\arraybackslash}p{0.27\linewidth}>{\raggedright\arraybackslash}p{0.22\linewidth}>{\raggedright\arraybackslash}p{0.28\linewidth}@{}}
\toprule
Measurement & What it establishes & Cost & What it cannot do \\
\midrule
Architecture exponents $(q,\sigma)$ (App.~\ref{app:arith})
& A bound valid for every $\theta\in\Theta$ and every state in the reachable ball, before any training run
& Symbolic; none at run time
& Being worst case, it can exceed by orders of magnitude anything a particular run visits \\
\addlinespace
Empirical exponents $\qhat_k,\shat_k$ (App.~\ref{app:qhat:exponents})
& Local scaling at the probed state, by rescaling the input over $s\in[1/4,4]$
& Extra forward and backward passes per checkpoint
& Depends on the probed state; a scale sweep is a ray, and the configurations that attain $\sigma$ are off it (App.~\ref{app:attention-sigma}) \\
\addlinespace
State and gradient RMS (App.~\ref{app:qhat:increments})
& The magnitudes actually accumulated along the run, and where they first grow
& Logged during training
& Characterizes the realized trajectory, not the architecture \\
\addlinespace
Relative increments $F_k$, $A_k$ (App.~\ref{app:qhat:increments})
& The contribution of a single residual branch, with the identity removed
& Logged during training
& $A_k$ mixes the Jacobian with the costate orientation, so it is not an estimator of $\sigma$ \\
\bottomrule
\end{tabular}
\end{table}

\subsection{Empirical Growth and Sensitivity Exponents}
\label{app:qhat:exponents}

Consider a residual update
\begin{equation}
    X_{k+1}=X_k+\Delta t\,v_k(X_k).
\end{equation}
The forward growth exponent $q$ controls the magnitude of the velocity,
\begin{equation}
    \|v_k(x;\theta)\|
    \leq c\|x\|^q+b,
    \label{eq:empirical:forward-bound}
\end{equation}
whereas the Jacobian sensitivity exponent $\sigma$ controls the sensitivity
of the velocity to perturbations of its input,
\begin{equation}
    L(R):=
    \sup_{\theta\in\Theta,\;\|x\|\leq R}
    \|\nabla_Xv_k(x;\theta)\|_2
    \leq \ell_1R^\sigma+\ell_0,
    \qquad R\geq 1.
    \label{eq:empirical:backward-bound}
\end{equation}
These bounds concern the admissible parameter set and reachable-state
region, rather than the states and weights encountered in a particular
training run. Their exponents can therefore be bounded before training
through the primitive-level arithmetic of Sec.~\ref{sec:exponents}.

Empirical exponents address a different question: how does a particular
block, at its current parameter values, respond to changes in input
magnitude near the states encountered by the model? They provide local
growth measurements rather than architecture-level certificates.
Although they can be evaluated before training at initialized weights,
their interpretation generally depends on the chosen parameter
realization and sampled states. Repeated evaluation during training is
also possible, but requires additional probing.

\paragraph{Empirical forward growth exponent.}
Let $\{X_k^{(i)}\}_{i=1}^B$ be the states at depth $k$ produced by a
fixed evaluation batch. For
$s\in\mathcal S=\{s_1,\ldots,s_9\}$, geometrically spaced in $[1/4,4]$,
we evaluate the velocity branch alone at the rescaled state and fit
\begin{equation}
    \log\big\|v_k(sX_k^{(i)})\big\|_F
    =
    \qhat_k\log\big\|sX_k^{(i)}\big\|_F
    +\beta_k+\varepsilon_{i,s},
    \qquad
    i=1,\ldots,B,\quad s\in\mathcal S,
    \label{eq:empirical:q}
\end{equation}
by ordinary least squares, per seed, and average over seeds.

Evaluating $v_k$ rather than the full residual update removes the
identity path, whose contribution could otherwise dominate the fitted
slope. Controlled rescaling also provides a substantially wider range
of input magnitudes than the incidental variation within a batch.

Since the rescaling range is bounded, $\qhat_k$ measures local growth
near the observed states rather than the asymptotic exponent in
\defPrim. For a block satisfying
$\|v_k(x)\|\leq c\|x\|^q+b$ with $b>0$, the local slope can lie below
$q$ at small scales. Conversely, nonlinear interactions within a
block can produce slopes above its asymptotic growth exponent over
a finite range.

Tab.~\ref{tab:qhat-mamba} applies the same protocol to the forecasting
blocks of Sec.~\ref{sec:tsf} (App.~\ref{app:tsf}).

\paragraph{Empirical Jacobian sensitivity exponent.}

The backward counterpart estimates how the Jacobian operator norm
changes with input magnitude. Using the same states and rescaling
grid $\mathcal S$, we fit
\begin{equation}
    \log\bigl\|
        \nabla_Xv_k(sX_k^{(i)})
    \bigr\|_2
    =
    \shat_k\log\bigl\|sX_k^{(i)}\bigr\|_F
    +\beta_k'+\varepsilon'_{i,s}.
    \label{eq:empirical:sigma}
\end{equation}
This estimator directly probes the quantity entering the local
Lipschitz bound of Assumption~\ref{ass:lip}, but only at sampled states and
fixed parameter values.

The spectral norm can be approximated without constructing the full
Jacobian. Power iteration applied to Jacobian-vector products,
with transposed products supplied by reverse-mode differentiation,
provides an estimate at the cost of additional differentiation
passes for each input scale.

Unlike forward growth, worst-case Jacobian sensitivity may be attained only at special input configurations. For example, the quadratic Jacobian growth of unnormalized attention occurs near tied maximal logits (App.~\ref{app:attention-sigma}), which may not be encountered by rescaling training states. Consequently, $\shat_k$ can substantially underestimate the worst-case growth exponent $\sigma$ and should be interpreted as a local diagnostic rather than an architecture-level certificate.

Fig.~\ref{fig:phase-empirical} places the measured forward exponents of every
trained variant on the certified plane of Fig.~\ref{fig:phase}.

\begin{figure}[t]
\centering
\includegraphics[width=0.86\linewidth]{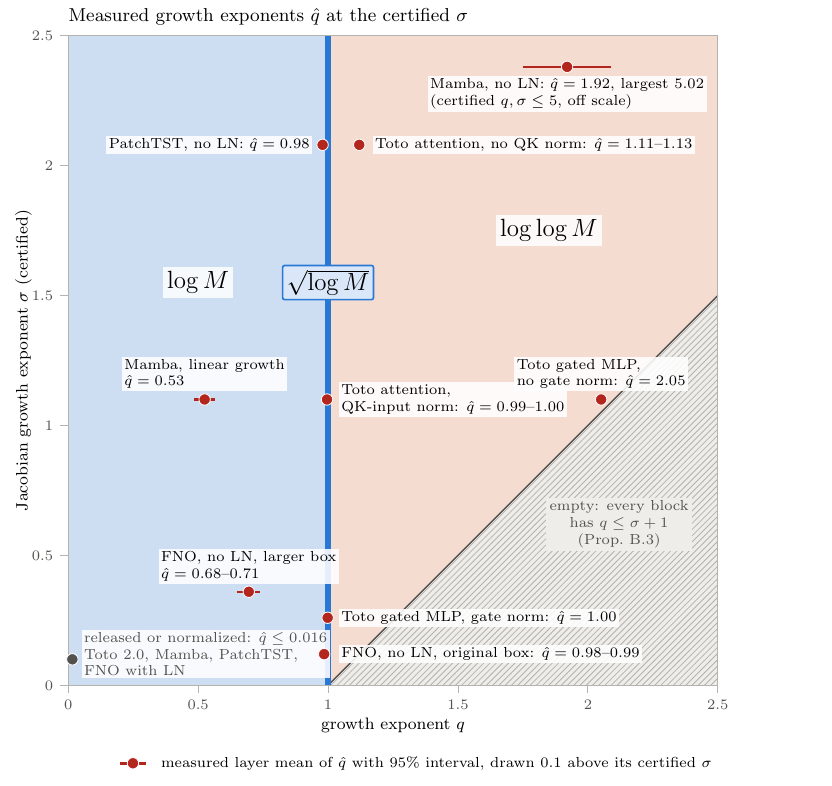}
\caption{Measured forward exponents on the certified $(q,\sigma)$ plane of
Fig.~\ref{fig:phase}. Open circles mark the certified pair of each trained
variant. Red markers give the layer mean of $\qhat_k$ at the final checkpoint
with its $95\%$ $t$-interval, drawn $0.1$ above the certified $\sigma$ of the
same variant so that each measurement is paired with its certificate. Only the
horizontal coordinate is measured: $\shat_k$ is not shown, since it probes the
Jacobian at sampled states and can underestimate $\sigma$ when the worst case is
attained only at special inputs. The plotted values are those of
Tabs.~\ref{tab:qhat-mamba}, \ref{tab:qhat-fno} and~\ref{tab:toto-qhat}.}
\label{fig:phase-empirical}
\end{figure}

\subsection{Training-Time Forward and Backward Measurements}
\label{app:qhat:increments}

Unlike empirical exponent estimation, which requires evaluating a
block at multiple rescaled inputs, the forward states and backward
costates are already computed during ordinary training. They therefore provide inexpensive, trajectory-dependent diagnostics that respond directly to evolving training dynamics, allowing emerging instabilities to be detected without additional input-scaling experiments.

\paragraph{Forward state magnitude and relative increment.}
\begin{itemize}
    \item \textbf{Input RMS} measures the magnitude of the forward state:
    \begin{equation}
        \operatorname{RMS}(X_k)
        = \frac{\|X_k\|_F}{\sqrt{N}},
        \label{eq:empirical:forward-rms}
    \end{equation}
    where $N$ is the number of scalar entries in the state.

    \item \textbf{Relative forward increment} measures the contribution
    of the residual branch relative to its input:
    \begin{equation}
        F_k
        := \frac{\|X_{k+1}-X_k\|_F}{\|X_k\|_F}
        = \Delta t\,\frac{\|v_k(X_k)\|_F}{\|X_k\|_F}.
        \label{eq:empirical:forward-increment}
    \end{equation}
\end{itemize}
Unlike $\qhat_k$, these quantities track the realized forward
dynamics rather than growth under controlled input rescaling.

\paragraph{Backward costate magnitude and relative increment.}
Let $p_k=\nabla_{X_k}G$ denote the costate, satisfying
$p_k=p_{k+1}+\Delta t\,\nabla_Xv_k(X_k)^\top p_{k+1}$.
\begin{itemize}
    \item \textbf{Gradient RMS} measures the magnitude of the
    propagated backward signal:
    \begin{equation}
        \operatorname{RMS}(p_k)
        = \frac{\|p_k\|_F}{\sqrt{N}}.
        \label{eq:empirical:backward-rms}
    \end{equation}

    \item \textbf{Relative backward increment} measures the
    Jacobian contribution along the propagated costate direction:
    \begin{equation}
        A_k := \frac{\|p_k-p_{k+1}\|_F}{\|p_{k+1}\|_F} = \Delta t\,
        \frac{\|\nabla_Xv_k(X_k)^\top p_{k+1}\|_F}
             {\|p_{k+1}\|_F}
        \leq \Delta t\,L(\|X_k\|).        \label{eq:empirical:backward-increment}
    \end{equation}
\end{itemize}
Unlike $\shat_k$, $A_k$ measures realized, direction-dependent
Jacobian sensitivity rather than spectral-norm growth under
input rescaling. Moreover, $A_k$ need not imply net gradient
amplification, which is measured by
\begin{equation}
    B_k:=\frac{\|p_k\|_F}{\|p_{k+1}\|_F},
    \qquad |1-A_k|\leq B_k\leq 1+A_k.
    \label{eq:empirical:amplification-bounds}
\end{equation}

\emph{Why the two increments are parallel, and where the parallel stops.} The
forward and backward relative increments have the same structure,
\begin{equation}
\begin{aligned}
    X_{k+1}-X_k &=\Delta t\,v_k(X_k),\\
    p_k-p_{k+1} &=\Delta t\,\nabla_Xv_k(X_k)^\top p_{k+1},
\end{aligned}
\label{eq:empirical:increment-duality}
\end{equation}
and both remove the identity contribution so as to isolate one residual
update. Their interpretations differ: the forward increment measures the
velocity output, whereas the backward increment measures a Jacobian action
along a loss-dependent direction. This is why $A_k$ cannot in general serve as
an estimator of $\sigma$. Its value depends on the orientation of the
propagated costate as well as on the Jacobian, so a change in gradient
alignment moves $A_k$ with $\|\nabla_Xv_k\|_2$ held fixed.

\section{Benchmarks and reproducibility}
\label{app:experiments}\label{app:setup}\label{app:exponents}

This appendix specifies what each experiment is. Tab.~\ref{tab:setup} is the
overview of the four tasks: backbone, depth, normalization variants and the
section that reports each. The four subsections below then give, per task, the
data, the architecture settings, the training recipe, the evaluation protocol
and the exponent derivation for the backbone. Everything needed to rebuild a
run is here; the measurements the runs produced are in App.~\ref{app:results},
under the same four subsection names.

\begin{table}[h]
\centering
\caption{Task configurations. $D$: number of residual updates; Transformer-type backbones count the attention and MLP sub-blocks of a layer separately, so GPT-2 and Toto~2.0, each instantiated with $12$ layers, have $D=24$. The derivation of the exponent $q$ for each variant is given in App.~\ref{app:exponents}. Parameter counts are real-valued degrees of freedom; the FNO's complex spectral weights count as two each. The training set is given as sliding windows per horizon for ETTm1 and Weather, finite-element solves for the original and the larger parameter box, tokens for OpenWebText, and the size of the GIFT-Eval pretraining subset for Toto~2.0; App.~\ref{app:experiments} gives the splits.}
\label{tab:setup}
\footnotesize
\setlength{\tabcolsep}{4pt}
\begin{tabular}{@{}lllp{0.30\linewidth}@{}}
\toprule
Task (data, size; where discussed) & Backbone (params) & Depth & Variants (exponent $q$) \\
\midrule
\begin{tabular}[t]{@{}l@{}}Forecasting\\(ETTm1, Weather)\\{\scriptsize $34$K/$37$K windows}\\{\scriptsize Sec.~\ref{sec:tsf}, App.~\ref{app:tsf}}\end{tabular} & Mamba (0.58--2.27M) & $D\in\{3,8\}$ & RMSNorm, Pre-LN, Peri-LN (0); linear growth (1); no-LN (5) \\
 & PatchTST (0.62--2.24M) & $D\in\{3,8\}$ & RMSNorm, Pre-LN, Peri-LN (0); no-LN (1) \\
\addlinespace
\begin{tabular}[t]{@{}l@{}}Forecasting\\(GIFT-Eval)\\{\scriptsize $28$ sets, $3.73$\,GB}\\{\scriptsize Sec.~\ref{sec:foundation}, App.~\ref{app:toto}}\end{tabular} & Toto~2.0 (10.5M) & $D=24$ & Pre-RMSNorm (0); no-LN (2); no-LN + QK input norm (2); no-LN + gate norm (1); no-LN + gate \& QK input norm (1) \\
\addlinespace
\begin{tabular}[t]{@{}l@{}}Operator learning\\(3-D battery cooling)\\{\scriptsize $216$/$2{,}816$ solves}\\{\scriptsize Sec.~\ref{sec:fno}, App.~\ref{app:fno}}\end{tabular} & FNO (6.46M) & $D=4$ & in-branch LN (0); no-LN (1) \\
\addlinespace
\begin{tabular}[t]{@{}l@{}}Language modeling\\(OpenWebText)\\{\scriptsize $9$B tokens}\\{\scriptsize Sec.~\ref{sec:lm}, App.~\ref{app:owt}}\end{tabular} & GPT-2 (124M) & $D=24$ & Pre-LN (0); no-LN (1); no-LN + QK input norm (1) \\
\bottomrule
\end{tabular}
\end{table}

\paragraph{Supplementary material.}
Code for the three benchmarks with released source is included, in three
self-contained directories, \texttt{tsf\_small/}, \texttt{fno\_battery/} and
\texttt{nanogpt\_openwebtext/}. Each contains the training entry points, the
measurement and aggregation code, and the tables that the figures of this paper
are drawn from; each carries a \texttt{README} or \texttt{REPRODUCE} file
listing the commands in order. Model weights, checkpoints and generated
datasets are not included, since all of them are regenerated by the scripts;
the files that are included are the ones a reader would otherwise have to
retrain to obtain. 
For the time-series foundation model benchmark we provide the architecture and
model configuration rather than the training code.

\subsection{Time-series forecasting}
\label{app:tsf}

\paragraph{Data.} ETTm1 \citep{zhou2021informer}: $7$ channels at $15$-minute sampling, fixed $12/4/4$-month borders ($34{,}560/11{,}520/11{,}520$ time steps).
Weather \citep{wu2021autoformer}: $21$ channels at $10$-minute sampling, $0.7/0.1/0.2$ split of $52{,}696$ rows ($36{,}887/5{,}270/10{,}539$).
Input length $96$ (12 patches of length $16$, stride $8$); horizons $H\in\{96,192,336,720\}$; sliding windows with stride $1$ give about $34.4$K (ETTm1) and $36.7$K (Weather) training windows per horizon.

\paragraph{Model.}
Both backbones are channel-independent: each channel is embedded and
forecast on its own, with no cross-channel mixing inside the model.

\emph{Mamba.} CMamba \citep{zeng2024cmamba} built on the selective SSM of \citet{gu2024mamba}: width $128$, $d_{\rm ff}=128$, state size $16$, convolution width $4$, expansion $2$, $\Delta$-rank $32$, parallel scan, dropout $0$. We replace CMamba's gated global data-dependent MLP with the standard MLP of PatchTST (width $256$), so that the two backbones differ only in their token-mixing block, and we drop its channel mixing in favor of channel independence; the selective state-space block, which carries the $q=5$ exponent, is unchanged.
Parameters: $0.583$M ($D=3$, $H=96$) to $2.265$M ($D=8$, $H=720$), about $0.145$M per residual update; the flatten head has $0.148$M ($H=96$) to $1.106$M ($H=720$) parameters.
Adam, learning rate $10^{-3}$, batch $64$, MSE loss, learning-rate schedule \texttt{type3}, at most $100$ epochs with early stopping (patience $3$).
Variants: RMSNorm baseline \citep{zhang2019root}, Pre-LN \citep{xiong2020layer}, and Peri-LN \citep{kim2025periln} ($q=0$); linear growth ($q=1$); no-LN, i.e., normalization removed ($q=5$). The no-LN variant is labeled \emph{free velocity} in our code and result files.

\emph{PatchTST.} PatchTST \citep{nie2023patchtst} with $16$ heads, width $128$, $d_{\rm ff}=256$, end padding, dropout $0.2$, RevIN, and MLP width $256$.
Parameters: $0.615$M ($D=3$, $H=96$) to $2.236$M ($D=8$, $H=720$), about $0.132$M per residual update; channel independence makes the counts identical on both datasets.
Variants: RMSNorm baseline, Pre-LN, Peri-LN ($q=0$), and no-LN ($q=1$).

\paragraph{Protocol.}

\emph{Training.} Adam with MSE loss, at most $100$ epochs, and early stopping
at patience~$3$ on the validation split, in every run. The optimizer settings
are those of each backbone's reference recipe and differ between them: Mamba
uses learning rate $10^{-3}$ and batch~$64$ with a decay that holds the rate
for two epochs and then multiplies it by $0.9$ each epoch, on both datasets;
PatchTST uses learning rate $10^{-4}$ and batch~$128$, with the one-cycle
schedule of \citet{nie2023patchtst} at warm-up fraction $0.4$ on ETTm1 and a
halving-per-epoch decay on Weather. The comparison is therefore within a
backbone across normalization variants, not between backbones. CMamba's
channel-mixup augmentation is active in
every run: at each step a random permutation of the channels, scaled per
sample and channel by $\mathcal N(0,1)$, is added to both the input and the
target. It is applied identically to every variant, so it does not confound
the normalization comparison. Seeds are $2021$--$2026$ on ETTm1 and
$2021$--$2023$ on Weather, giving $4$ horizons $\times$ $2$ depths $\times$
$9$ seed--dataset pairs $=72$ runs per variant.

\emph{Evaluation.} MSE and MAE on the test split, pooled over all sliding
windows, horizon steps and channels, and computed on the standardized
targets: the runs do not request the inverse transform, so the scale is that
of the training statistics. Blow-up detection checks the training loss for
NaN at every step and the final test MSE against the criterion of
Sec.~\ref{sec:results}. The outcome is bimodal: of the $648$ runs, $46$ end
with a non-finite test MSE and $2$ with a finite MSE of order $10^{12}$ and
$10^{17}$, while the largest among the remaining $600$ is $0.47$. All $48$ are
marked as blow-up. Accuracy is averaged over surviving seeds.
\paragraph{Exponent derivation.}

\emph{Transformer block (PatchTST).}
The block is drawn in Fig.~\ref{fig:velocity-patchtst}. The attention branch is
\[ v^{\rm attn}(X)=W_O\,\mathrm{softmax}\big(QK^{\top}/\sqrt{d_k}\big)\,V, \]
with $Q,K,V$ linear in $X$ ($q=1$).
The logits $QK^{\top}$ are bilinear ($q=2$ by (iii)), but softmax is row-stochastic, so $\mathrm{softmax}(QK^{\top})$ has $q=0\cdot2=0$ by (i).
The product with $V$ gives $0+1=1$ by (iii), and $W_O$ keeps $q=1$.
The MLP branch $W_2\,\mathrm{GELU}(W_1X)$ is a composition of $q=1$ primitives, hence $q=1$.
Without normalization both velocities are at $q=1$; with a normalization in front of each branch (Pre-LN) or around it, rule (i) or (v) gives $q=0$.
Removing normalization from a Transformer with ungated MLP therefore leaves it at the threshold, not above it.
Attention is a multi-token primitive, so the norm in which its exponent is stated matters. With $Q,K,V$ linear in $X$ and a row-stochastic softmax, each output column satisfies $\|v^{\rm attn}(X)_{:,j}\|_2\le\|W_O\|_{\rm op}\|W_V\|_{\rm op}\|X\|_{\infty,2}$, where $\|X\|_{\infty,2}:=\max_k\|X_{:,k}\|_2\le\|X\|_F$. Its $q=1$ therefore holds in this worst-token sense, and, with a possibly larger coefficient, in the Frobenius norm used throughout.

\begin{figure}[h]
\centering
\begin{tikzpicture}[
  font=\scriptsize,
  node distance=2.2mm,
  op/.style={draw, rounded corners=1.5pt, inner sep=1.4pt, minimum height=4.4mm, align=center, fill=white},
  nrm/.style={op, dashed, draw=black!60, text=black!75},
  fix/.style={op, draw=blue!70!black, fill=blue!8, line width=0.5pt},
  prod/.style={draw, circle, inner sep=0pt, minimum size=3.4mm, fill=white},
  q/.style={text=black!70, inner sep=0.3pt},
  res/.style={align=left, inner sep=0.5pt},
  ar/.style={-{Stealth[length=1.3mm]}, black!80, line width=0.35pt},
  lane/.style={font=\scriptsize\bfseries, anchor=west, text width=13mm, align=left, inner sep=0pt},
]
\def\QUP{0.3mm}
\def\QUP{0.3mm}
\node[lane] (la) at (0,0) {PatchTST};
\node[nrm, right=0mm of la.east] (a0) {Norm};
\node[op, right=of a0] (a1) {$Q,K,V{=}WX$};
\node[op, right=of a1] (a2) {$\mathrm{softmax}(QK^{\top})V$};
\node[op, right=of a2] (a3) {$W_O$};
\node[op, right=5mm of a3] (a4) {Linear};
\node[op, right=of a4] (a5) {GELU};
\node[op, right=of a5] (a6) {Linear};
\draw[ar] (a0) -- (a1); \draw[ar] (a1) -- (a2); \draw[ar] (a2) -- (a3);
\draw[ar] (a4) -- (a5); \draw[ar] (a5) -- (a6);
\node[q, above=\QUP of a1] {1};
\node[q, above=\QUP of a2] {$0{\cdot}2{+}1{=}1$};
\node[q, above=\QUP of a3] {1};
\node[q, above=\QUP of a4] {1}; \node[q, above=\QUP of a5] {1}; \node[q, above=\QUP of a6] {1};
\node[q, below=0.4mm of a2] {attention};
\node[q, below=0.4mm of a5] {MLP};
\node[res, right=2.5mm of a6] {no-LN: $q{=}1$\\{\color{blue!65!black}Norm: $q{=}0$}};
\end{tikzpicture}
\caption{Velocity field of one PatchTST block.}
\label{fig:velocity-patchtst}
\end{figure}

\emph{Mamba block.}
The block is drawn in Fig.~\ref{fig:velocity-mamba}. With input $u$ (the residual state, or its normalization), the block computes $x,z=W_{\rm in}u$ ($q=1$), $x\leftarrow\mathrm{SiLU}(\mathrm{conv}(x))$ ($q=1$), and the selective parameters $\Delta=\mathrm{softplus}(W_\Delta x)$, $B=W_Bx$, $C=W_Cx$ (each $q=1$).
For a Hurwitz state matrix $A$ and $\Delta\ge0$, the transition $\bar A=\exp(\Delta A)$ satisfies $\|\bar A\|\le1$ ($q=0$); without this assumption $\bar A$ is not polynomially bounded.
The input matrix $\bar B=\Delta\odot B$ has $q=1+1=2$ by (iii).
The recurrence $h_t=\bar A_th_{t-1}+\bar B_t\odot x_t$ adds $q(\bar B)+q(x)=3$ at every step and $\bar A_t$ does not increase the exponent, so by induction over the finite sequence $h$ has $q=3$.
The readout $y=C\odot h$ has $q=1+3=4$, the gated output $y\odot\mathrm{SiLU}(z)$ has $q=4+1=5$, and the output projection and the channel-mixing MLP keep it at $5$ by (i).
The no-LN block is therefore at $q=5$; a normalization in front of the block gives $q=0$.

\begin{figure}[h]
\centering
\begin{tikzpicture}[
  font=\scriptsize,
  node distance=2.2mm,
  op/.style={draw, rounded corners=1.5pt, inner sep=1.4pt, minimum height=4.4mm, align=center, fill=white},
  nrm/.style={op, dashed, draw=black!60, text=black!75},
  fix/.style={op, draw=blue!70!black, fill=blue!8, line width=0.5pt},
  prod/.style={draw, circle, inner sep=0pt, minimum size=3.4mm, fill=white},
  q/.style={text=black!70, inner sep=0.3pt},
  res/.style={align=left, inner sep=0.5pt},
  ar/.style={-{Stealth[length=1.3mm]}, black!80, line width=0.35pt},
  lane/.style={font=\scriptsize\bfseries, anchor=west, text width=13mm, align=left, inner sep=0pt},
]
\def\QUP{0.3mm}
\def\QUP{0.3mm}
\node[lane] (lb) at (0,0) {Mamba};
\node[nrm, right=0mm of lb.east] (b0) {Norm};
\node[op, right=of b0] (b1) {$x,z{=}WX$;\\conv, SiLU};
\node[op, right=of b1] (b3) {$\Delta,B,C{=}Wx$};
\node[op, right=of b3] (b4) {$\bar B{=}\Delta{\odot}B$};
\node[op, right=of b4] (b5) {$h{=}\bar A h{+}\bar B{\odot}x$};
\node[op, right=of b5] (b6) {$y{=}C{\odot}h$};
\node[prod, right=3.6mm of b6] (b7) {$\odot$};
\node[op, below=1.8mm of b7] (b8) {SiLU$(z)$};
\draw[ar] (b0) -- (b1); \draw[ar] (b1) -- (b3); \draw[ar] (b3) -- (b4);
\draw[ar] (b4) -- (b5); \draw[ar] (b5) -- (b6); \draw[ar] (b6) -- (b7); \draw[ar] (b8) -- (b7);
\node[q, above=\QUP of b1] {1};
\node[q, above=\QUP of b3] {1}; \node[q, above=\QUP of b4] {$1{+}1{=}2$};
\node[q, above=\QUP of b5] {$2{+}1{=}3$}; \node[q, above=\QUP of b6] {$1{+}3{=}4$};
\node[q, above=\QUP of b7] {$4{+}1$}; \node[q, left=0.4mm of b8] {1};
\node[res, right=2.5mm of b7] {no-LN: {\color{red!75!black}$\mathbf{q{=}5}$}\\{\color{blue!65!black}Norm: $q{=}0$}};
\end{tikzpicture}
\caption{Velocity field of the native Mamba block. $\bar A=\exp(\Delta A)$ is a contraction ($q=0$) for Hurwitz $A$, and $\sigma$ denotes the sigmoid ($q=0$) here.}
\label{fig:velocity-mamba}
\end{figure}

\emph{Linear-growth Mamba block.}
Two primitives are changed (Fig.~\ref{fig:velocity-mamba-lin}).
The parameter branch receives $\hat x=x/\|x\|$ ($q=0$), so $\Delta,B,C$ are affine functions of a bounded input and have $q=1\cdot0=0$ by (i); consequently $\bar B$ has $q=0$, the recurrence $h$ has $q=0+1=1$, and the readout $y=C\odot h$ has $q=0+1=1$.
The output gate uses the sigmoid ($q=0$), so $y\odot\sigma(z)$ has $q=1+0=1$.
The residual stream is unchanged, and the block is at $q=1$ without a normalization layer.

\begin{figure}[h]
\centering
\begin{tikzpicture}[
  font=\scriptsize,
  node distance=2.2mm,
  op/.style={draw, rounded corners=1.5pt, inner sep=1.4pt, minimum height=4.4mm, align=center, fill=white},
  nrm/.style={op, dashed, draw=black!60, text=black!75},
  fix/.style={op, draw=blue!70!black, fill=blue!8, line width=0.5pt},
  prod/.style={draw, circle, inner sep=0pt, minimum size=3.4mm, fill=white},
  q/.style={text=black!70, inner sep=0.3pt},
  res/.style={align=left, inner sep=0.5pt},
  ar/.style={-{Stealth[length=1.3mm]}, black!80, line width=0.35pt},
  lane/.style={font=\scriptsize\bfseries, anchor=west, text width=13mm, align=left, inner sep=0pt},
]
\def\QUP{0.3mm}
\def\QUP{0.3mm}
\node[lane] (lc) at (0,0) {Mamba, lin.\ growth};
\node[op, right=0mm of lc.east] (c1) {$x,z{=}WX$;\\conv, SiLU};
\node[fix, right=of c1] (c2) {$\hat x{=}x/\|x\|$};
\node[op, right=of c2] (c3) {$\Delta,B,C{=}W\hat x$};
\node[op, right=of c3] (c4) {$\bar B{=}\Delta{\odot}B$};
\node[op, right=of c4] (c5) {$h{=}\bar A h{+}\bar B{\odot}x$};
\node[op, right=of c5] (c6) {$y{=}C{\odot}h$};
\node[prod, right=3.6mm of c6] (c7) {$\odot$};
\node[fix, below=1.8mm of c7] (c8) {$\sigma(z)$};
\draw[ar] (c1) -- (c2); \draw[ar] (c2) -- (c3); \draw[ar] (c3) -- (c4); \draw[ar] (c4) -- (c5);
\draw[ar] (c5) -- (c6); \draw[ar] (c6) -- (c7); \draw[ar] (c8) -- (c7);
\node[q, above=\QUP of c1] {1}; \node[q, above=\QUP of c2] {0};
\node[q, above=\QUP of c3] {0}; \node[q, above=\QUP of c4] {$0{+}0{=}0$};
\node[q, above=\QUP of c5] {$0{+}1{=}1$}; \node[q, above=\QUP of c6] {$0{+}1{=}1$};
\node[q, above=\QUP of c7] {$1{+}0$}; \node[q, left=0.4mm of c8] {0};
\node[res, right=2.5mm of c7] {no norm:\\{\color{blue!65!black}$\mathbf{q{=}1}$}};
\end{tikzpicture}
\caption{The linear-growth Mamba block; the two modified primitives are shaded.}
\label{fig:velocity-mamba-lin}
\end{figure}

\paragraph{Code and reproduction.}
ETTm1 \citep{zhou2021informer} and Weather \citep{wu2021autoformer} are
public. The supplementary directory \texttt{tsf\_small/} registers the CMamba
forecaster \citep{zeng2024cmamba} under the name \texttt{Mamba} and PatchTST
\citep{nie2023patchtst}, with one script per dataset and backbone under
\texttt{scripts/long\_term\_forecast/}. Fig.~\ref{fig:tsf} and
Tables~\ref{tab:blowup} and~\ref{tab:blowup-full} are built from the
spreadsheets in \texttt{test\_results/}, which are assembled from the per-run
test metrics.

\subsection{Operator learning}
\label{app:fno}

\paragraph{Data.}

\emph{Problem.} We follow the MathWorks battery-module example \citep{mathworksbattery}.
The heat equation of Sec.~\ref{sec:fno} is solved by finite elements on $\Omega=\Omega_{\rm cell}\cup\Omega_{\rm tab}\cup\Omega_{\rm conn}$ (Tab.~\ref{tab:materials}) up to $t=600$\,s, with convective cooling on the front and back faces $\Gamma_{\rm conv}$.
The operator maps the six-channel input
$u(x)=(x_1,x_2,x_3,\;T_{\rm amb}\mathbf 1_{\Omega}(x),\;h\,\mathbf 1_{\Gamma_{\rm conv}}(x),\;Q\,\mathbf 1_{\Omega_{\rm cell}}(x))$
to the terminal field $v(x)=T(x,600)$, both sampled on a $32^3$ grid; the coordinate channels let the network localize the geometry, and each parameter is extended by zero outside the region where it acts.

\begin{table}[h]
\centering
\caption{Material properties of the three regions \citep{mathworksbattery}.}
\label{tab:materials}
\small
\begin{tabular}{@{}lccc@{}}
\toprule
Region & $K$ [W/(m\,K)] & $\rho$ [kg/m$^3$] & $c_p$ [J/(kg\,K)] \\
\midrule
$\Omega_{\rm cell}$ & $\mathrm{diag}(2,80,80)$ & 780 & 785 \\
$\Omega_{\rm tab}$ & $386\,I$ & 2700 & 890 \\
$\Omega_{\rm conn}$ & $400\,I$ & 540 & 840 \\
\bottomrule
\end{tabular}
\end{table}

\emph{Parameter sensitivity.} Tab.~\ref{tab:fno-target} lists the target field when one parameter of a box is moved $50\%$ of its range width outside that box and the others are at mid-range.
The sensitivities quoted in Sec.~\ref{sec:fno} are finite differences of the spatial mean between the two shifts of each parameter.
$T_{\rm amb}$ translates the field without changing its spread, $Q$ rescales its excess over ambient and moves it outside the range of target values seen in training, and $h$ mainly changes its spread.
The three roles are the same in the larger box, and so are two of
the three magnitudes: the spatial-mean response is $1.00$\,K/K to
$T_{\rm amb}$ in both boxes and $4.04$ against $3.94$\,K/W to $Q$, close to
the pure-adiabatic $4.36$\,K/W, while the response to $h$ is $-0.16$ against
$-0.22$\,K per unit $h$. The larger box reaches further in absolute terms,
to a mean of $506.7$\,K and a maximum of $522.0$\,K at $Q=55$\,W, well
outside anything in its training range. Its $h=1$ row is the near-adiabatic
limit made visible: with convection all but switched off the field is almost
uniform, $[394.1,395.8]$\,K, which is the same fact that makes $h$ a nearly
inert parameter here.

\begin{table}[h]
\centering
\caption{Target temperature field $T(\cdot,600\,{\rm s})$ under one-parameter shifts of the original box, and of the larger box. $^\ast$: floored at $1$, so $30\%$ rather than $50\%$ of the range width below the lower bound, as in Tab.~\ref{tab:fno-ood-levels}.}
\label{tab:fno-target}
\small
\begin{tabular}{@{}lccc@{}}
\toprule
Shifted parameter & spatial mean [K] & 1st--99th pct.\ [K] & min--max [K] \\
\midrule
\multicolumn{4}{@{}l}{\emph{original box}} \\
mid-range & 350.5 & [341.4, 353.3] & [341.4, 353.3] \\
$T_{\rm amb}=270$\,K & 330.5 & [321.4, 333.3] & [321.4, 333.3] \\
$T_{\rm amb}=310$\,K & 370.5 & [361.4, 373.3] & [361.4, 373.3] \\
$h=5$ & 352.3 & [348.9, 353.4] & [348.9, 353.4] \\
$h=25$ & 349.1 & [335.5, 353.3] & [335.4, 353.3] \\
$Q=5$\,W & 310.2 & [307.1, 311.1] & [307.1, 311.1] \\
$Q=25$\,W & 390.9 & [375.6, 395.5] & [375.6, 395.6] \\
\midrule
\multicolumn{4}{@{}l}{\emph{larger box}} \\
mid-range & 388.5 & [365.8, 395.5] & [365.7, 395.5] \\
$T_{\rm amb}=250$\,K & 348.4 & [325.8, 355.4] & [325.6, 355.4] \\
$T_{\rm amb}=330$\,K & 428.5 & [405.8, 435.5] & [405.7, 435.5] \\
$h=1^\ast$ & 395.3 & [394.1, 395.8] & [394.1, 395.8] \\
$h=55$ & 383.5 & [345.8, 395.2] & [345.6, 395.2] \\
$Q=1^\ast$\,W & 293.9 & [293.0, 294.2] & [293.0, 294.2] \\
$Q=55$\,W & 506.7 & [456.8, 522.0] & [456.5, 522.0] \\
\bottomrule
\end{tabular}
\end{table}

\emph{OOD dataset.} The two training boxes are Tab.~\ref{tab:fno-domains}.
For each parameter with training range $[\ell,u]$ and width $w=u-\ell$, a
test value at distance $\delta$ lies $\delta w$ below the lower bound or
above the upper bound, and each parameter takes one of the three levels
\begin{equation*}
\ell-\delta w,\qquad (\ell+u)/2,\qquad u+\delta w,\qquad \delta\in\{10\%,25\%,50\%\}.
\end{equation*}
For each $\delta$, all $3^3=27$ combinations are simulated except the all-centre point, which lies inside the box; each of the remaining $26$ points is outside the training range in at least one parameter (Tab.~\ref{tab:fno-ood-levels}).
In the original box the centre values lie between training grid values, so no test point reuses a training value on any axis.
In the larger box the $T_{\rm amb}$ centre ($290$\,K) is a training grid value, but the $h$ and $Q$ centres are not, so no test point coincides with a training point.
Because $h$ and $Q$ are floored at $1$ in the larger box, their lowest level is $1$, i.e., $30\%$ rather than $50\%$ of the width below the range.
\begin{table}[h]
\centering
\caption{Training sets for the FNO. Iterations correspond to $300$ epochs.}
\label{tab:fno-domains}
\small
\begin{tabular}{@{}lcc@{}}
\toprule
& original box & larger box \\
\midrule
$T_{\rm amb}$ [K] & 280--300 & 270--310 \\
$h$ [W/(m$^2$K)] & 10--20 & 10--40 \\
$Q$ [W per cell] & 10--20 & 10--40 \\
values per axis & 6 / 6 / 6 & 11 / 16 / 16 \\
spacing & 4 / 2 / 2 & 4 / 2 / 2 \\
$N$ & 216 & 2816 \\
training iterations & 6{,}600 & 84{,}600 \\
\bottomrule
\end{tabular}
\end{table}

\begin{table}[h]
\centering
\caption{OOD test levels per parameter: $\ell-\delta w$, the centre, and $u+\delta w$ for $\delta\in\{10,25,50\}\%$. $^\ast$: floored at $1$.}
\label{tab:fno-ood-levels}
\small
\begin{tabular}{@{}llc ccc c ccc@{}}
\toprule
& & & \multicolumn{3}{c}{below} & & \multicolumn{3}{c}{above} \\
\cmidrule(lr){4-6}\cmidrule(lr){8-10}
Box & Parameter & $[\ell,u]$ & $-50\%$ & $-25\%$ & $-10\%$ & centre & $+10\%$ & $+25\%$ & $+50\%$ \\
\midrule
\multirow{3}{*}{original} & $T_{\rm amb}$ [K] & $[280,300]$ & 270 & 275 & 278 & 290 & 302 & 305 & 310 \\
& $h$ [W/(m$^2$K)] & $[10,20]$ & 5 & 7.5 & 9 & 15 & 21 & 22.5 & 25 \\
& $Q$ [W] & $[10,20]$ & 5 & 7.5 & 9 & 15 & 21 & 22.5 & 25 \\
\midrule
\multirow{3}{*}{larger} & $T_{\rm amb}$ [K] & $[270,310]$ & 250 & 260 & 266 & 290 & 314 & 320 & 330 \\
& $h$ [W/(m$^2$K)] & $[10,40]$ & $1^\ast$ & 2.5 & 7 & 25 & 43 & 47.5 & 55 \\
& $Q$ [W] & $[10,40]$ & $1^\ast$ & 2.5 & 7 & 25 & 43 & 47.5 & 55 \\
\bottomrule
\end{tabular}
\end{table}

\emph{Nonlinear variant: why.} The operator above is affine in two of its
three parameters, for structural rather than numerical reasons. Substituting
$\theta=T-T_{\rm amb}$ removes every occurrence of $T_{\rm amb}$, which is
constant in space and time and therefore cancels from the equation, the
convection boundary condition and the initial condition, and leaves a linear
problem with a source proportional to $Q$. Hence
\begin{equation}\label{eq:fno-affine}
T(x,t;T_{\rm amb},h,Q)=T_{\rm amb}+Q\,U_h(x,t).
\end{equation}
An extrapolation comparison on such a target is open to the objection that
what separates the two variants is the affine structure of the data rather
than the growth regime of the block. A second pair of sweeps therefore repeats
both boxes with a target operator that is nonlinear in $T_{\rm amb}$ and $Q$
and non-separable in the two.

\emph{Nonlinear variant: what changes.} The constant per-cell generation is
replaced by a temperature-activated source,
\begin{equation}\label{eq:fno-nl}
\rho c_p\,\partial_t T=\nabla\cdot(K_0\nabla T)
+Q\,e^{\beta(T-T_{\rm ref})}\,\mathbf 1_{\Omega_{\rm cell}},
\qquad \beta=0.002\ {\rm K}^{-1},\quad T_{\rm ref}=290\,{\rm K},
\end{equation}
with the conductivity $K_0$ of Tab.~\ref{tab:materials} and the boundary and
initial conditions unchanged. The substitution that gives
\eqref{eq:fno-affine} now fails, because the source depends on $T$ absolutely
rather than through $\theta$. In the adiabatic limit, writing
$a=Q/(\rho c_pV)$,
\begin{equation}\label{eq:fno-nl-adiabatic}
\frac{d\theta}{dt}=a\,e^{\beta(T_{\rm amb}-T_{\rm ref})}e^{\beta\theta}
\quad\Longrightarrow\quad
\theta=-\tfrac1\beta\log\big(1-a\,e^{\beta(T_{\rm amb}-T_{\rm ref})}\beta t\big),
\end{equation}
which is nonlinear in $Q$, nonlinear in $T_{\rm amb}$, and non-separable,
since the $T_{\rm amb}$ slope depends on $Q$. A source \emph{linear} in $T$
would not suffice: in the same limit it gives
$T=T_{\rm amb}e^{a\beta t}+\mathrm{const}$, still affine in $T_{\rm amb}$.

The form is the standard one for thermally activated generation in batteries,
$q\propto Ae^{-E_a/RT}$ \citep{bernardi1985energy,spotnitz2003abuse}, with
$e^{\beta(T-T_{\rm ref})}$ its linearized exponent,
$\beta\approx E_a/(RT_{\rm ref}^2)$. The magnitude is not a material value:
inverting at $T_{\rm ref}=290$\,K gives $E_a\approx1.4$\,kJ/mol against
$50$--$150$\,kJ/mol for lithium-ion decomposition, and the sign is that of the
abuse branch rather than of ohmic heating in normal operation. This is a
synthetic nonlinearity introduced to remove the affine structure of the
target, not a validated battery model; what the experiment requires of it is
that it be measurable and numerically sound.

\paragraph{Model.}
The FNO consists of a lifting network (convolution, LN, GELU, convolution),
four residual Fourier layers, and a projection network (convolution, LN,
GELU, convolution, inverse target normalization), with hidden width $64$ and
$4$ Fourier modes per dimension (Fig.~\ref{fig:fno-arch}).

\begin{figure}[h]
\centering
\begin{tikzpicture}[
  font=\scriptsize,
  node distance=2.4mm,
  op/.style={draw, rounded corners=1.5pt, inner sep=1.5pt, minimum height=4.6mm, align=center, fill=white},
  nrm/.style={op, dashed, draw=black!60, text=black!75},
  grp/.style={draw=black!35, rounded corners=2pt, inner sep=1.6mm},
  q/.style={text=black!70, inner sep=0.3pt},
  ar/.style={-{Stealth[length=1.3mm]}, black!80, line width=0.35pt},
  sum/.style={draw, circle, inner sep=0pt, minimum size=3.4mm, fill=white},
]
\def\QUP{0.3mm}
\node (u) at (0,0) {$u$};
\node[op, right=of u] (lift) {lifting\\{\tiny conv, LN, GELU, conv}};
\node[op, right=5mm of lift] (f1) {Fourier\\layer $0$};
\node[right=1.2mm of f1] (dots) {$\cdots$};
\node[op, right=1.2mm of dots] (f4) {Fourier\\layer $3$};
\node[op, right=5mm of f4] (proj) {projection\\{\tiny conv, LN, GELU, conv}};
\node[right=of proj] (v) {$T(\cdot,600\,{\rm s})$};
\draw[ar] (u) -- (lift); \draw[ar] (lift) -- node[above, q] {$X_0$} (f1);
\draw[ar] (f1) -- (dots); \draw[ar] (dots) -- (f4);
\draw[ar] (f4) -- node[above, q] {$X_4$} (proj); \draw[ar] (proj) -- (v);
\node[q, above=0.4mm of f1] {$X_{k+1}=X_k+v_k(X_k)$};
\begin{scope}[yshift=-17mm]
\node (xk) at (0,0) {$X_k$};
\node[op, right=7mm of xk, yshift=3.2mm] (K) {$\mathcal K_k$ (spectral conv)};
\node[op, right=7mm of xk, yshift=-3.2mm] (W) {$W_k$ (pointwise conv)};
\node[sum, right=6mm of $(K.east)!0.5!(W.east)$] (s1) {$+$};
\node[nrm, right=4mm of s1] (ln) {LN};
\node[op, right=of ln] (g) {GELU};
\node[sum, right=6mm of g] (s2) {$+$};
\node[right=of s2] (xk1) {$X_{k+1}$};
\draw[ar] (xk.east) -- ++(2.5mm,0) |- (K.west);
\draw[ar] (xk.east) -- ++(2.5mm,0) |- (W.west);
\draw[ar] (K.east) -| (s1.north); \draw[ar] (W.east) -| (s1.south);
\draw[ar] (s1) -- (ln); \draw[ar] (ln) -- (g); \draw[ar] (g) -- node[above, q] {$v_k$} (s2); \draw[ar] (s2) -- (xk1);
\draw[ar] (xk.south) -- ++(0,-5.5mm) -| node[pos=0.25, below, q] {identity} (s2.south);
\node[q, above=\QUP of K] {$q=1$};
\node[q, anchor=south west] at (s1.north east) {$1$};
\node[q, above=\QUP of ln] {$0$};
\node[q, above=\QUP of g] {$1$};
\node[align=left, anchor=west] at ($(xk1.east)+(3mm,0)$) (res)
  {with LN: $q(v_k)=1\cdot0\cdot1=0$\\[0.4mm]{\color{red!75!black}no-LN: $q(v_k)=1\cdot1=1$}};
\end{scope}
\end{tikzpicture}
\caption{FNO for the battery-cooling operator (top) and one residual Fourier layer (bottom), with the exponent of each primitive. The dashed LN is present in the LN variant ($q=0$) and removed in the no-LN variant ($q=1$); the LNs in the lifting and projection networks are kept in both variants. Fig.~\ref{fig:velocity-fno} gives the same layer in the format used for the other architectures.}
\label{fig:fno-arch}
\end{figure}

\paragraph{Protocol.}

\emph{Training.} Adam with a constant learning rate of $10^{-3}$, minibatch size $8$ with the training set reshuffled every epoch, an $80/20$ train/validation split, and $300$ epochs ($6{,}600$ iterations in the original box, $84{,}600$ in the larger box); $5$ seeds per variant.

\emph{Evaluation.} (i) The validation split: parameter sets held out from training, on the training grid. (ii) In-box: the $3^3=27$ combinations of three interior levels per parameter, $\{282,290,298\}$\,K and $\{11,15,19\}$ for $h$ and $Q$ in the original box, and $\{274,290,306\}$\,K and $\{13,25,37\}$ in the larger box. The $h$ and $Q$ levels lie between training grid values in both boxes, so no in-box point is a training point. (iii) OOD shells: the sets defined under \emph{OOD dataset} above, one per $\delta$. (iv) One parameter shifted: the six points of the $\delta=50\%$ shell in which exactly one parameter is off-centre.

\paragraph{Exponent derivation.}
Each residual Fourier layer has the velocity $v_k=\mathrm{GELU}\circ\mathrm{LN}\circ(\mathcal K_k+W_k)$, so the normalization is LayerNorm placed \emph{inside} the branch, between the linear part and the activation, rather than on the branch input (Pre-LN) or after the residual sum (Post-LN); we keep the example's scheme unchanged.
The spectral and pointwise convolutions are linear ($q=1$), so their sum is at $q=1$ by rule (ii); composition with LN ($q=0$) and GELU ($q=1$) gives $q(v_k)=1\cdot0\cdot1=0$ by rule (i).
The no-LN variant removes only the LN inside each residual Fourier layer, which leaves $q(v_k)=1\cdot1=1$. The Jacobian sensitivity exponent is $\sigma(v_k)=0$ in both variants: the two convolutions are linear and GELU has bounded derivative, so all three primitives have $\sigma=0$, and rules (i) and (ii) of Theorem~\ref{thm:sigma-arith} give $\sigma\le0\cdot1+0=0$. The FNO is therefore the one architecture here whose un-normalized variant is subcritical on both sides, $(q,\sigma)=(1,0)$, which is consistent with its residual stream never overflowing; what the exponent changes is extrapolation, not finiteness (Sec.~\ref{sec:fno}). The LNs in the lifting and projection networks are kept in both variants: these networks are applied once, outside the residual stream, so they do not enter the exponent of any $v_k$.

\begin{figure}[h]
\centering
\begin{tikzpicture}[
  font=\scriptsize,
  node distance=2.2mm,
  op/.style={draw, rounded corners=1.5pt, inner sep=1.4pt, minimum height=4.4mm, align=center, fill=white},
  nrm/.style={op, dashed, draw=black!60, text=black!75},
  fix/.style={op, draw=blue!70!black, fill=blue!8, line width=0.5pt},
  prod/.style={draw, circle, inner sep=0pt, minimum size=3.4mm, fill=white},
  q/.style={text=black!70, inner sep=0.3pt},
  res/.style={align=left, inner sep=0.5pt},
  ar/.style={-{Stealth[length=1.3mm]}, black!80, line width=0.35pt},
  lane/.style={font=\scriptsize\bfseries, anchor=west, text width=13mm, align=left, inner sep=0pt},
]
\def\QUP{0.3mm}
\def\QUP{0.3mm}
\node[lane] (le) at (0,0) {FNO};
\node[op, right=0mm of le.east, yshift=2.7mm] (f1) {$\mathcal K_kX$ (spectral conv)};
\node[op, right=0mm of le.east, yshift=-2.7mm] (f2) {$W_kX$ (pointwise conv)};
\node[prod, right=4mm of $(f1.east)!0.5!(f2.east)$] (f3) {$+$};
\node[nrm, right=5mm of f3] (f4) {LN};
\node[op, right=of f4] (f5) {GELU};
\draw[ar] (f1.east) -| (f3.north); \draw[ar] (f2.east) -| (f3.south);
\draw[ar] (f3) -- (f4); \draw[ar] (f4) -- (f5);
\node[q, above=\QUP of f1] {1};
\node[q, below=\QUP of f2] {1};
\node[q, above=\QUP of f4] {0};
\node[q, above=\QUP of f5] {1};
\node[res, right=2.5mm of f5] {no-LN: $q{=}1$, $\sigma{=}0$\\{\color{blue!65!black}LN: $q{=}0$, $\sigma{=}0$}};
\end{tikzpicture}
\caption{One residual Fourier layer of the FNO. Both convolutions are linear, so this is the one block here with $q=1$ and $\sigma=0$ without normalization.}
\label{fig:velocity-fno}
\end{figure}

\paragraph{Code and reproduction.}
The PDE model, geometry and reference network are the MathWorks
battery-module cooling example \citep{mathworksbattery}, included unmodified
under \texttt{example/} in the supplementary directory
\texttt{fno\_battery/}. The datasets are not public data but finite-element
solves, one per parameter combination, regenerated by \texttt{datagen/}; the
\texttt{README} states which scripts produce which set and why the example's
own parameter box is regenerated on the local mesh rather than read from the
shipped file. The CSVs behind Fig.~\ref{fig:fno},
Tab.~\ref{tab:fno-results} are included, so
those can be regenerated without retraining.
The same holds for the nonlinear variant: the per-seed and
aggregated CSVs are under \texttt{stability\_nl\_original/} and
\texttt{stability\_nl\_larger\_domain\_sampling\_rate\_preserved/}, and the
OOD scores under \texttt{ood/}, so Fig.~\ref{fig:fno-nl} and the lower
blocks of Tables~\ref{tab:fno-results} and~\ref{tab:qhat-fno} can be
regenerated without retraining. The source
term \eqref{eq:fno-nl} is defined in \texttt{datagen/cellHeatSpec.m}, the
calibration of $\beta$, of the solver tolerance and of the rejected
conductivity form in \texttt{datagen/calibrate\{Alpha,Tol,Beta,Form\}.m}, and
the two sweeps are launched by \texttt{run\_fno\_nl\_original.sh} and
\texttt{run\_fno\_nl\_srp.sh}. Both the source form and the conductivity form
default to the linear problem, so every dataset generated before this
experiment stays bit-reproducible.

\subsection{Language modeling: GPT-2 on OpenWebText}
\label{app:owt}

\paragraph{Data.}
OpenWebText \citep{gokaslan2019openwebtext}, about $8$ million documents and
$9$ billion tokens, tokenized once with the GPT-2 BPE vocabulary into a
single binary stream from which contiguous blocks of $1{,}024$ tokens are
drawn; the same stream is used by every variant.

\paragraph{Model.}
nanoGPT GPT-2 124M \citep{radford2019language,karpathy2022nanogpt}: $12$
layers, $D=24$ residual updates, $12$ heads, width $768$, $d_k=64$, context
$1{,}024$, biases, dropout $0$. Three variants: \emph{Pre-LN} as released
($q=0$); \emph{no-LN}, with \texttt{ln\_1}, \texttt{ln\_2} and \texttt{ln\_f}
replaced by identities ($q=1$, $\sigma=2$); and \emph{no-LN + QK-input norm},
which additionally applies RMS normalization to the residual-stream input
before the query and key projections, leaving the value projection and the
residual stream unnormalized ($q=1$, $\sigma=1$).

\paragraph{Protocol.}

\emph{Training.} Batch $12\times1{,}024$ tokens with $40$
gradient-accumulation steps ($491{,}520$ tokens per iteration); AdamW with
learning rate $6\times10^{-4}$, $750$ warm-up iterations, cosine decay to
$6\times10^{-5}$ at iteration $20{,}000$, $\beta=(0.9,0.95)$, weight decay
$0.1$, gradient clipping at $1.0$; bf16 autocast, whose largest finite value
$3.39\times10^{38}$ equals that of FP32. Seeds $1337$--$1341$ per variant.

\emph{Evaluation.} The validation loss is the mean over $200$ held-out
batches, evaluated every $250$ iterations, and a run is called diverged at the
first non-finite value of it. All five Pre-LN runs are finite through
iteration $20{,}000$. The no-LN runs first give a non-finite validation loss
at iterations $1{,}250$ (seed $1337$), $3{,}000$ ($1340$), $4{,}750$ ($1338$)
and $7{,}500$ ($1339$) 
since validation is only evaluated every $250$
iterations, only seed $1339$ has a finite checkpoint at $6{,}750$, which is
the checkpoint the depth profiles and the replay of App.~\ref{app:owt:results}
start from.

\paragraph{Exponent derivation.}
Each nanoGPT layer consists of two sequential residual sub-blocks, attention and MLP, giving $D=2\times12=24$ in  \eqref{eq:residual}.
Without normalization, attention has $(q,\sigma)=(1,2)$ by  Lemma~\ref{lem:attention-sigma}, while the Linear--GELU--Linear MLP has $(q,\sigma)=(1,0)$ under bounded weights by rules (i) and (ii). We compare three normalization variants (Fig.~\ref{fig:velocity-gpt2}).
Pre-LN normalizes the input to each residual branch, yielding $(q,\sigma)=(0,0)$. No-LN removes these normalizations, retaining $q=1$ but increasing $\sigma$ to $2$ through attention. Thus, its forward growth remains at the stability threshold, while its backward sensitivity exceeds it. No-LN + QK-input norm normalizes the inputs to the query and key projections, leaving the value projection and residual stream unnormalized. This reduces the attention sensitivity exponent to $\sigma=1$, since queries and keys have bounded growth while values retain linear growth, without changing $q=1$ due to the row-stochastic softmax. The resulting $(q,\sigma)=(1,1)$ isolates the effect of controlling backward sensitivity without reducing forward growth.

\begin{figure}[h]
\centering
\begin{tikzpicture}[
  font=\scriptsize,
  node distance=2.2mm,
  op/.style={draw, rounded corners=1.5pt, inner sep=1.4pt, minimum height=4.4mm, align=center, fill=white},
  nrm/.style={op, dashed, draw=black!60, text=black!75},
  fix/.style={op, draw=blue!70!black, fill=blue!8, line width=0.5pt},
  prod/.style={draw, circle, inner sep=0pt, minimum size=3.4mm, fill=white},
  q/.style={text=black!70, inner sep=0.3pt},
  res/.style={align=left, inner sep=0.5pt},
  ar/.style={-{Stealth[length=1.3mm]}, black!80, line width=0.35pt},
  lane/.style={font=\scriptsize\bfseries, anchor=west, text width=13mm, align=left, inner sep=0pt},
]
\def\QUP{0.3mm}
\node[lane] (lf) at (0,0) {GPT-2};
\node[op, right=0mm of lf.east] (fnum) {\bfseries 1};
\node[nrm, right=1.2mm of fnum] (g0) {LayerNorm};
\node[op, right=of g0] (g1) {$Q,K,V{=}WX$};
\node[op, right=of g1] (g2) {$\mathrm{softmax}(QK^{\top}\!/\!\sqrt{d_k})V$};
\node[op, right=of g2] (g3) {$W_O$};
\draw[ar] (fnum) -- (g0); \draw[ar] (g0) -- (g1); \draw[ar] (g1) -- (g2); \draw[ar] (g2) -- (g3);
\node[q, above=\QUP of g1] {1};
\node[q, above=\QUP of g2] {$0{\cdot}2{+}1{=}1$};
\node[q, above=\QUP of g3] {1};
\node[res, right=2.5mm of g3] {attention\\ no-LN: $q{=}1$, {\color{red!75!black}$\sigma{=}2$}};
\node[op] (inum) at ($(fnum)+(0,-1.05)$) {\bfseries 1$'$};
\node[fix, right=1.2mm of inum] (i1) {$Q,K{=}W_{Q,K}N(X)$\\[-1pt] $V{=}W_VX$};
\node[op, right=of i1] (i2) {$\mathrm{softmax}(QK^{\top}\!/\!\sqrt{d_k})V$};
\node[op, right=of i2] (i3) {$W_O$};
\draw[ar] (inum) -- (i1); \draw[ar] (i1) -- (i2); \draw[ar] (i2) -- (i3);
\node[q, above=\QUP of i1] {$Q,K{:}\,0$, $V{:}\,1$};
\node[q, above=\QUP of i2] {$0{\cdot}2{+}1{=}1$};
\node[q, above=\QUP of i3] {1};
\node[res, right=2.5mm of i3] {attention, QK-input norm\\ $q{=}1$, {\color{blue!60!black}$\sigma{=}1$}};
\node[op] (hnum) at ($(fnum)+(0,-2.1)$) {\bfseries 2};
\node[nrm, right=1.2mm of hnum] (h0) {LayerNorm};
\node[op, right=of h0] (h1) {Linear};
\node[op, right=of h1] (h2) {GELU};
\node[op, right=of h2] (h3) {Linear};
\draw[ar] (hnum) -- (h0); \draw[ar] (h0) -- (h1); \draw[ar] (h1) -- (h2); \draw[ar] (h2) -- (h3);
\node[q, above=\QUP of h1] {1}; \node[q, above=\QUP of h2] {1}; \node[q, above=\QUP of h3] {1};
\node[res, right=2.5mm of h3] {MLP\\ no-LN: $q{=}1$, $\sigma{=}0$};
\node[q, below=1.2mm of hnum.south west, anchor=north west, text width=108mm, align=left]
  {$12$ layers $\times$ $2$ sub-blocks gives $D{=}24$; without LN every $q_k\le1$, yet $\sigma{=}2$ in \textbf{1} breaks the backward certificate.
   Normalizing only the inputs of the query and key projections (\textbf{1}$'$, blue) leaves $V$ at degree $1$ and takes $Q,K$ to degree $0$, so $q$ is unchanged and $\sigma$ drops to $1$; the MLP is the same in both variants.};
\end{tikzpicture}
\caption{One GPT-2 layer as implemented in nanoGPT \citep{karpathy2022nanogpt}: two residual sub-blocks, attention then a plain MLP. Unlike the gated MLP of Fig.~\ref{fig:velocity-toto}, the plain MLP stays at $q=1$. Dashed boxes are the LayerNorms, present in the Pre-LN variant and removed in the other two. Row \textbf{1}$'$ is the attention sub-block of the third variant, which normalizes the inputs of the query and key projections only; it is the same $q$ with $\sigma$ brought down to $1$. The MLP is shared by all three.}
\label{fig:velocity-gpt2}
\end{figure}

\paragraph{Code and reproduction.}
OpenWebText \citep{gokaslan2019openwebtext} is public and is tokenized by
\texttt{owt\_dataset/prepare.py} in the supplementary directory
\texttt{nanogpt\_openwebtext/}; the resulting binary is not shipped. The model
is nanoGPT \citep{karpathy2022nanogpt} with a \texttt{use\_layernorm} flag
added, which is what selects the no-LN variant above.
\texttt{REPRODUCE.md} lists the six stages from tokenization through the four
CSVs that Fig.~\ref{fig:owt} reads. Divergence replays are not reproducible
to the iteration: disabling compilation changes the numerics and the
per-iteration probe is itself a perturbation, so a replay trips wherever it
trips, and the analysis is indexed to the detected onset rather than to an
absolute iteration.

\subsection{Time-series foundation model: Toto~2.0}
\label{app:toto}

\paragraph{Data.}
\emph{Training corpus.} We use a $3.73$ GB subset of the GIFT-Eval pretraining corpus \citep{aksu2024gifteval}, comprising $28$ datasets across energy, transport, nature, web and cloud, sales, economics, and health. The pretraining corpus is provided separately from the GIFT-Eval evaluation datasets to avoid data leakage. We reserve a held-out portion of the training subset for in-distribution evaluation.

\emph{OOD dataset.} We use the GIFT-Eval benchmark \citep{aksu2024gifteval}, comprising $23$ datasets with approximately $144{,}000$ time series and $177$ million observations across seven domains and ten sampling frequencies. Evaluation covers $97$ dataset-frequency-horizon configurations, with prediction horizons ranging from $6$ to $900$ steps.

\paragraph{Model.}
Toto~2.0 \citep{khwaja2026toto} provides released checkpoints ranging from $4$M to $2.5$B parameters. We train a $10.5$M-parameter proxy from scratch using the released architecture and training recipe, preserving its architectural components while reducing model size. The proxy has $12$ layers ($11$ time-axis and $1$ variate-axis), each with attention and gated-MLP residual updates ($D=24$), hidden width $256$, $4$ attention heads, gated-MLP width $688$, patch size $32$, and context length $4{,}096$. We retain the original attention scaling, positional encoding, u-$\mu$P residual parameterization, and optimizer configuration.

We compare five normalization variants: \emph{Pre-RMSNorm}, the baseline with parameter-free RMSNorm on each residual branch input; \emph{no-LN}, which removes these normalizations; \emph{gate-norm}, which normalizes one input to the gated MLP; \emph{QK-input norm}, which normalizes the query and key inputs; and \emph{gate+QK-input norm}, which combines both modifications. The value projection remains unnormalized in all variants. All variants have identical parameter counts ($10{,}477{,}376$) and share the same initialization for each seed.

\paragraph{Protocol.}

\emph{Training.} The released $10$M-proxy hyperparameters with the $30$k-step
schedule shrunk proportionally to $5{,}000$ steps: warmup over the first
$1{,}000$ steps, constant, then linear decay to zero over the last $1{,}750$;
NorMuon on the two-dimensional weights of the Transformer layers (learning rate
$0.65$, $\mu=0.96$, $\beta_2=0.999$, weight decay $2\times10^{-8}$, Nesterov,
Polar Express) and AdamW on the patch embedding, output head, biases and
PerDimScale (learning rate $0.012$, betas $(0.91,0.972)$, no weight decay),
both with u-$\mu$P fan-in and depth scaling at depth $12$; batch $64$ samples as
four micro-batches of $16$; pinball loss over nine quantiles on next-patch
prediction of $\mathrm{asinh}((y-\mathrm{loc})/\mathrm{scale})$; contiguous
patch masking with span $c\sim U\{1,\dots,16\}$ patches and masking probability
$p\sim U(0,0.4)$ per series; gradient clipping at $7.0$; bf16 autocast with
\texttt{torch.compile} off; seeds $1337$--$1341$. A run halts at the first non-finite loss or
parameter. One implementation detail matters: NorMuon updates its
learning-rate tensor in place, so the schedule must hold the base rate
separately and write into that tensor, and the trainer asserts every step that
the optimizer's actual rate equals the schedule. An earlier version aliased the
tensor, which drove the rate to zero during warmup; those runs were discarded.

\emph{Stability diagnostics.}
During training, we monitor loss, gradient norms, residual-stream magnitudes, attention logits, and non-finite values to identify instability. Every $250$ steps, we additionally estimate the empirical forward growth and Jacobian sensitivity exponents for each residual sub-block using the procedures in App.~\ref{app:qhat}. Checkpoints are saved every $1{,}000$ steps.

\emph{Evaluation.} We follow the official GIFT-Eval evaluation protocol (commit \texttt{cfaefe09}), covering $97$ configurations across short-, medium-, and long-term horizons. Prediction lengths are determined by sampling frequency and horizon multiplier, with the number of evaluation windows set to $\min(\max(1,\lceil 0.1,\mathrm{min_len}/H\rceil),20)$. We use a context length of $4{,}096$, the standard GluonTS wrapper, and all eleven metrics from the official evaluation notebook.

Our implementation differs from the official protocol in three respects: (i) Arrow shards are read directly using \texttt{pyarrow}, one shard at a time, because the \texttt{gift-eval} package has dependency conflicts with \texttt{toto-2}; (ii) multivariate forecasts are transposed to match the format expected by GluonTS~$0.17$; and (iii) the reference baseline is the GluonTS \texttt{SeasonalNaivePredictor}, evaluated per variate, rather than the official \texttt{statsforecast} baseline.

For in-distribution evaluation, we apply the same evaluation code to the held-out training subset, restricted to short-term horizons and at most $300$ items per dataset. For time-split datasets, evaluation windows remain within the held-out tail. We exclude \texttt{solar\_power/short} from the aggregate because its all-zero nighttime windows yield infinite CRPS.

We evaluate the Pre-RMSNorm, gate-norm, and gate+QK-input norm variants at steps $1{,}000$ and $5{,}000$, and the QK-input norm variant at step $1{,}000$. The no-LN variant has no finite checkpoint beyond initialization.

\paragraph{Exponent derivation.}
Each Toto~2.0 layer consists of two sequential residual sub-blocks, attention followed by a gated MLP, each with its own skip connection and parameter-free RMSNorm applied to the branch input. Thus, the depth in \eqref{eq:residual} counts residual sub-blocks rather than Transformer layers, giving $D=2\times12=24$ for our proxy. The distinction between time-axis and variate-axis attention affects the attention dimension but not the growth or sensitivity exponents.

\emph{Gated MLP branch.} The feed-forward network is SwiGLU, $W_2\bigl((W_gX)\odot\mathrm{SiLU}(W_uX)\bigr)$: one linear gate multiplied entrywise by a Lipschitz activation of a second linear map, with the two halves produced by a single fused projection. Both factors have $q=1$, so the Hadamard product has $q=2$ by (iii) and the output projection keeps it there by (i). For the Jacobian, rule (iii) of Theorem~\ref{thm:sigma-arith} gives $\sigma\le\max(0+1,1+0)=1$, and (i) leaves it at $1$. The branch is therefore at $(q,\sigma)=(2,1)$.

\emph{Attention branch.} As in PatchTST, $(q,\sigma)=(1,2)$. The released configuration disables query-key normalization, so the logits are the unmodified bilinear form and Lemma~\ref{lem:attention-sigma} applies verbatim, with $\sqrt{d_k}$ replaced by $d_k$ because the implementation uses the $\mu$P logit scaling $1/d_k$. The per-dimension query scale and the rotary or xPos projection are linear maps and fold into $\kappa$ without changing an exponent.

\emph{Without normalization.} Removing the two RMSNorms leaves
\begin{equation*}
(q,\sigma)_{\rm attn}=(1,2),\qquad (q,\sigma)_{\rm MLP}=(2,1),
\end{equation*}
so over the $24$ updates $\max_kq_k=2$ and $\max_k\sigma_k=2$. The forward certificate fails through the gate and the backward certificate fails through the attention, the only architecture we examine in which both fail at once. Normalizing the input of the gate returns the MLP branch to $q=1$ without a normalization layer, as in the linear-growth Mamba block, and enabling query-key normalization bounds the logits and lowers the attention branch from $\sigma=2$ to $\sigma=1$. Any SwiGLU or GeGLU Transformer has the same structure.

\emph{The residual form is not additive.} The released model combines each sub-block with its skip by a $\tau$-rule rather than by addition. In the forward pass the split is the identity and the recombination is a weighted sum, so
\begin{equation*}
X_{k+1}=\beta_kX_k+\alpha_k\,F_k(N(X_k)),\qquad
\beta_k=(1+\tau_k^2)^{-1/2},\quad\alpha_k=\tau_k\beta_k,\quad\alpha_k^2+\beta_k^2=1,
\end{equation*}
with $\tau_k$ decreasing with depth and separate values for the two sub-blocks ($\tau^{\rm attn}=0.150$ and $\tau^{\rm mlp}=0.0289$ at the first layer, in the ratio $5.14$ fixed by the configuration). This is \eqref{eq:residual} with $\Delta t=1$ and $v_k(x)=(\beta_k-1)x+\alpha_kF_k(N(x))$, so the damping term puts the normalized block at $q=1$ with $c_k=1-\beta_k$ ($1.1\times10^{-2}$ and $4.2\times10^{-4}$ at the first layer), not at $q=0$. Since $\beta_k<1$, the exact recursion $\|X_{k+1}\|\le\beta_k\|X_k\|+\alpha_kb_k$ is a contraction and $\mathcal R_D$ is bounded uniformly in $D$, whereas Assumption~\ref{ass:growth} absorbs the sign of the damping and reports growth $(1+c_k)$; the certificate is sound but conservative for this block.

\begin{figure}[h]
\centering
\begin{tikzpicture}[
  font=\scriptsize,
  node distance=2.2mm,
  op/.style={draw, rounded corners=1.5pt, inner sep=1.4pt, minimum height=4.4mm, align=center, fill=white},
  nrm/.style={op, dashed, draw=black!60, text=black!75},
  fix/.style={op, draw=blue!70!black, fill=blue!8, line width=0.5pt},
  prod/.style={draw, circle, inner sep=0pt, minimum size=3.4mm, fill=white},
  q/.style={text=black!70, inner sep=0.3pt},
  res/.style={align=left, inner sep=0.5pt},
  ar/.style={-{Stealth[length=1.3mm]}, black!80, line width=0.35pt},
  lane/.style={font=\scriptsize\bfseries, anchor=west, text width=13mm, align=left, inner sep=0pt},
]
\def\QUP{0.3mm}
\node[lane, text width=13mm] (ld) at (0,0) {Toto 2.0};
\node[op, right=0mm of ld.east] (dnum) {\bfseries 1};
\node[nrm, right=1.2mm of dnum] (d0) {RMSNorm};
\node[fix, right=of d0] (d1) {$Q,K{=}W_{Q,K}Z$\\[-1pt] $V{=}W_VX$};
\node[op, right=of d1] (d2) {$\mathrm{softmax}(QK^{\top}\!/d_k)V$};
\node[op, right=of d2] (d3) {$W_O$};
\draw[ar] (dnum) -- (d0); \draw[ar] (d0) -- (d1); \draw[ar] (d1) -- (d2); \draw[ar] (d2) -- (d3);
\node[q, above=\QUP of d1] {$V{:}\,1$};
\node[q, above=\QUP of d2] {$0{\cdot}2{+}1{=}1$};
\node[q, above=\QUP of d3] {1};
\node[res, right=2.5mm of d3] {attention: $q{=}1$,\\ {\color{red!75!black}$\sigma{=}2$} if $Z{=}X$, {\color{blue!60!black}$\sigma{=}1$} if $Z{=}\hat X$};
\node[op] (enum) at ($(dnum)+(0,-1.15)$) {\bfseries 2};
\node[nrm, right=1.2mm of enum] (e0) {RMSNorm};
\node[fix, right=of e0] (e1) {$W_aX,\ W_bZ'$};
\node[op, right=of e1] (e2) {$W_aX{\odot}\mathrm{SiLU}(W_bZ')$};
\node[op, right=of e2] (e3) {$W_2$};
\draw[ar] (enum) -- (e0); \draw[ar] (e0) -- (e1); \draw[ar] (e1) -- (e2); \draw[ar] (e2) -- (e3);
\node[q, above=\QUP of e1] {$1$, $\deg Z'$};
\node[q, above=\QUP of e2] {$1{+}\deg Z'$};
\node[q, above=\QUP of e3] {$1{+}\deg Z'$};
\node[res, right=2.5mm of e3] {gated MLP: $\sigma{=}1$,\\ {\color{red!75!black}$q{=}2$} if $Z'{=}X$, {\color{blue!60!black}$q{=}1$} if $Z'{=}\hat X$};
\node[q, below=1.4mm of enum.south west, anchor=north west, text width=112mm, align=left] (note)
  {One layer runs \textbf{1} then \textbf{2}, each wrapped in its own residual connection, so $D{=}2\times12=24$.
   $\hat X$ is the parameter-free RMSNorm, and $\deg Z'$ is $1$ for $Z'{=}X$ and $0$ for $Z'{=}\hat X$.
   The two blue boxes are independent switches: $Z$ sets the $Q,K$ inputs and moves $\sigma$, $Z'$ sets the gate input and moves $q$.
   Their four combinations are the four arms trained without residual-stream norms; Pre-RMSNorm keeps the dashed norms instead and is $(0,0)$.};
\node[anchor=north west] at ($(note.south west)+(3mm,-1mm)$)
  {\scriptsize\begin{tabular}{@{}llll@{}}
   \toprule
   $Z$ ($Q,K$) & $Z'$ (gate) & arm & $(q,\sigma)$ \\
   \midrule
   $X$ & $X$ & no-LN & $(2,2)$ \\
   $\hat X$ & $X$ & QK-input norm & $(2,1)$ \\
   $X$ & $\hat X$ & gate-norm & $(1,2)$ \\
   $\hat X$ & $\hat X$ & gate\,+\,QK-input norm & $(1,1)$ \\
   \bottomrule
   \end{tabular}};
\end{tikzpicture}
\caption{One Toto~2.0 layer consists of two sequential residual sub-blocks: attention and a gated MLP. The highlighted inputs $Z$ and $Z'$ independently control the attention sensitivity exponent $\sigma$ and the gated-MLP growth exponent $q$, respectively. Their four combinations define the no-LN normalization variants in Tab.~\ref{tab:toto-arms}, while Pre-RMSNorm retains the dashed residual-stream normalizations and yields $(q,\sigma)=(0,0)$. The value projection $W_V$ always receives the unnormalized residual stream.}
\label{fig:velocity-toto}
\end{figure}

\paragraph{Code and reproducibility.}
The Toto~2.0 architecture and training methodology are publicly available \citep{khwaja2026toto}. Our experiments use a smaller proxy trained on the GIFT-Eval pretraining subset described above, rather than the original mixture containing proprietary Datadog data. The experimental implementation is not included in the supplementary material; the model configuration, normalization variants, training procedure, and evaluation protocol are documented above to facilitate reproduction.

\section{Additional experiments}
\label{app:results}

This appendix provides detailed experimental results and additional analyses supporting Sec.~\ref{sec:results}. Following the same four tasks, we report training stability statistics, empirical growth exponents, and supplementary evaluations of generalization and backward sensitivity. The corresponding datasets, model configurations, training procedures, evaluation protocols, and architectural exponent derivations are given in App.~\ref{app:experiments}.

\subsection{Time-series forecasting}
\label{app:tsf:results}

\paragraph{Blow-up counts.}
Tab.~\ref{tab:blowup-full} breaks down the blow-up rates in Tab.~\ref{tab:blowup} by dataset, depth, and prediction horizon. Blow-ups occur exclusively in Mamba no-LN ($q=5$), particularly at greater depth, while all variants with $q\leq 1$ remain stable. 

\begin{table}[h]
\centering
\caption{Blow-up counts corresponding to results reported in Tab.~\ref{tab:blowup}. ETTm1: 6 seeds per cell; Weather: 3 seeds per cell.}
\label{tab:blowup-full}
\small
\setlength{\tabcolsep}{4pt}
\begin{tabular}{@{}llcc cccc c cccc c@{}}
\toprule
& & & & \multicolumn{5}{c}{ETTm1} & \multicolumn{5}{c}{Weather} \\
\cmidrule(lr){5-9}\cmidrule(l){10-14}
Backbone & Variant & $q$ & $D$ & 96 & 192 & 336 & 720 & $\Sigma/24$ & 96 & 192 & 336 & 720 & $\Sigma/12$ \\
\midrule
Mamba & no-LN & 5 & 3 & 2 & 3 & 4 & 3 & 12 & 0 & 1 & 0 & 1 & 2 \\
 & & & 8 & 6 & 6 & 6 & 6 & 24 & 1 & 3 & 3 & 3 & 10 \\
 & RMSNorm & 0 & 3, 8 & 0 & 0 & 0 & 0 & 0 & 0 & 0 & 0 & 0 & 0 \\
 & Pre-LN & 0 & 3, 8 & 0 & 0 & 0 & 0 & 0 & 0 & 0 & 0 & 0 & 0 \\
 & Peri-LN & 0 & 3, 8 & 0 & 0 & 0 & 0 & 0 & 0 & 0 & 0 & 0 & 0 \\
 & linear growth & 1 & 3, 8 & 0 & 0 & 0 & 0 & 0 & 0 & 0 & 0 & 0 & 0 \\
\midrule
PatchTST & no-LN & 1 & 3, 8 & 0 & 0 & 0 & 0 & 0 & 0 & 0 & 0 & 0 & 0 \\
 & RMSNorm & 0 & 3, 8 & 0 & 0 & 0 & 0 & 0 & 0 & 0 & 0 & 0 & 0 \\
 & Pre-LN & 0 & 3, 8 & 0 & 0 & 0 & 0 & 0 & 0 & 0 & 0 & 0 & 0 \\
 & Peri-LN & 0 & 3, 8 & 0 & 0 & 0 & 0 & 0 & 0 & 0 & 0 & 0 & 0 \\
\bottomrule
\end{tabular}
\end{table}

\paragraph{Forward exponent of the trained models.}
Tab.~\ref{tab:qhat-mamba} reports the empirical forward exponents $\qhat_k$ (App.~\ref{app:qhat}) for the runs in Tab.~\ref{tab:blowup}. The estimates distinguish the certified growth regimes: normalized variants yield $\qhat_k\approx 0$, linear-growth Mamba remains near or below $1$, and no-LN Mamba reaches $5.02$, consistent with its certified $q=5$. PatchTST no-LN has a mean estimate of $0.979$ ($95\%$ CI $[0.957,1.001]$), close to its certified $q=1$. Unlike the global certified bounds, these estimates measure local growth along trained trajectories, explaining why their averages can be smaller than the certified exponents. For Mamba no-LN, estimates are available for $30$ of $72$ runs with finite end-of-training checkpoints, including $11$ with non-finite final test MSE. Their mean estimates are comparable to those of the $19$ non-divergent runs ($1.89$ versus $1.94$).

\begin{table}[h]
\centering
\caption{Forward exponent estimates $\qhat_k$ for the forecasting models of Sec.~\ref{sec:tsf}, measured at the end of training. Each run contributes its layerwise mean and maximum; reported values are averaged across runs with $95\%$ $t$-intervals, except the final column, which reports the maximum across all runs and layers. Each variant has $72$ runs, except Mamba no-LN, for which $30$ finite checkpoints are available. Corresponding blow-up rates are reported in  Tab.~\ref{tab:blowup}.}
\label{tab:qhat-mamba}
\small
\begin{tabular}{@{}llccccc@{}}
\toprule
& & & & \multicolumn{2}{c}{mean over runs} & largest over runs \\
\cmidrule(lr){5-6}\cmidrule(l){7-7}
Backbone & Variant & $q$ & runs
 & $\tfrac1L\sum_{k}\qhat_k$
 & $\max_k\qhat_k$
 & $\max_{\rm runs}\max_k\qhat_k$ \\
\midrule
\multirow{5}{*}{Mamba} & RMSNorm & $0$ & $72$ & $0.001\ [0.000,0.002]$ & $0.007\ [0.005,0.009]$ & $0.033$ \\
 & Pre-LN        & $0$ & $72$ & $0.002\ [0.001,0.004]$ & $0.010\ [0.007,0.014]$ & $0.064$ \\
 & Peri-LN       & $0$ & $72$ & $0.002\ [0.001,0.003]$ & $0.009\ [0.008,0.011]$ & $0.030$ \\
 & linear growth & $1$ & $72$ & $0.525\ [0.485,0.566]$ & $0.721\ [0.667,0.775]$ & $1.093$ \\
 & no-LN         & $5$ & $30$ & $1.920\ [1.752,2.087]$ & $2.694\ [2.385,3.004]$ & $5.019$ \\
\midrule
\multirow{4}{*}{PatchTST} & RMSNorm & $0$ & $72$ & $0.000\ [0.000,0.000]$ & $0.002\ [0.002,0.002]$ & $0.003$ \\
 & Pre-LN  & $0$ & $72$ & $0.000\ [0.000,0.000]$ & $0.002\ [0.002,0.002]$ & $0.004$ \\
 & Peri-LN & $0$ & $72$ & $0.000\ [0.000,0.000]$ & $0.002\ [0.002,0.003]$ & $0.004$ \\
 & no-LN   & $1$ & $72$ & $0.979\ [0.957,1.001]$ & $1.208\ [1.181,1.236]$ & $1.430$ \\
\bottomrule
\end{tabular}
\end{table}

\subsection{Operator learning}
\label{app:fno:results}

\paragraph{Nonlinear source term.}
To test whether the OOD advantage of no-LN depends on the affine structure of the original target operator \eqref{eq:fno-affine}, we repeat the experiments with the temperature-dependent Arrhenius source term described in App.~\ref{app:fno}. 
Fig.~\ref{fig:fno-nl} and the Arrhenius rows of Tab.~\ref{tab:fno-results} show that no-LN retains lower RMSE on every OOD shell in both parameter boxes, with the gap increasing at larger extrapolation distances. At $50\%$ extrapolation, LN has $3.4\times$ and $3.9\times$ the RMSE of no-LN in the original and larger boxes, respectively.

The nonlinear source generally increases OOD errors more for LN than for no-LN, particularly along the $Q^+$ direction, where the largest increase is $28.46$\,K versus $6.32$\,K in the larger box. In-distribution performance remains comparable, although the relative ordering varies across sweeps. Thus, the no-LN advantage persists beyond the affine target operator, within the nonlinear setting considered here.

\begin{figure}[h]
\centering
\includegraphics[width=\linewidth]{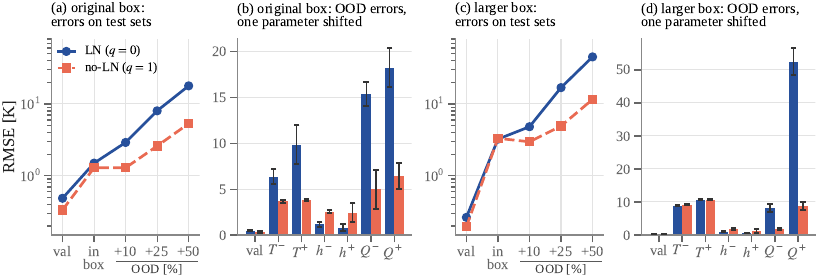}
\caption{FNO test RMSE for LN ($q=0$) and no-LN ($q=1$) with the
temperature-dependent source term, in the format of Fig.~\ref{fig:fno}. The
original parameter box is used in (a,b), and the larger box at the original
sampling density in (c,d). (a,c) Test sets ordered by distance from the
training data: validation, off-grid points within the training box, and OOD
points at $10$, $25$, and $50\%$ beyond the box. (b,d) One parameter shifted
$50\%$ beyond its range, with the others fixed at mid-range; whiskers show
$\pm1$ sd. Note the larger vertical range than in Fig.~\ref{fig:fno}.}
\label{fig:fno-nl}
\end{figure}

\paragraph{In-distribution and out-of-distribution prediction errors.}
Tab.~\ref{tab:fno-results} reports the FNO prediction errors for the constant and temperature-dependent source terms, corresponding to Figures~\ref{fig:fno} and \ref{fig:fno-nl}.
We evaluate both normalization variants on validation, in-distribution, and OOD parameter sets, including extrapolation along individual parameter axes. Errors are reported as RMSE averaged over five seeds, with the differences between source terms included to assess the effect of target nonlinearity.

\begin{table}[h]
\centering
\caption{FNO test RMSE [K] (mean $\pm$ sd over $5$ seeds) for constant and Arrhenius source terms (Figures~\ref{fig:fno} and~\ref{fig:fno-nl}). \textbf{(a)} Errors across increasing OOD distances. \textbf{(b)} Errors with one parameter shifted $50\%$ beyond its training range and the others fixed at their midpoints. The final block of each panel reports $\Delta\mathrm{RMSE}=\mathrm{RMSE}_{\rm Arr}-\mathrm{RMSE}_{\rm const}$ in kelvin, so that both variants are compared on one scale.
}
\label{tab:fno-results}
\small
\setlength{\tabcolsep}{3pt}

\textbf{(a) Distance from the training data}\\[2pt]
\begin{tabular*}{0.92\linewidth}{@{\extracolsep{\fill}}ll ccccc@{}}
\toprule
Box & Variant ($q$) & val & in box & $+10\%$ & $+25\%$ & $+50\%$ \\
\midrule
\multicolumn{7}{@{}l}{\emph{constant source term}} \\
\multirow{2}{*}{original} & LN (0) & $0.38{\scriptstyle\pm0.09}$ & $1.71{\scriptstyle\pm0.10}$ & $2.72{\scriptstyle\pm0.19}$ & $6.92{\scriptstyle\pm0.37}$ & $15.29{\scriptstyle\pm0.57}$ \\
 & no-LN (1) & $0.30{\scriptstyle\pm0.11}$ & $2.51{\scriptstyle\pm0.17}$ & $2.19{\scriptstyle\pm0.14}$ & $2.91{\scriptstyle\pm0.32}$ & $5.26{\scriptstyle\pm0.95}$ \\
\addlinespace[2pt]
\multirow{2}{*}{larger} & LN (0) & $0.31{\scriptstyle\pm0.24}$ & $2.73{\scriptstyle\pm0.15}$ & $3.51{\scriptstyle\pm0.37}$ & $10.89{\scriptstyle\pm0.96}$ & $24.06{\scriptstyle\pm1.05}$ \\
 & no-LN (1) & $0.23{\scriptstyle\pm0.13}$ & $2.96{\scriptstyle\pm0.15}$ & $2.55{\scriptstyle\pm0.16}$ & $3.15{\scriptstyle\pm0.18}$ & $4.34{\scriptstyle\pm0.19}$ \\
\midrule
\multicolumn{7}{@{}l}{\emph{Arrhenius source term}} \\
\multirow{2}{*}{original} & LN (0) & $0.49{\scriptstyle\pm0.03}$ & $1.50{\scriptstyle\pm0.17}$ & $2.91{\scriptstyle\pm0.26}$ & $8.02{\scriptstyle\pm0.55}$ & $17.87{\scriptstyle\pm0.84}$ \\
  & no-LN (1) & $0.33{\scriptstyle\pm0.08}$ & $1.29{\scriptstyle\pm0.04}$ & $1.29{\scriptstyle\pm0.07}$ & $2.60{\scriptstyle\pm0.21}$ & $5.33{\scriptstyle\pm0.20}$ \\
\addlinespace[2pt]
\multirow{2}{*}{larger} & LN (0) & $0.26{\scriptstyle\pm0.14}$ & $3.24{\scriptstyle\pm0.19}$ & $4.82{\scriptstyle\pm0.50}$ & $16.87{\scriptstyle\pm1.23}$ & $45.22{\scriptstyle\pm1.88}$ \\
  & no-LN (1) & $0.20{\scriptstyle\pm0.08}$ & $3.32{\scriptstyle\pm0.04}$ & $2.97{\scriptstyle\pm0.08}$ & $4.93{\scriptstyle\pm0.14}$ & $11.57{\scriptstyle\pm0.42}$ \\
\midrule
\multicolumn{7}{@{}l}{\emph{change in RMSE [K], constant to Arrhenius}} \\
\multirow{2}{*}{original} & LN (0) & $+0.11$ & $-0.21$ & $+0.19$ & $+1.10$ & $+2.58$ \\
  & no-LN (1) & $+0.04$ & $-1.21$ & $-0.90$ & $-0.31$ & $+0.07$ \\
\addlinespace[2pt]
\multirow{2}{*}{larger} & LN (0) & $-0.05$ & $+0.50$ & $+1.31$ & $+5.98$ & $+21.16$ \\
  & no-LN (1) & $-0.04$ & $+0.37$ & $+0.43$ & $+1.78$ & $+7.24$ \\
\bottomrule
\end{tabular*}

\vspace{8pt}
\textbf{(b) One parameter shifted $50\%$ beyond its range}\\[2pt]
\begin{tabular*}{0.92\linewidth}{@{\extracolsep{\fill}}ll cccccc@{}}
\toprule
Box & Variant ($q$) & $T^-$ & $T^+$ & $h^-$ & $h^+$ & $Q^-$ & $Q^+$ \\
\midrule
\multicolumn{8}{@{}l}{\emph{constant source term}} \\
\multirow{2}{*}{original} & LN (0) & $5.70{\scriptstyle\pm1.42}$ & $9.27{\scriptstyle\pm1.99}$ & $1.10{\scriptstyle\pm0.25}$ & $0.53{\scriptstyle\pm0.15}$ & $14.50{\scriptstyle\pm0.58}$ & $14.45{\scriptstyle\pm1.20}$ \\
 & no-LN (1) & $4.11{\scriptstyle\pm0.27}$ & $4.35{\scriptstyle\pm0.64}$ & $3.17{\scriptstyle\pm0.32}$ & $3.26{\scriptstyle\pm0.76}$ & $6.26{\scriptstyle\pm3.08}$ & $5.14{\scriptstyle\pm1.19}$ \\
\addlinespace[2pt]
\multirow{2}{*}{larger} & LN (0) & $8.53{\scriptstyle\pm0.12}$ & $8.43{\scriptstyle\pm0.21}$ & $0.65{\scriptstyle\pm0.11}$ & $0.50{\scriptstyle\pm0.10}$ & $9.56{\scriptstyle\pm0.85}$ & $24.00{\scriptstyle\pm2.48}$ \\
 & no-LN (1) & $8.58{\scriptstyle\pm0.07}$ & $8.64{\scriptstyle\pm0.14}$ & $1.43{\scriptstyle\pm0.33}$ & $0.88{\scriptstyle\pm0.17}$ & $1.53{\scriptstyle\pm0.21}$ & $2.44{\scriptstyle\pm1.16}$ \\
\midrule
\multicolumn{8}{@{}l}{\emph{Arrhenius source term}} \\
\multirow{2}{*}{original} & LN (0) & $6.35{\scriptstyle\pm0.80}$ & $9.87{\scriptstyle\pm2.08}$ & $1.16{\scriptstyle\pm0.25}$ & $0.79{\scriptstyle\pm0.36}$ & $15.40{\scriptstyle\pm1.31}$ & $18.26{\scriptstyle\pm2.09}$ \\
  & no-LN (1) & $3.67{\scriptstyle\pm0.15}$ & $3.83{\scriptstyle\pm0.11}$ & $2.56{\scriptstyle\pm0.17}$ & $2.43{\scriptstyle\pm1.05}$ & $4.97{\scriptstyle\pm2.15}$ & $6.41{\scriptstyle\pm1.42}$ \\
\addlinespace[2pt]
\multirow{2}{*}{larger} & LN (0) & $8.93{\scriptstyle\pm0.10}$ & $10.69{\scriptstyle\pm0.23}$ & $0.95{\scriptstyle\pm0.16}$ & $0.60{\scriptstyle\pm0.11}$ & $8.27{\scriptstyle\pm1.25}$ & $52.47{\scriptstyle\pm4.05}$ \\
  & no-LN (1) & $9.15{\scriptstyle\pm0.11}$ & $10.71{\scriptstyle\pm0.22}$ & $1.70{\scriptstyle\pm0.30}$ & $1.28{\scriptstyle\pm0.58}$ & $1.75{\scriptstyle\pm0.23}$ & $8.76{\scriptstyle\pm1.20}$ \\
\midrule
\multicolumn{8}{@{}l}{\emph{change in RMSE [K], constant to Arrhenius}} \\
\multirow{2}{*}{original} & LN (0) & $+0.65$ & $+0.59$ & $+0.06$ & $+0.26$ & $+0.90$ & $+3.81$ \\
  & no-LN (1) & $-0.44$ & $-0.52$ & $-0.61$ & $-0.83$ & $-1.29$ & $+1.28$ \\
\addlinespace[2pt]
\multirow{2}{*}{larger} & LN (0) & $+0.39$ & $+2.27$ & $+0.29$ & $+0.10$ & $-1.30$ & $+28.46$ \\
  & no-LN (1) & $+0.58$ & $+2.07$ & $+0.27$ & $+0.41$ & $+0.23$ & $+6.32$ \\
\bottomrule
\end{tabular*}
\end{table}

\paragraph{Forward exponent of the trained models.}
Tab.~\ref{tab:qhat-fno} reports the empirical forward exponents $\qhat_k$ (App.~\ref{app:qhat})
at the final checkpoints. The LN variants yield $\qhat_k\approx 0$, consistent with their certified $q=0$. For no-LN, the mean estimate is $0.984$ ($95\%$ CI $[0.960,1.008]$) in the original parameter box, closely matching the certified $q=1$, and $0.705$ ($95\%$ CI $[0.673,0.738]$) in the larger box. Unlike the global certified exponent, the empirical estimate depends on the states visited during training.

The nonlinear source-term experiments yield nearly identical estimates: $0.987$ and $0.684$ for no-LN in the original and larger boxes, respectively, while LN remains near zero. Thus, the observed growth regimes are preserved under the change of target operator, consistent with the architecture-based exponent characterization of Theorem~\ref{thm:arith}.

\begin{table}[h]
\centering
\caption{Empirical forward exponents $\qhat_k$ for the FNO variants at the final checkpoints ($6{,}600$ iterations for the original box and $84{,}600$ for the larger box). Each of the $5$ seeds contributes its layerwise mean and maximum ($D=4$); entries report means with $95\%$ $t$-intervals across seeds, except the final column, which gives the maximum over all seeds and layers. The lower block uses the temperature-dependent source term
(App.~\ref{app:fno})  under the same experimental settings.}
\label{tab:qhat-fno}
\small
\begin{tabular}{@{}llcc cc c@{}}
\toprule
& & & & \multicolumn{2}{c}{mean over seeds} & largest over seeds \\
\cmidrule(lr){5-6}\cmidrule(l){7-7}
Box & Variant & $q$ & seeds
 & $\tfrac1L\sum_{k}\qhat_k$
 & $\max_k\qhat_k$
 & $\max_{\rm seeds}\max_k\qhat_k$ \\
\midrule
\multicolumn{7}{@{}l}{\emph{constant source term}} \\
\multirow{2}{*}{original} & LN & $0$ & $5$ & $-0.001\ [-0.002,0.001]$ & $0.001\ [0.000,0.003]$ & $0.003$ \\
 & no-LN & $1$ & $5$ & $0.984\ [0.960,1.008]$ & $1.036\ [1.026,1.047]$ & $1.044$ \\
\addlinespace[2pt]
\multirow{2}{*}{larger} & LN & $0$ & $5$ & $0.000\ [-0.002,0.002]$ & $0.002\ [0.000,0.004]$ & $0.005$ \\
 & no-LN & $1$ & $5$ & $0.705\ [0.673,0.738]$ & $0.824\ [0.796,0.852]$ & $0.862$ \\
\midrule
\multicolumn{7}{@{}l}{\emph{temperature-dependent source term}} \\
\multirow{2}{*}{original} & LN & $0$ & $5$ & $-0.001\ [-0.002,0.001]$ & $0.001\ [0.000,0.002]$ & $0.001$ \\
 & no-LN & $1$ & $5$ & $0.987\ [0.971,1.004]$ & $1.025\ [1.001,1.049]$ & $1.055$ \\
\addlinespace[2pt]
\multirow{2}{*}{larger} & LN & $0$ & $5$ & $0.001\ [-0.001,0.003]$ & $0.004\ [0.001,0.007]$ & $0.008$ \\
 & no-LN & $1$ & $5$ & $0.684\ [0.648,0.719]$ & $0.786\ [0.760,0.813]$ & $0.805$ \\
\bottomrule
\end{tabular}
\end{table}

\subsection{Language modeling: GPT-2 on OpenWebText}
\label{app:owt:results}

\paragraph{Attention sensitivity across depth.}
Tab.~\ref{tab:owt-blocks} reports layerwise attention coefficients and activation statistics at iteration $6{,}750$. Compared with Pre-LN, no-LN exhibits substantially larger query-key coefficients and attention sensitivity factors, particularly in deeper blocks. In block~10, the maximum attention logit reaches $4.0\times10^4$, while the estimated attention factor reaches $3.9\times10^5$. These measurements identify the deep attention blocks as the principal source of backward amplification.

\begin{table}[h]
\centering
\caption{GPT-2 124M on OpenWebText at iteration $6{,}750$. $\kappa$: $\max_h\|W_Q^{h\top}W_K^h\|_2$ (Pre-LN: mean over $5$ seeds). All other columns: no-LN (seed $1339$). $R$: typical token norm at the attention input; logit: largest attention logit; $c^{\rm attn}=\|W_O\|_2\|W_V\|_2$; $c^{\rm MLP}=1.13\|W_2\|_2\|W_1\|_2$; $L^{\rm attn}(R)$: layerwise attention factor.}
\label{tab:owt-blocks}
\footnotesize
\setlength{\tabcolsep}{5pt}
\begin{tabular}{@{}r cc c c ccc@{}}
\toprule
block & $\kappa$ Pre-LN & $\kappa$ no-LN & $R$ & logit & $c^{\rm attn}$ & $c^{\rm MLP}$ & $L^{\rm attn}(R)$ \\
\midrule
0 & $3.3$ & $614.5$ & $1.0$ & $8.0$ & $3.9$ & $13.0$ & $285$ \\
1 & $5.1$ & $39.9$ & $2.4$ & $6.0$ & $1.6$ & $6.1$ & $47.7$ \\
2 & $7.3$ & $58.6$ & $4.1$ & $10.9$ & $1.4$ & $3.2$ & $174$ \\
3 & $7.2$ & $36.9$ & $4.1$ & $12.0$ & $2.0$ & $3.0$ & $161$ \\
4 & $8.0$ & $44.4$ & $4.1$ & $13.1$ & $2.7$ & $3.3$ & $256$ \\
5 & $6.2$ & $53.2$ & $4.2$ & $23.1$ & $4.7$ & $3.9$ & $549$ \\
6 & $4.5$ & $46.6$ & $4.3$ & $29.8$ & $9.7$ & $5.3$ & $1.0\times10^{3}$ \\
7 & $4.6$ & $26.0$ & $4.5$ & $24.5$ & $4.3$ & $21.2$ & $287$ \\
8 & $4.2$ & $23.9$ & $11.7$ & $96.0$ & $29.6$ & $73.9$ & $1.2\times10^{4}$ \\
9 & $4.5$ & $19.4$ & $48.0$ & $1.6\times10^{3}$ & $5.6$ & $100.6$ & $3.1\times10^{4}$ \\
10 & $4.4$ & $175.5$ & $50.2$ & $4.0\times10^{4}$ & $7.1$ & $102.0$ & $3.9\times10^{5}$ \\
11 & $4.1$ & $4.2$ & $40.4$ & $229$ & $12.7$ & $93.4$ & $1.1\times10^{4}$ \\
\bottomrule
\end{tabular}
\end{table}

\paragraph{Divergence replay.}
Fig.~\ref{fig:owt-replay} resolves the no-LN divergence in Fig.~\ref{fig:owt} by depth, using seed $1339$ replayed from iteration $6{,}750$. Instability is concentrated in blocks~8--11, with the maximum attention logit increasing from $4.0\times10^4$ to $6.2\times10^9$ before the first non-finite value at iteration $6{,}781$. The residual stream remains finite and far below the bf16 overflow limit, consistent with instability arising through backward amplification rather than forward-state overflow.

\begin{figure}[h]
\centering
\includegraphics[width=\linewidth]{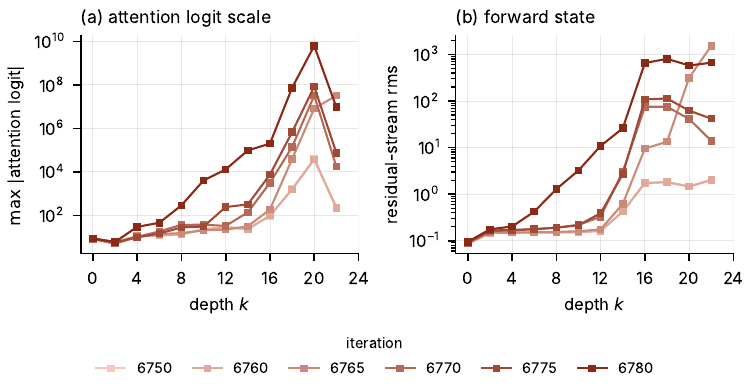}
\caption{no-LN $(q,\sigma)=(1,2)$, seed $1339$, replayed from iteration
$6{,}750$: (a) attention logit scale and (b) residual-stream rms against depth,
at six iterations of the replay. This is the depth-resolved form of
Fig.~\ref{fig:owt}c, from the same replay; the first non-finite value appears
at $6{,}781$.}
\label{fig:owt-replay}
\end{figure}

\paragraph{Stability with QK-input norm.}
Fig.~\ref{fig:owt-qknorm} compares Pre-LN $(q,\sigma)=(0,0)$ and no-LN with QK-input norm $(1,1)$ through iteration $20{,}000$. Both variants remain stable across all seeds, despite substantial attention logits and different residual-stream growth patterns. Thus, logit magnitude alone does not determine stability; controlling its sensitivity to the residual state is critical. QK-input norm achieves this without normalizing the residual stream, although its final validation loss ($3.21$--$3.23$) remains higher than that of Pre-LN ($3.03$--$3.06$).

\begin{figure}[h]
\centering
\includegraphics[width=\linewidth]{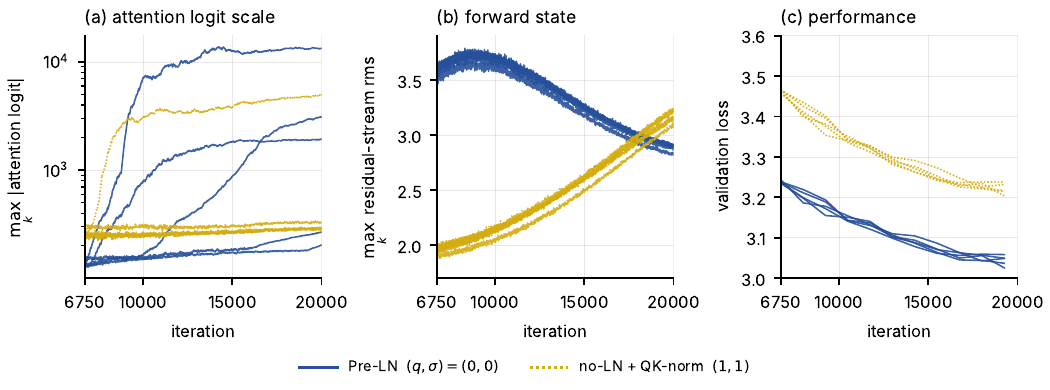}
\caption{Pre-LN $(0,0)$ against no-LN with QK-input norm $(1,1)$ from iteration
$6{,}750$ to $20{,}000$, one line per seed, $5$ seeds each; all ten runs are
finite at $20{,}000$. (a) Maximum over depth of the attention logit, per
iteration. (b) Maximum over depth of the residual-stream rms, per iteration.
(c) Validation loss, evaluated every $250$ iterations.}
\label{fig:owt-qknorm}
\end{figure}

\subsection{Time-series foundation model: Toto~2.0}
\label{app:toto:results}

\paragraph{Training stability and forecasting accuracy.} Tab.~\ref{tab:toto-arms} and  Fig.~\ref{fig:attention_mlp_toto} summarize training stability and forecasting accuracy across normalization variants. Both $q=2$ variants diverge in all five seeds, whereas all variants with $q\leq1$ complete training, regardless of $\sigma$. Among the stable variants, gate-norm and gate+QK-input norm achieve comparable validation loss and slightly lower CRPS than Pre-RMSNorm on both held-out and OOD GIFT-Eval evaluations.

\begin{table}[h]
\centering
\caption{Training stability and forecasting accuracy of the Toto~2.0 proxy ($5$ seeds per variant, $5{,}000$ training steps). Accuracy is measured by CRPS relative to seasonal naive (lower is better), as the geometric mean of that ratio over the $26$ scored held-out datasets and over the $97$ GIFT-Eval configurations, then the mean $\pm$ sd across the $5$ seeds. {Every variant reports its absolute value; the arms share seeds and bit-identical initial weights, so much of the across-seed spread is common to the rows and is wider than the gaps between them, and the seed-paired differences follow the table.} Held-out evaluation is in-distribution; GIFT-Eval is OOD. Missing entries indicate unavailable checkpoints due to divergence.}
\label{tab:toto-arms}
\footnotesize
\setlength{\tabcolsep}{1.7pt}
\begin{tabular}{@{}llllcccc@{}}
\toprule
& & & & \multicolumn{2}{c}{step $1{,}000$} & \multicolumn{2}{c}{step $5{,}000$} \\
\cmidrule(lr){5-6}\cmidrule(l){7-8}
Arm & $(q,\sigma)$ & Diverged & Final val. & held-out & GIFT-Eval & held-out & GIFT-Eval \\
\midrule
Pre-RMSNorm & $(0,0)$ & $0/5$ & $0.1062$ & $0.5296{\scriptstyle\pm0.0091}$ & $0.7498{\scriptstyle\pm0.0080}$ & $0.4726{\scriptstyle\pm0.0028}$ & $0.7142{\scriptstyle\pm0.0037}$ \\
gate+QK-input norm & $(1,1)$ & $0/5$ & $0.1057$ & $0.5151{\scriptstyle\pm0.0089}$ & $0.7268{\scriptstyle\pm0.0076}$ & $0.4673{\scriptstyle\pm0.0035}$ & $0.7038{\scriptstyle\pm0.0093}$ \\
gate-norm & $(1,2)$ & $0/5$ & $0.1060$ & $0.5265{\scriptstyle\pm0.0074}$ & $0.7234{\scriptstyle\pm0.0066}$ & $0.4681{\scriptstyle\pm0.0042}$ & $0.7016{\scriptstyle\pm0.0040}$ \\
QK-input norm & $(2,1)$ & $5/5$ & -- & $0.5327{\scriptstyle\pm0.0092}$ & $0.7393{\scriptstyle\pm0.0062}$ & -- & -- \\
no-LN & $(2,2)$ & $5/5$ & -- & -- & -- & -- & -- \\
\bottomrule
\end{tabular}
\end{table}

\paragraph{Forward exponent of the trained models.}
Tab.~\ref{tab:toto-qhat} reports the empirical forward exponents $\qhat_k$ (App.~\ref{app:qhat}) at iteration $500$, before any variant diverges. The estimates agree with the certified growth regimes, notably confirming $q=2$ for the unnormalized gated MLP in both no-LN and QK-input norm. Attention estimates slightly exceed $1$ when queries and keys receive unnormalized inputs ($1.11$--$1.13$), but approach $1$ with QK-input norm ($0.99$). These measurements support the predicted distinction between the forward growth of gated MLPs and the effect of query-key normalization on attention.

\begin{table}[h]
\centering
\caption{Forward exponent $\qhat_k$ of the Toto~2.0 arms at iteration $500$,
by sub-block parity: even depths are attention, odd depths are the gated MLP.
Mean over $5$ seeds of each seed's layer mean, with a $95\%$ $t$-interval
across seeds; the last column of each block is the largest single estimate over
all seeds and sub-blocks of that parity. $q$ is the certified exponent of that
sub-block.}
\label{tab:toto-qhat}
\small
\setlength{\tabcolsep}{4pt}
\begin{tabular}{@{}lcccccc@{}}
\toprule
& \multicolumn{3}{c}{even depths (attention)} & \multicolumn{3}{c}{odd depths (gated MLP)} \\
\cmidrule(lr){2-4}\cmidrule(l){5-7}
Arm & $q$ & $\tfrac1{12}\sum_k\qhat_k$ & largest & $q$ & $\tfrac1{12}\sum_k\qhat_k$ & largest \\
\midrule
Pre-RMSNorm & $0$ & $0.009\ [0.008,0.010]$ & $0.043$ & $0$ & $0.016\ [0.014,0.018]$ & $0.040$ \\
gate+QK-input norm & $1$ & $0.998\ [0.997,0.998]$ & $1.006$ & $1$ & $0.999\ [0.999,1.000]$ & $1.011$ \\
gate-norm & $1$ & $1.130\ [1.126,1.135]$ & $1.360$ & $1$ & $0.998\ [0.997,1.000]$ & $1.011$ \\
QK-input norm & $1$ & $0.994\ [0.993,0.995]$ & $1.005$ & $2$ & $2.052\ [2.050,2.054]$ & $2.103$ \\
no-LN & $1$ & $1.110\ [1.106,1.114]$ & $1.345$ & $2$ & $2.050\ [2.048,2.053]$ & $2.083$ \\
\bottomrule
\end{tabular}
\end{table}

\end{document}